\documentclass[10pt]{article}

\usepackage[utf8]{inputenc} 
\usepackage[T1]{fontenc}

\usepackage{url}          
\usepackage{booktabs}      
\usepackage{amsfonts}       
\usepackage{nicefrac}       
\usepackage{microtype}      
\usepackage{fullpage}
\usepackage{mathtools}
\usepackage{xcolor}
\newcommand*{\colorboxed}{}
\def\colorboxed#1#{
	\colorboxedAux{#1}
}
\newcommand*{\colorboxedAux}[3]{
	\begingroup
	\colorlet{cb@saved}{.}
	\color#1{#2}
	\boxed{
		\color{cb@saved}
		#3
	}
	\endgroup
}

\usepackage{amssymb}
\usepackage{amsmath,amsthm}

\usepackage{enumitem}

\usepackage{algorithm}
\usepackage{algorithmic}

\usepackage{xcolor,colortbl}
\definecolor{DarkBlue}{RGB}{22,54,93}

\usepackage[colorlinks = true]{hyperref}        
\usepackage{cleveref}
\hypersetup{linkcolor =red, citecolor = blue}

\usepackage{natbib}
\usepackage{graphicx}
\usepackage{subfigure}
\usepackage{adjustbox}
\usepackage{amsmath}
\usepackage{amssymb}
\usepackage{mathtools}
\usepackage{amsthm}
\usepackage{threeparttable}

\theoremstyle{plain}
\newtheorem{theorem}{Theorem}[section]
\newtheorem{proposition}[theorem]{Proposition}
\newtheorem{lemma}[theorem]{Lemma}
\newtheorem{corollary}[theorem]{Corollary}
\theoremstyle{definition}
\newtheorem{definition}[theorem]{Definition}
\newtheorem{assumption}[theorem]{Assumption}
\theoremstyle{remark}

\usepackage[textsize=tiny]{todonotes}

\allowdisplaybreaks[4]

\title{Reward-free Reinforcement Learning under Low-rank MDPs: Model-based Algorithm and Improved Sample Complexity}

\author
{
	Yuan Cheng\thanks{\small School of Statistics and Finance, University of Science and Technology of China, Hefei 230026, China; e-mail: {\tt   cy16@mail.ustc.edu.cn}}
	,~~~Songtao Feng\thanks{\small School of Electrical Engineering and Computer Science, The Pennsylvania State University, University Park, PA 16802, USA; e-mail: {\tt   sxf302@psu.edu}}
	,~~~Jing Yang\thanks{\small School of Electrical Engineering and Computer Science, The Pennsylvania State University, University Park, PA 16802, USA; e-mail: {\tt  yangjing@psu.edu }}
	,~~~Hong Zhang\thanks{\small School of Statistics and Finance, University of Science and Technology of China, Hefei 230026, China; e-mail: {\tt   zhangh@mail.ustc.edu.cn}}
	,~~~Yingbin Liang\thanks{\small Department of Electrical and Computer Engineering, The Ohio State University, OH 43210, USA; e-mail: {\tt   liang.889@osu.edu}}
}

 \usepackage[utf8]{inputenc} 
 \usepackage[T1]{fontenc}    
 \usepackage{hyperref}       
 \usepackage{url}            
 \usepackage{booktabs}       
 \usepackage{amsfonts}       
 \usepackage{nicefrac}      
 \usepackage{microtype}      
 \usepackage{xcolor}        
 
 \usepackage[utf8]{inputenc} 
 \usepackage[T1]{fontenc}

 \usepackage{url}           
 \usepackage{booktabs}       
 \usepackage{amsfonts}     
 \usepackage{nicefrac}      
 \usepackage{microtype}      
 \usepackage{mathtools}
 \usepackage{xcolor}

 \usepackage{amssymb}
 \usepackage{amsmath,amsthm}

 \usepackage{enumitem}
 
 \usepackage{algorithm}
 \usepackage{algorithmic}
 
 \usepackage{xcolor,colortbl}
 \definecolor{DarkBlue}{RGB}{22,54,93}

 \usepackage{cleveref}
 
 \usepackage{natbib}
 \usepackage{graphicx}
 \usepackage{subfigure}
 \usepackage{adjustbox}
 \usepackage{mathrsfs}

 \usepackage{amsmath}
 \usepackage{amssymb}
 \usepackage{mathtools}
 \usepackage{amsthm}
 \usepackage{threeparttable}

 \theoremstyle{plain}
 
\crefname{assumption}{assumption}{assumptions}

 \usepackage[textsize=tiny]{todonotes}
 \newcommand{\Pb}{\mathbb{P}}
 \newcommand{\Rb}{\mathbb{R}}
 \newcommand{\hP}{\widehat{P}}
 
 \newcommand{\hphi}{\widehat{\phi}}
 \newcommand{\tphi}{\widetilde{\phi}}
 \newcommand{\hmu}{\widehat{\mu}}

 \newcommand{\Eb}{\mathbb{E}}

 \newcommand{\sphi}{\phi^\star}
 
 \newcommand{\Mc}{\mathcal{M}}
 \newcommand{\Uc}{\mathcal{U}}
 \newcommand{\Sc}{\mathcal{S}}
 \newcommand{\Ac}{\mathcal{A}}
 \newcommand{\Dc}{\mathcal{D}}
 \newcommand{\Nc}{\mathcal{N}}
 \newcommand{\hb}{\widehat{b}}
 \newcommand{\hw}{\widehat{w}}
 
 \newcommand{\ph}{{h^\prime}}
 
 \newcommand{\tap}{{\widetilde{\alpha}}}

 \newcommand{\Rc}{\mathcal{R}}
 \newcommand{\RM}[1]{\left(\romannumeral#1\right)}
 \newcommand{\URM}[1]{\left(\mathrm{\uppercase\expandafter{\romannumeral#1}}\right)}

 \newcommand{\norm}[1]{\left\lVert#1\right\rVert}
 \newcommand{\dif}{\mathop{}\!\mathrm{d}}

\newcommand{\Doff}{\mathcal{D}_{\mathrm{down}}}

\newcommand{\Non}{N_{\mathrm{on}}}
\newcommand{\Noff}{N_{\mathrm{off}}}
\newcommand{\xioff}{\xi_{\mathrm{down}}}
\newcommand{\xioffm}{\max\{\xi_{\mathrm{down}},1\}}

\DeclareMathOperator*{\argmin}{argmin}
 
\renewcommand{\star}{*}

\makeatletter
\newcommand{\uset}[3][0ex]{
  \mathrel{\mathop{#3}\limits_{
    \vbox to#1{\kern -5\ex@
    \hbox{$#2$}\vss}}}}
\makeatother

\allowdisplaybreaks[4]

\title{Provable Benefit of Multitask Representation Learning in Reinforcement Learning}

\begin{document}
\maketitle
\begin{abstract}
As representation learning becomes a powerful technique to reduce sample complexity in reinforcement learning (RL) in practice, theoretical understanding of its advantage is still limited. In this paper, we theoretically characterize the benefit of representation learning under the low-rank Markov decision process (MDP) model. We first study multitask low-rank RL (as upstream training), where all tasks share a common representation, and propose a new multitask reward-free algorithm called REFUEL. REFUEL learns both the transition kernel and the near-optimal policy for each task, and outputs a well-learned representation for downstream tasks. Our result demonstrates that multitask representation learning is provably more sample-efficient than learning each task individually, as long as the total number of tasks is above a certain threshold. We then study the downstream RL in both online and offline settings, where the agent is assigned with a new task sharing the same representation as the upstream tasks. For both online and offline settings, we develop a sample-efficient algorithm, and show that it finds a near-optimal policy with the suboptimality gap bounded by the sum of the estimation error of the learned representation in upstream and a vanishing term as the number of downstream samples becomes large. Our downstream results of online and offline RL further capture the benefit of employing the learned representation from upstream as opposed to learning the representation of the low-rank model directly. To the best of our knowledge, this is the first theoretical study that characterizes the benefit of representation learning in exploration-based reward-free multitask RL for both upstream and downstream tasks.
\end{abstract}

\section{Introduction}
Multitask representation learning has arisen as a popular and powerful framework to deal with data-scarce environment in supervised learning, bandits and reinforcement learning (RL) recently \citep{DBLP:journals/jmlr/MaurerPR16,du2020few,DBLP:conf/icml/TripuraneniJJ21,DBLP:conf/iclr/YangHLD21,DBLP:journals/corr/abs-2202-10066,lu2021power,DBLP:conf/icml/AroraDKLS20,DBLP:conf/icml/HuCJL021}. The general multitask representation learning framework consists of both upstream multiple source tasks learning and downstream target task learning. In RL applications, in upstream, the agent is assigned with $T$ source tasks sharing a common representation and the goal is to learn a near-optimal policy for each task via learning the representation. In downstream, with the help of the learned representation, the agent aims to find a near-optimal policy of a new task that shares the same representation as the source tasks. While representation learning has achieved great success in supervised learning \citep{du2020few,DBLP:conf/icml/TripuraneniJJ21,DBLP:journals/jmlr/MaurerPR16,DBLP:conf/icml/KongSSKO20} and multi-armed bandits (MAB) problems \citep{DBLP:conf/iclr/YangHLD21,DBLP:journals/corr/abs-2201-04805,DBLP:journals/corr/abs-2202-10066}, most works in multitask RL mainly focus on empirical algorithms \citep{DBLP:conf/icml/Sodhani0P21,DBLP:journals/corr/abs-2202-11960,DBLP:conf/nips/TehBCQKHHP17} with limited theoretical works~\citep{DBLP:conf/icml/AroraDKLS20,DBLP:conf/icml/HuCJL021,DBLP:conf/uai/BrunskillL13,DBLP:journals/corr/abs-2201-08732,DBLP:journals/ia/CalandrielloLR15,lu2021power,DBLP:conf/iclr/DEramoTBR020}. Generally speaking, there are two main challenges for multitask representation learning in RL. First, differently from supervised learning with static datasets, in MAB and RL, representation learning and data collection are intertwined with each other. The agent has to choose an exploration policy in each iteration based on the representation learned in the past. Such iterative process introduces temporal dependence to the collected data. Note that none of the aforementioned theoretical multitask RL studies took the influence of sequential exploration and temporal dependence in data into consideration. Second, data distribution mismatch between the target task and the source tasks may occur. Additional assumption is usually required in order to ensure the representation learned from sources task will benefit the target task~\citep{lu2021power}.
    
In this work, we aim to unveil the fundamental impact of multitask representation learning on RL by answering the following two open questions:
\begin{itemize}
    \item First, in the upstream, is it possible to design exploration-based multitask representation learning algorithms for RL to achieve efficiency gain compared with single-task RL after taking the interactive data collection process and temporal data dependency into account? 
    \item Second, can the learned representation from upstream bring efficiency gain in sample complexity to the downstream learning?
\end{itemize}

\noindent{}\textbf{Main contributions:} We focus on multitask representation learning under low-rank MDPs \citep{NEURIPS2020_e894d787}, where the transition kernel admits a low-rank decomposition that naturally involves representation learning. 
It answers the aforementioned questions affirmatively, as summarized below.
\begin{itemize}
     \item First, we propose a new multitask representation learning algorithm called REFUEL under low-rank MDPs in upstream. REFUEL features joint modeling learning, pseudo cumulative value function construction for exploration, joint exploration policy learning and a unique termination criterion. These novel designs enable REFUEL to leverage the shared representation across source tasks to improve the learning efficiency.   
     \item We show that REFUEL can 1) find a near-optimal policy and 2) learn a near-accurate model for each of $T$ source tasks to certain average accuracy $\epsilon_u$, with a sample complexity of $\widetilde{O}\left(\frac{H^3d^2K^2}{T\epsilon_u^2}+\frac{\left(H^3d^4K+H^5dK^2+H^5d^3K\right)}{\epsilon_u^2}\right.$ $\left.+\frac{H^5K^3}{dT\epsilon_u^2}\right)$, where $H,K,d$ denote the horizon steps, the cardinality of action space and the low-rank dimension, respectively. Compared with state-of-the-art single-task representation learning in low-rank MDPs~\citep{uehara2021representation}, REFUEL achieves reduced sample complexity as long as the number $T$ of tasks satisfies $T=\Omega\left(\frac{K}{d^5}\right)$, which shows the benefit of multitask learning provably. To the best of our knowledge, this is the first theoretical result demonstrating the benefit of representation learning in exploration-based reward-free multitask RL.
     \item Finally, we apply the learned representation from REFUEL to downstream online and offline learning problems. In \textbf{offline} RL, we propose a new sample efficient algorithm with the suboptimality gap bounded by $\widetilde{O}(H\sqrt{d}\xioff+H\max\{d,\sqrt{p}\}/\sqrt{\Noff})$, where $\Noff$ is the number of samples of offline dataset and $p$ measures the model complexity of reward. 
     In \textbf{online} RL, we propose an online algorithm with the suboptimality gap bounded by $\widetilde{O}\left({Hd\xioff}+Hd^{1/2}N_{on}^{-1/2}\max\{d,\sqrt{p}\}\right)$, where $\Non$ is the number of iterations in the online algorithm. In both the offline and online settings, the first term captures the \textbf{approximation error} where $\xioff$ measures the quality of the learned representation from upstream, and is negligible when the upstream learning is sufficiently accurate, and the second term improves the dependency on $H$ and $d$ compared with the state of the art \citep{uehara2021representation} for low-rank MDPs without the knowledge of the representation, showing the benefit of employing the learned representation from upstream.
\end{itemize}

\noindent{}\textbf{Notations:} Throughout the paper, we use $[N]$ to denote set $\{1,\dots,N\}$ for $N\in \mathbb{N}$, use $\norm{x}_2$ to denote the $\ell_2$ norm of vector $x$, use $\norm{X}_{\mathrm{F}}$ to denote the Frobenius norm of matrix $X$, use $\triangle(\Ac)$ to denote the probability simplex over set $\Ac$, and use $\mathcal{U}(\Ac)$ to denote uniform sampling over $\Ac$, given $|\Ac|<\infty$.  Furthermore, for any symmetric positive definite matrix $\Sigma$, we let $\norm{x}_\Sigma : = \sqrt{x^\top \Sigma x}$. 

\section{Related work}
 
Due to the rapid growth of theoretical studies in RL, we discuss only highly relevant work below. \\
\textbf{Multitask bandits, multitask RL and meta-RL:} The benefit of multitask learning in linear bandits has been investigated in	 \cite{DBLP:conf/iclr/YangHLD21,DBLP:journals/corr/abs-2201-04805,DBLP:journals/corr/abs-2202-10066,DBLP:journals/corr/abs-2202-13001,DBLP:conf/nips/DeshmukhDS17,DBLP:conf/uai/CellaP21,DBLP:conf/icml/HuCJL021}. For multitask RL,
 	 \cite{DBLP:conf/icml/AroraDKLS20} showed that representation learning can reduce sample complexity for imitation learning. \cite{DBLP:journals/ia/CalandrielloLR15} studied multitask learning under linear MDPs with known representation.
 	 \cite{DBLP:conf/icml/HuCJL021} studied multitask RL with low inherent Bellman error~\citep{zanette2020learning} with known representation. A concurrent work \cite{DBLP:journals/corr/abs-2205-15701} studied multitask RL with low inherent Bellman error, bounded Eluder dimension and unknown representation.
 	  Assuming that all tasks are sampled from an unknown distribution, \cite{DBLP:conf/uai/BrunskillL13} studied the benefit of multitask RL when each task is sampled independently from a distribution over a finite set of MDPs while \cite{DBLP:journals/corr/abs-2201-08732} focused on meta-RL for linear mixture model~\citep{DBLP:conf/icml/YangW20}.
 	 Considering all tasks share a common representation, \cite{DBLP:conf/iclr/DEramoTBR020} showed the benefit in the convergence rate on value iteration and \cite{lu2021power} proved the sample efficiency gain of multitask RL under low-rank MDPs. None of the above works took the influence of sequential exploration and temporal dependence in data into consideration, where \cite{lu2021power} assumed a generative model is accessible and \cite{DBLP:conf/iclr/DEramoTBR020} assumed that the dataset is given in prior. In our study, we design reward-free exploration for multitask RL. \\
\textbf{Single-task RL under tabular and linear MDPs:} 
{\em Offline RL} under tabular MDPs was studied in \cite{DBLP:conf/icml/YangW20,yin2021towards,xie2021policy,DBLP:conf/nips/RenLDDS21,DBLP:conf/nips/YinW21a}, where provably efficient algorithms were proposed.
Further, offline RL under linear MDPs (with the the known representation) was studied in
\cite{Jin2021IsPP}, where an algorithm called PEVI was proposed and the suboptimality gap was characterized. \cite{DBLP:conf/nips/XieCJMA21} considered function approximation under the Bellman-consistent assumptions and \cite{DBLP:conf/nips/ZanetteWB21} considered the linear Bellman complete model which can be specialized to the linear MDP.
\cite{DBLP:journals/corr/abs-2203-05804} studied offline RL under linear MDPs and leverages the variance information to improve the sample complexity. \cite{DBLP:conf/iclr/WangFK21,DBLP:conf/icml/Zanette21} studied the statistical hardness of offline RL with linear representations and provided the exponential lower bounds.

{\em Online RL} under tabular MDPs has been extensively studied for both model-based and model-free settings~\citep{dann2017unifying,DBLP:conf/icml/StrehlLWLL06,DBLP:conf/nips/JinABJ18,DBLP:conf/icml/AzarOM17,DBLP:conf/icml/OsbandRW16,DBLP:journals/jmlr/JakschOA10}, and with the access to a generative model~\citep{DBLP:conf/aaai/KoenigS93,DBLP:conf/nips/AzarMGK11,DBLP:conf/alt/LattimoreH12,DBLP:conf/soda/SidfordWWY18}. A line of recent studies focused on function approximation under linear MDPs with known representation and proposed provably efficient algorithms~\citep{jin2020provably,zanette2020learning,zanette2020provably,DBLP:conf/iclr/0001WDK21,DBLP:conf/nips/NeuP20,DBLP:conf/colt/SunJKA019,DBLP:conf/icml/CaiYJW20}. Specifically, \cite{jin2020provably} also studied the setting where the representation of the transition probability has a small model misspecification error. We remark that although our downstream online/offline RL also has a misspecification error, it is due to the estimation error of the representation from upstream. Thus, our analysis of the suboptimality gap naturally relies on bridging the connection between upstream and downstream performances, which is not the case in \cite{jin2020provably}.

\textbf{Single-task RL under low-rank MDPs:} RL with Low-rank MDPs has been studied recently in~\cite{NEURIPS2020_e894d787} for the reward-free setting with the model-identification guarantee.  \citet{modi2021model} proposed a model-free algorithm for reward-free RL under low-nonnegative-rank MDPs. \cite{uehara2021representation} proposed provably efficient model-based algorithms for both online and offline RL with known reward.
\cite{du2019provably,misra2020kinematic,zhang2022efficient} studied the \textit{block MDP} which is a special case of low-rank MDPs. Further, algorithms have been proposed for MDP models with low Bellman rank \citep{jiang2017contextual}, low witness rank \citep{sun2019model}, bilinear classes \citep{du2021bilinear} and low Bellman eluder dimension~\citep{jin2021bellman}, which can be specialized to low-rank MDPs. 
Those algorithms can achieve better sample complexity bound, but contain computationally costly steps as remarked in \citet{uehara2021representation}. Moreover, \citet{zhang2021provably} studied a slightly different low-rank model, and showed a problem dependent regret upper bound for their proposed algorithm.

\textbf{Comparison to concurrent work:}
After we submitted this work for potential publication in May 2022, a concurrent work~\citep{DBLP:journals/corr/abs-2205-14571} was posted on arXiv. Both our work and \citet{DBLP:journals/corr/abs-2205-14571} study a common setting where the upstream learning is reward-free and has no access to a generative model, and downstream online RL applies the learned representation from upstream.  However, the two studies differ substantially in upstream algorithm design and sample complexity guarantee. First, one prominent difference in the upstream algorithm design is that our algorithm REFUEL {\it jointly} learns policies among the source tasks for exploration, whereas their algorithm learns the exploration policies for the source tasks {\it individually}. Second, for downstream online RL, although both results achieve the same suboptimality gap $Hd^{3/2}\Non^{-1/2}$, our guarantee requires much fewer samples from upstream than theirs due to the joint policy learning among source tasks and a different reward-free exploration scheme in our design. Further, \cite{DBLP:journals/corr/abs-2205-14571} directly assumes 
policy-level model misspecification from upstream to guarantee the downstream performance, whereas we make assumptions on the model classes $\Phi,\Psi$, based on which we theoretically characterize the misspecification error of the learned representation from upstream~(\Cref{lemma: Approximate Feature for new task}).

We further note that our work and \citet{DBLP:journals/corr/abs-2205-14571} also study some non-overlapping topics, such as downstream offline RL in our study, whereas upstream learning with generative models in their study. 

\section{Problem formulation}
	
\subsection{Episodic MDP}
Consider an episodic non-stationary Markov Decision Process (MDP) with finite horizon $H$, denoted by $\mathcal{M}=(\mathcal{S},\mathcal{A},H,P,r)$, where $\mathcal{S}$ is the state space (possibly infinite), $\mathcal{A}$ is the finite action space with cardinality $K$, $H$ is the number of steps in each episode, $P=\{P_h\}_{h\in[H]}$ is the collection of transition measures with $P_h:\mathcal{S}\times\mathcal{A}\mapsto\triangle(\mathcal{S})$, $r=\{r_h\}_{h\in[H]}$ is the collection of deterministic reward functions with $r_h:\mathcal{S}\times\mathcal{A}\mapsto[0,1]$. The initial state ${s}_1$ is assumed to be fixed for each episode and the sum of the reward is normalized with $\sum_{h=1}^{H}r_h\in[0,1]$. 

Denote policy $\pi=\{\pi_h\}_{h\in[H]}$ as a sequence of mappings where $\pi_h:\mathcal{S}\mapsto\triangle(\mathcal{A})$. We use $\pi_h(a|s)$ to denote the probability of selecting action $a$ in state $s$ at step $h$. Given a starting state $s_h$, we use $s_{h'} \sim (P, \pi)$ to denote a state sampled by executing policy $\pi$ under the transition model $P$ for $h'-h$ steps and $\Eb_{(s_h, a_h) \sim (P, \pi)}\left[\cdot\right]$ to denote the expectation over states $s_h \sim (P,\pi)$ and actions $a_h \sim \pi$. Given policy $\pi$ and a state $s$ at step $h$, the (state) value function is defined as $V_{h,P,r}^\pi(s):={\mathbb{E}}_{(s_{h'},a_{h'})\sim(P,\pi)}\left[\sum_{h'=h}^{H}r_{h'}(s_{h'},a_{h'})\middle| s_h=s\right]$.
Similarly, the (state-)action value function for a given state-action pair $(s,a)$ under policy $\pi$ at step $h$ is defined as $Q_{h,P,r}^\pi(s,a)=r_h(s,a)+(P_hV_{h+1,P,r}^{\pi})(s,a)$,
where $(P_h f)(s,a)=\Eb_{s'\sim P_h(\cdot|s,a)}\left[f(s')\right]$ for any function $f:\mathcal{S}\mapsto\mathbb{R}$. 
For simplicity, we omit subscript $h$ for $h=1$. 
Since the MDP begins with the same initial state $s_1$, to simplify the notation, we use $V_{P,r}^\pi$ to denote $V_{1,P,r}^\pi(s_1)$ without causing ambiguity in the following. 
Let $P^*$ denote the transition kernel of the underlying environment. Given a reward function $r$, there always exists an optimal policy $\pi^*$ that yields the optimal value function $V_{P^*,r}^{\pi^*}=\sup_\pi V_{P^*,r}^\pi$, abbreviated as $V_{P^*,r}^{*}$. 

This paper focuses on low-rank MDPs \citep{jiang2017contextual,NEURIPS2020_e894d787} defined as follows.
\begin{definition}[Low-rank MDPs]\label{definition: Low_rank}
A transition kernel $P_h^*: \Sc \times \Ac \rightarrow \triangle(\Ac)$ admits a low-rank decomposition with dimension $d \in \mathbb{N}$ if there exists two embedding functions $\phi_h^\star: \Sc \times \Ac \rightarrow \Rb^d$ and $\mu_h^\star: \Sc \rightarrow \Rb^d$ such that
\begin{align*}
P_h^\star(s^\prime|s,a)=\left\langle\phi_h^\star(s,a),\mu_h^\star(s^\prime)\right\rangle, \quad \forall s, s^\prime \in \Sc, a \in \Ac.
\end{align*}
Without loss of generality, we assume $\norm{\phi_h^*(s,a)}_2\leq 1$ for all $(s,a)\in\mathcal{S}\times\mathcal{A}$ and for any function $g: \Sc \mapsto [0,1]$, $\norm{\int \mu^\star_h(s)g(s)ds }_2\leq\sqrt{d}$. An MDP is a low-rank MDP with dimension $d$ if for $h \in [H]$, its transition kernel $P_h^*$ admits a low-rank decomposition with dimension $d$. Let $\phi^\star=\{\phi^\star_h\}_{h \in [H]}$ and $\mu^\star=\{\mu^\star_h\}_{h \in [H]}$ be the embeddings for $P^\star$, where $\sphi$ is also called representation in RL literature.
\end{definition}

\vspace{-0.05in}
\subsection{Reward-free multitask RL and its downstream learning}
\vspace{-0.05in}We consider a multitask RL consisting of reward-free upstream and reward-known downstream.

In {\bf reward-free multitask} upstream learning, the agent is expected to conduct {\bf reward-free exploration} over $T$ source tasks, where all reward functions $\{r^t\}_{t\in[T]}$ are unknown. Each task $t \in [T]$ is associated with a low-rank MDP $\Mc^{t}=(\Sc,\Ac,H,P^t,r^t)$. All $T$ tasks are identical except for (a) their transition model $P^t$, which admits a low-rank decomposition with dimension $d$: $P^t(s^\prime|s,a)=\langle\sphi(s^\prime),\mu^{(\star,t)}(s,a)\rangle$ for all $t \in [T]$ and (b) their reward $r^t$. Namely, {\bf all tasks share the same feature function $\sphi$}, but may differ in $\mu^{(\star,t)}$ as well as reward $r^t$. The goal of upstream learning is to find a near-optimal policy and a near-accurate model for any task $t\in[T]$ and any reward function $r^t,t \in [T]$ via sufficient reward-free exploration, and output a well-learned representation for downstream task. For each $t \in [T]$, we use $P^{(\star,t)}$ to denote the true transition kernel of task $t$.

In downstream learning, the agent is assigned with a new target task $T+1$ associated with a low-rank MDP $\Mc^{T+1}=(\Sc,\Ac,H,P^{T+1},r^{T+1})$. The task shares the same $\Sc,\Ac,H$ and $\sphi$ with the $T$ upstream tasks, but has a task-specific $\mu^{(\star,T+1)}$. We also assume that the reward for the new task $r^{T+1}$ is given to the agent. Here, the agent is assumed to take the learned representation function $\hphi$ during the upstream training, and then interacts with the new task environment $\Mc^{T+1}$ in an {online or} offline manner to find a near-optimal policy for the new task. By utilizing the representation learned in the upstream, the agent is expected to achieve better sample efficiency in the downstream. 
Finally, we use $P^{(\star,T+1)}$ to denote the true transition kernel of task $T+1$.

\vspace{-0.05in}	
\section{Upstream reward-free multitask RL}\label{sec: upstream}
\vspace{-0.05in}
In this section, we present our proposed algorithm for upstream reward-free multitask RL in low-rank MDPs and characterize its theoretical performance. 

Since it is impossible to learn a model in polynomial time if there is no assumption on features $\phi_h$ and $\mu_h$, we first adopt the following conventional assumption~\citep{NEURIPS2020_e894d787}.
\begin{assumption}\label{assumption: realizability}(Realizability).
A learning agent can access to a model class {$\{(\Phi,\Psi)\}$} that contains the true model, 
namely, for any $h \in [H], t \in[T]$, the embeddings  $\phi_h^\star \in \Phi, \mu_h^{(\star,t)} \in \Psi$. 
\end{assumption}
While we assume the cardinality of function classes to be finite for simplicity, extensions to infinite classes with bounded statistical complexity are not difficult \citep{sun2019model,NEURIPS2020_e894d787}. 

\begin{algorithm}[t]
 		\caption{{\bf REFUEL }(\textbf{RE}ward \textbf{F}ree m\textbf{U}ltitask r\textbf{E}presentation Learning) \label{alg: Upstream}}
 		\begin{algorithmic}[1]
 			\STATE {\bfseries Input:} Regularizer $\lambda_n$, parameter $\tap_n$, $\zeta_n$,B , $\delta$, $\epsilon_u$,{ Models $\{(\mu,\phi): \mu \in \Psi, \phi \in \Phi \}$}.
 			\STATE Initialize $\pi_0(\cdot|s)$ to be uniform, set $\mathcal{D}_h^{(0,t)}=\emptyset$ for $h \in [H],t \in [T]$.
 			\FOR{$n=1,\ldots,N_u$}
 			\FOR{$h=1,\ldots,H$}
 			\FOR{$t=1,\ldots, T$}
 			\STATE Under MDP $\Mc^t$, use $\pi_t^{n-1}$ to roll into $s_{h-1}^{(n,t,h)}$, uniformly choose $a_{h-1}^{(n,t,h)},a_{h}^{(n,t,h)}$ and enter into $s_{h}^{(n,t,h)},s_{h+1}^{(n,t,h)}$ and collect data {\small $s_1^{(n,t,h)}, a_1^{(n,t,h)},\ldots, s_h^{(n,t,h)}, a_h^{(n,t,h)}, s_{h+1}^{(n,t,h)}$.}
 			\STATE Add the triple $(s_h^{(n,t, h)}, a_h^{(n,t, h)}, s_{h+1}^{(n,t, h)})$ to the dataset $\Dc_h^{(n,t)} = \Dc_h^{(n-1,t)} \cup \{(s_h^{(n,t, h)}, a_h^{(n,t, h)}, s_{h+1}^{(n,t, h)})\}$.
 			\ENDFOR
 			\STATE Learn $\left(\hphi_h^{(n)},\hmu_h^{(n,1)},\ldots,\hmu_h^{(n,T)}\right)$ via $MLE\left(\bigcup_{t\in [T]} \Dc_h^{(n,t)}\right)$ as \Cref{eq: MLE}. \label{line: MLE}
 			\FOR{$t=1,\ldots,T$}
 			\STATE Update estimated transition kernels $\hP_h^{(n,t)}(\cdot|\cdot,\cdot)$ as \Cref{ineq: p}, empirical covariance matrix $\widehat{U}_h^{(n,t)}$ as \Cref{ineq: def of U} and exploration-driven reward $\hb^{(n,t)}(\cdot,\cdot)$ as \Cref{ineq: bonus function}. \label{line: transition}
 			\ENDFOR
 			\ENDFOR
 			\STATE Set \textbf{ Pseudo Cumulative Value Function} $PCV\left(\hP^{(n,t)},\hb_{h}^{(n,t)},\pi_t;T\right)$ as \Cref{ineq: Pseudo Cumulative Value}.
 			\STATE Get exploration policy $\pi_1^n, \dots, \pi_T^n = \mathop{\arg \max}_{\pi_1, \dots, \pi_T}PCV\left(\hP^{(n,t)},\hb_{h}^{(n,t)},\pi_t;T\right)$.\label{line: planning} 
 			\IF{$2PCV\left(\hP^{(n,t)},\hb_{h}^{(n,t)},\pi_t;T\right)+2\sqrt{KT\zeta_n}\leq T\epsilon_u$}
 			\STATE {\bfseries Terminate} and set $n_u = n$ and {\bfseries output} $\hphi=\hphi^{(n_u)},\hP^{(1)}=\hP^{(n_u,1)},\ldots,\hP^{(T)}=\hP^{(n_u,T)}$.
 			\ENDIF
 			\ENDFOR
 			\STATE {\bfseries Output: $\hphi=\hphi^{(n_u)},\hP^{(1)}=\hP^{(n_u,1)},\ldots,\hP^{(T)}=\hP^{(n_u,T)}$}.
 		\end{algorithmic}
 	\end{algorithm}
\vspace{-0.05in}
\subsection{Algorithm design}
\vspace{-0.05in}We first describe our proposed algorithm REFUEL (\textbf{RE}ward \textbf{F}ree m\textbf{U}ltitask r\textbf{E}presentation Learning) depicted in \Cref{alg: Upstream}. 

In each iteration $n$, for each task $t$ and time step $h$, the agent executes the exploration policy $\pi_t^{n-1}$ followed by two uniformly chosen actions to collect trajectories for $H$ episodes. Note that $\pi_t^{n-1}$ is designed at the previous iteration. Then, to utilize the common feature of data collected from different source tasks, the agent processes the data jointly as following steps.	
 	
\noindent{}{\bf Joint MLE-based Modeling Learning:}
The agent passes all previously collected  data to estimate the low-rank components $\hphi_h^{(n)}, \hmu_h^{(n,1)},\ldots,\hmu_h^{(n,t)}$ simultaneously via the MLE oracle $ MLE\left(\bigcup_{t\in [T]} \Dc_h^{(n,t)}\right)$ (line \ref{line: MLE}) on the joint distribution defined as follows: \vspace{-0.03in}
\begin{align} \left(\hspace{-0.01in}\hphi_h^{(n)}\hspace{-0.03in},\hmu_h^{(n,1)}\hspace{-0.03in},\ldots,\hmu_h^{(n,T)}\hspace{-0.01in}\right) = \hspace{-0.1in}\mathop{\arg\max}_{\phi_h \in \Phi,\mu_h^{1},\ldots,\mu_h^{T}\in \Psi} \sum_{n=1}^N\hspace{-0.01in} \log\hspace{-0.01in}\left(\hspace{-0.01in}\prod_{t=1}^T\langle\phi_h(s_h^{(n,t,h)}\hspace{-0.03in},a_h^{(n,t,h)}),\mu_h^t(s_{h+1}^{(n,t,h)})\rangle\hspace{-0.01in}\right). \label{eq: MLE}
\end{align}
\vspace{-0.03in}The MLE oracle above generalizes the MLE oracle commonly adopted for single-task RL in the literature \citep{NEURIPS2020_e894d787,uehara2021representation}. In practice, MLE oracle can be efficiently implemented if the model classes $\Phi,\Psi$ are properly parameterized such as by neural networks.

 	Next, for each task $t$, the agent uses the learned embeddings $\hphi_h^{(n)},\hmu_h^{(n,t)}$ to update the estimated transition kernel $\hP^{(n,t)}$ at each step $h$ as
 	\begin{align}
 	\textstyle \hP_h^{(n,t)}(s'|s,a) = \langle \hphi_h^{(n)}(s,a), \hmu_h^{(n,t)}(s')\rangle. \label{ineq: p}
 	\end{align}
 	
\noindent{}{\bf Pseudo Cumulative Value Function Construction for Exploration:} The agent first uses the representation estimator $\hphi_h^{(n)}$ to update the empirical covariance matrix $\widehat{U}_h^{(n,t)}$ as
 	\begin{align}
 	\textstyle \widehat{U}_h^{(n,t)} = \sum_{\tau=1}^{n} {\hphi}_h^{(n)}(s_h^{(\tau,t,h+1)},a_h^{(\tau,t,h+1)})(\hphi_h^{(n)}(s_h^{(\tau,t,h+1)},a_h^{(\tau,t,h+1)}))^{\top} + \lambda_n I, \label{ineq: def of U}
 	\end{align}
 	where 
 	$\{s_h^{(\tau,t,h+1)},a_h^{(\tau,t,h+1)}, s_{h+1}^{(\tau,t,h+1)}\}$ is the tuple collected at the $\tau$-th iteration, $(h+1)$-th episode, and step $h$. Then, the agent uses both $\hphi_h^{(n)}$ and $\widehat{U}_h^{(n,t)}$ to provide an exploration-driven reward as
 	\begin{align}
 	\textstyle \hb_{h}^{(n,t)}(s_{h},a_{h})=\min\left\{\tap_n\left\|\hphi_{h}^{(n)}(s_{h},a_{h})\right\|_{(\widehat{U}^{(n,t)}_{h})^{-1}},B\right\},\label{ineq: bonus function}
 	\end{align}
where $B$ and $\tap_n$ are pre-determined parameters. To integrate information of learned models and exploration-driven rewards for $T$ tasks, we define a \textbf{Pseudo Cumulative Value Function} (PCV) as\vspace{-0.04in}
 	 \begin{align}
 	   \textstyle PCV\left(\hP^{(n,t)},\hb_{h}^{(n,t)},\pi_t;T\right)=\sum_{h=1}^{H-1}\sqrt{\sum_{t=1}^T\left\{\mathop{\Eb}_{s_{h} \sim (\hP^{(n,t)}, \pi_t) \atop a_{h}\sim\pi_t} \left[\hb_{h}^{(n,t)}(s_{h},a_{h})\right] \right\}^2}. \label{ineq: Pseudo Cumulative Value}
 	 \end{align}
{\bf Joint Policy Learning for Exploration:} The agent then learns $T$ policies via optimization over the PCV, and uses these policies as the exploration policy for each task in the next iteration. Since PCV measures the cumulative model estimation error over all tasks, these policies will explore the state-action space where the model estimation has large uncertainty on average of $T$ tasks.
Here we adopt an optimization oracle for line \ref{line: planning} in \Cref{alg: Upstream}. 
 	
\noindent{}{\bf Termination:} REFUEL is equipped with a termination condition and outputs the current estimated model and representation during iterations if the PCV plus certain minor term is below a threshold, which suffices to guarantee that for each task $t$, the returned model $\hP^{(t)}$ from REFUEL reaches an average proximity to the true model $P^{(\star,t)}$.

The joint learning and decision making among all tasks involved in the four main components differentiate REFUEL from single-task RL algorithms, and enable it to exploit the shared representation across the source tasks to improve the sample efficiency. 

\vspace{-0.05in}	
\subsection{Theoretical results on sample complexity}
\vspace{-0.05in}In this section, we first establish the sample complexity upper bound for REFUEL and then compare our multitask result with that of single-task RL.
 	\begin{theorem}\label{theorem: Upstream sample complexity}
 		For any fixed $\delta \in (0,{|\Psi|^{-\min\{T,\frac{K}{d^2}\}}})$,  set
 		$\lambda_n=O(d\log(|\Phi|nTH/\delta))$, $\zeta_n  = O( {\log\left(|\Phi||\Psi|^T nH/\delta\right)}/{n})$, $\tap_n=O(\sqrt{K\log(|\Phi||\Psi|^TnH/\delta)+\lambda_ndT})$ and $B=O(\sqrt{T+K/d^2})$.
 		 Let $\hphi,\hP^{(1)},\ldots,\hP^{(T)}$ be the output of \Cref{alg: Upstream}. Then, under \Cref{assumption: realizability}, 
 		with probability at least $1-\delta$, for any policy $\pi_t$ of task $t \in [T]$ and any step $h \in [H]$, we have
 		\begin{align}
 		\textstyle \frac{1}{T}\sum_{t=1}^{T}&\mathbb{E}_{s_h\sim (P^{(\star,t)},\pi_t),a_h\sim \pi_t}  \left[\left\|\hP_h^{(t)}(\cdot|s_h,a_h)-P_h^{(\star,t)}(\cdot|s_h,a_h)\right\|_{TV}\right] \leq \epsilon_u.  \label{ineq: upstream guarantee}
 		\end{align}
Further, for any task $t$, given any reward $r^t$, let $\widehat{\pi}_t={\arg\max}_\pi V_{\hP^{(t)},r^t}^\pi$. Then with probability at least $1-\delta$, $\frac{1}{T}\sum_{t=1}^T[ V_{P^{(\star,t)},r^t}^\star-V_{P^{(\star,t)},r^t}^{\widehat{\pi}_t}] \leq \epsilon_u$. 

 		Meanwhile, the number of trajectories collected by each task is at most \vspace{-0.04in}
 		 \begin{equation}
 		     	\textstyle \widetilde{O}\left(\frac{H^3d^2K^2}{T\epsilon_u^2}+\frac{\left(H^3d^4K+H^5dK^2+H^5d^3K\right)}{\epsilon_u^2}+\frac{H^5K^3}{dT\epsilon_u^2}\right)\label{ineq: Theorem 3.1 sample complexity}.
 		 \end{equation}
\end{theorem}
\vspace{-0.04in}\Cref{ineq: upstream guarantee} in \Cref{theorem: Upstream sample complexity} indicates that the estimated transition kernels of the $T$ MDPs meet the required accuracy on average. \Cref{theorem: Upstream sample complexity} further indicates that with the data collected during such reward-free exploration, if every task is further given any reward function, near-optimal policies can be found for all tasks on average as well. 

For the sample complexity bound in \Cref{ineq: Theorem 3.1 sample complexity}, the first term is inversely proportional to $T$, which captures sample-benefit for each task due to learning the $T$ models jointly. The second term is related to the number of samples to guarantee the concentration of empirical covariance matrix $\widehat{U}_h^{(n,t)}$ in order to identify desirable exploration policies. Such concentration needs to be satisfied for each task $t$ independently and can not be implemented jointly, which makes the second term independent with $T$. The third term originally arises from the model estimation error shift among $T$ tasks when implementing joint MLE-based model estimation. In the algorithm, this error shift is then transferred to $B$ defined in \Cref{ineq: bonus function}, which provides an explicit upper bound for the exploration-driven reward $\hb_h^{(n,t)}$. Note that $\hb_h^{(n,t)}$ is involved in the PCV to guide exploration, which finally contributes to convergence of REFUEL and affects its sample complexity. 

Compared with the state-of-the-art sample complexity that scales in $\widetilde{O}\left(\frac{H^5d^4K^2}{\epsilon_u^2}\right)$\footnote{We convert their results in discounted setting to our episodic MDP setting by replacing $1/(1-\gamma)=\Theta(H)$.} when performing $T$ single-task RL independently~\citep{uehara2021representation},
REFUEL achieves lower sample complexity when $T=\Omega\left(\frac{K}{d^5}\right)$, and it reduces the single-task sample complexity in~\citet{uehara2021representation} by an order of $O\left(\min\{{H^2d^2T},{H^2K},{d^3},{dK},\frac{d^5T}{K}\}\right)$.

Technically, the joint learning feature of REFUEL requires new analytical techniques. We outline the main steps of the proof of \Cref{theorem: Upstream sample complexity} here and defer the detailed proof to \Cref{Appd: proof of thm1}.
\textbf{Step 1:} 
We develop a new upper bound on model estimation error for each task, which captures the advantage of joint MLE model estimation over single-task learning, as shown in \Cref{prop1: Bounded TV}.
\textbf{Step 2:} 
We establish the PCV as an uncertainty measure that captures the difference between the estimated and ground truth models. This justifies using such a PCV as a guidance for further exploration, as shown in \Cref{prop2: total value difference}.
 \textbf{Step 3:} We show that the summation of the PCVs over all iterations is sublinear with respect to the total number $N_u$ of iterations, as shown in \Cref{prop3: Bound of summation of exploration-driven reward function}, which further implies polynomial efficiency of REFUEL in learning the models.

\section{Downstream RL}
In downstream RL, the agent is given a new target task $T+1$ under MDP $\mathcal{M}^{T+1}$ and aims to find a near-optimal policy for it. The downstream task is assumed to share the same representation as the upstream tasks, and hence the agent can adopt the representation learned from the upstream to expedite its learning. Since the agent does not know the task-specific embedding function $\mu^{T+1}$, the agent is allowed to interact with MDP $\Mc^{T+1}$ in an online RL setting or have access to an offline dataset  $\Doff=\{(s_h^\tau,a_h^\tau,r_h^\tau,s_{h+1}^\tau)\}_{\tau,h=1}^{\Noff,H}$ in an offline RL setting, which is rolled out from some behavior policy $\rho$.

In this section, we first introduce the connections between upstream and downstream MDPs in \Cref{subsec: downstream assumption}. We then provide algorithms for the downstream task and characterize their suboptimality gap for {\bf offline} RL in \Cref{sec:offlinealg,sec:off-result}, respectively, and for {\bf online} RL in \Cref{sec:onlinealg,sec:on-result}, respectively.

\subsection{Connections between upstream and downstream MDPs}\label{subsec: downstream assumption}

In downstream RL tasks, the agent will adopt the feature $\hphi$ learned from upstream \Cref{alg: Upstream}. To ensure that the feature learned in upstream can still work well in downstream tasks, upstream and downstream MDPs should have certain connections. We next elaborate some reasonable assumptions on transition kernels to build such connections. 

      	The following reachability assumption is common in relevant literature~\citep{modi2021model}.
 	\begin{assumption}[Reachability]\label{assumption: reachability}
 		For each source task $t$ in upstream, whose transition kernel is $P^{(\star,t)}$, there exists a policy $\pi_t^0$ such that $\min_{s\in\Sc} \Pb^{(\pi_t^0,t)}_h(s)\geq \kappa_u$, where $\Pb^{(\pi_t^0,t)}_h(\cdot): \Sc \rightarrow \Rb$ is the density function over $\Sc$ using policy $\pi_t^0$ to roll into state $s$ at timestep $h$. 
 	\end{assumption}
 	\begin{assumption}\label{assumption: finite measurement}
 		Assume the state space $\Sc$ is compact and has finite measure  $1/\upsilon$. Hence, the uniform distribution on $\Sc$ has the density function $f(s)=\upsilon$.
 	\end{assumption}
Under \Cref{assumption: finite measurement}, we denote $\Uc(\Sc)$ as the uniform distribution over $\Sc$ and $\Uc(\Sc,\Ac)$ as the uniform distribution over $\Sc \times \Ac$. 
 	\begin{assumption}\label{assumption: smooth TV distance}
For any two different models in the model class $\{(\Phi,\Psi)\}$, say $P^1(s^\prime|s,a)=\left\langle\phi^1(s,a),\mu^1(s^\prime)\right\rangle$ and $P^2(s^\prime|s,a)=\left\langle\phi^2(s,a),\mu^2(s^\prime)\right\rangle$, there exists a constant $C_R$ such that for all $ (s, a) \in  \Sc \times \Ac$ and $h \in [H]$,
 		\begin{align}
 		&\|P^1_h(\cdot|s,a)-P^2_h(\cdot|s,a)\|_{TV} \leq  C_R \mathbb{\Eb}_{(s_h,a_h)\sim \Uc(\Sc,\Ac)}{\|P^1_h(\cdot|s_h,a_h)-P^2_h(\cdot|s_2,a_2)\|_{TV}}.
 		\end{align}
		For normalization, we assume that for any $\phi \in \Phi$,  $\norm{\phi(s,a)}_2\leq 1$ and for any $\mu \in \Psi$ and any function $g: \Sc \rightarrow [0,1], \norm{\int \mu_h(s)g(s)ds }_2\leq\sqrt{d}$. 
 	\end{assumption}
 	\Cref{assumption: smooth TV distance} ensures that for each source task $t$ and any $(s,a) \in \Sc \times \Ac$,  the total variation distance between the learned $\hP^{(t)}(\cdot|s,a)$ and true transition kernels  $\hP^{(\star,t)}(\cdot|s,a)$ will not be too large when the expectation of the total variation distance over the entire state action space is small. 
 	\begin{assumption}\label{assumption: Linear combination}
 		For the underlying MDP of task $T+1$, we assume that the reward function $r^{T+1}$ is chosen from a function class $\mathcal{R}$ with $\varepsilon$-covering number $\Nc_{\mathcal{R}}(\varepsilon)$ for any given $\varepsilon$, and  $\Nc_{\mathcal{R}}(\varepsilon)=\widetilde{O}\left(\frac{1}{\varepsilon^p}\right)$. ). 
 		The true transition kernel $P^{(\star,T+1)}$ can be $\xi$-approximated by a linear combination of $T$ source tasks, i.e. there exist $T$ (unknown) coefficients $c_1,\ldots,c_T \in [0,C_L]$ and $\xi \geq 0$ such that \vspace{-0.04in}
 		\begin{align}
 	\textstyle	\forall (s, a) \in \Sc \times \Ac, h \in [H],\quad \left\|P^{(\star,T+1)}(\cdot|s,a)-\sum_{t=1}^Tc_t P^{(\star,t)}(\cdot|s,a)\right\|_{TV} \leq \xi.
 		\end{align}
 		Furthermore, $\xi$ is called the linear combination misspecification.
 	\end{assumption}
    \vspace{-0.05in}\Cref{assumption: Linear combination} establishes the underlying connection between upstream source tasks and the downstream target task. The assumption on $\Nc_\varepsilon(\mathcal{R})$ is more general than the common assumption on reward under linear MDP and can be easily achieved in practice. For example, if the reward is linear with respect to an unknown feature $\tphi: \Sc \times \Ac \rightarrow \Rb^{p}$, i.e., there exists $\theta \in \Rb^{p}$ and $\norm{\theta} \leq R$, $r(s,a)=\langle\tphi(s,a),\theta \rangle$, then $\Nc_\varepsilon(\mathcal{R}) = \widetilde{O}\left(\left(1+\frac{2R}{\varepsilon}\right)^{p}\right)$(see Lemma D.5 in \cite{jin2020provably}), which satisfies the assumption. We remark here $\tphi$ is not necessarily the same as $\sphi$, and even its dimension $p$ can be different from $d$.

 We then introduce the definition of $\epsilon$-approximate linear MDP, and provide a guarantee for the estimated feature.

 	\begin{definition}[$\epsilon$-approximate linear MDP]
 		For any $\epsilon >0$, MDP $\Mc = (\Sc,\Ac, H, P, r)$ is an $\epsilon$-approximate linear MDP with a time-dependent feature map $\phi: \Sc \times \Ac \rightarrow \Rb^d$ if, for any $h \in [H]$, there exist a time-dependent unknown (signed) measures $\mu$ over $\Sc$ such that for any $(s, a) \in \Sc \times \Ac$, we have
 		\begin{align}
 		\left\|P_h(\cdot|s,a)-\left\langle\phi_h(s,a),\mu_h(\cdot)\right\rangle\right\|_{TV} \leq \epsilon. \label{ineq: approxiamte feature}
 		\end{align}
 		Any $\phi$ satisfying \Cref{ineq: approxiamte feature} is called an $\epsilon$-approximate feature map for $\Mc$. 
 	\end{definition}
 	
 	 Let $\xi_{down}=\xi+\frac{C_LC_RT\upsilon\epsilon_u}{\kappa_u}$. The following lemma shows that the feature $\hphi$ learned in upstream is a $\xi_{down}$-approximate feature map and can approximate the true feature in the new task. 
 	\begin{lemma}\label{lemma: Approximate Feature for new task} Under \Cref{assumption: reachability,assumption: finite measurement,assumption: smooth TV distance,assumption: Linear combination},
 		 the output of \Cref{alg: Upstream} $\hphi$ is a $\xioff$-approximate feature for MDP $\Mc^{T+1} = (\Sc,\Ac, H, P^{(\star,T+1)}, r, s_1)$, i.e. there exist a time-dependent unknown (signed) measure $\hmu^\star$ over $\Sc$ such that for any $(s, a) \in \Sc \times \Ac$, we have
 		\begin{align}
 		\textstyle\left\|P_h^{(\star,T+1)}(\cdot|s,a)-\left\langle\hphi_h(s,a),\hmu_h^\star(\cdot)\right\rangle\right\|_{TV} \leq \xi_{down}. 
 		\end{align}
 		Furthermore, for any $g: \Sc \rightarrow [0,1]$, $\norm{\int \hmu^\star_h(s)g(s)ds }_2\leq C_L\sqrt{d}$.
 	\end{lemma}

\subsection{Algorithm for downstream offline RL}\label{sec:offlinealg}

We present our pessimistic value iteration algorithm for downstream offline RL in \Cref{alg: Offline RL}. Although our algorithm shares similar design principles as traditional algorithms for linear MDPs~\citep{Jin2021IsPP}, it differs from them significantly as described in the following, due to the misspecification of representation from upstream and the general rather than linear reward function adopted.
For ease of exposition, we define Bellman operator $\mathbb{B}_h$ as $(\mathbb{B}_hf)(s,a)=r_h(s,a)+(P_h^{(\star,T+1)}f)(s,a)$ for any $f:\mathcal{S}\times\mathcal{A}\mapsto\mathbb{R}$. The main body of the algorithm consists of a backward iteration over steps. In each iteration $h$, the agent executes the following main steps.

\begin{algorithm}
\caption{Downstream Offline RL}
\label{alg: Offline RL}
 		\begin{algorithmic}[1]
 			\STATE {\bfseries Input:} {Feature $\hphi$, dataset $\Doff=\{(s_h^\tau,a_h^\tau,r_h^\tau,s_{h+1}^\tau)\}_{\tau,h=1}^{\Noff,H}$, parameters $\lambda$, $\beta$, $\xioff$.} \label{line:alg2-1}
 			\STATE {\bfseries Initialization:} {$\widehat{V}_{H+1}=0$.}\label{line:alg2-2}
 			\FOR{$h=H,H-1,\ldots,1$}\label{line:alg2-3}
 			\STATE $\widehat{w}_h\hspace{-0.02in}=\hspace{-0.02in}{\Lambda}_h^{-1}\sum_{\tau=1}^{\Noff}\hphi_h(s_h^\tau,a_h^\tau)\widehat{V}_{h+1}(s_{h+1}^\tau)$ where $\Lambda_h\hspace{-0.02in}=\hspace{-0.02in}\sum_{\tau=1}^{\Noff}\hphi_h(s_h^\tau,a_h^\tau)\hphi_h(s_h^\tau,a_h^\tau)^\top\hspace{-0.02in}+\hspace{-0.02in}\lambda I_d$.\label{line:alg2-4} 
 			\STATE $\widehat{Q}_{\hspace{-0.01in}h}\hspace{-0.01in}(\hspace{-0.01in}\cdot,\hspace{-0.01in}\cdot\hspace{-0.01in})\hspace{-0.02in}=\hspace{-0.02in}\min\{\hspace{-0.01in}r_{\hspace{-0.01in}h}\hspace{-0.01in}(\hspace{-0.01in}\cdot,\hspace{-0.01in}\cdot\hspace{-0.01in})\hspace{-0.02in}+\hspace{-0.02in}\hphi_h\hspace{-0.01in}(\hspace{-0.01in}\cdot,\hspace{-0.01in}\cdot\hspace{-0.01in})^{\hspace{-0.03in}\top}\hspace{-0.03in}\widehat{w}_h\hspace{-0.02in}-\hspace{-0.02in}\Gamma_{\hspace{-0.01in}h}\hspace{-0.01in}(\hspace{-0.01in}\cdot,\hspace{-0.01in}\cdot\hspace{-0.01in}),\hspace{-0.01in}1\hspace{-0.01in}\}^{\hspace{-0.01in}+}$, where $\Gamma_{\hspace{-0.01in}h}\hspace{-0.01in}(\hspace{-0.01in}\cdot,\hspace{-0.01in}\cdot\hspace{-0.01in})\hspace{-0.02in}=\hspace{-0.02in}\xioff\hspace{-0.03in}+\hspace{-0.02in}\beta[\hphi_h\hspace{-0.01in}(\hspace{-0.01in}\cdot,\hspace{-0.01in}\cdot\hspace{-0.01in})^{\hspace{-0.03in}\top}\hspace{-0.03in}\Lambda_{\hspace{-0.01in}h}^{\hspace{-0.01in}-\hspace{-0.01in}1}\hphi_h\hspace{-0.01in}(\hspace{-0.01in}\cdot,\hspace{-0.01in}\cdot\hspace{-0.01in})]^{\hspace{-0.01in}1\hspace{-0.01in}/\hspace{-0.01in}2}\hspace{-0.01in}$. \label{line:alg2-5}
 			\STATE{$\widehat{V}_h(\cdot)=\widehat{Q}_h(\cdot,\widehat{\pi}_h(\cdot))$, where $\widehat{\pi}_h(\cdot)=\arg\max_{\pi_h}\widehat{Q}_h(\cdot,\widehat{\pi}_h(\cdot))$.}\label{line:alg2-6}
 			\ENDFOR\label{line:alg2-7}
 			\STATE{\textbf{Output:} $\{\widehat{\pi}_h\}_{h=1}^H$.}\label{line:alg2-8}
 \end{algorithmic}
\end{algorithm}

\noindent{\bf Approximate Bellman update with general reward function.}
We construct an estimated Bellman update 
$\widehat{\mathbb{B}}_h\widehat{V}_{h+1}$ based on the dataset $\Doff$ to approximate $\mathbb{B}_h\widehat{V}_{h+1}$. Here $\widehat{V}_{h+1}$ is an estimated value function constructed in the previous iteration based on $\Doff$. Since our reward function is more general and not necessarily linear in the feature, we choose to parameterized $P_h^{(\star,T+1)}V^*_{h+1}$ instead of $Q_h^*(s,a)=r_h(s,a)+(P_h^{(\star,T+1)}V_{h+1}^*)(s,a)$ due to the linear structure of $P_h^{(\star,T+1)}$. 
Although the ground truth representation is unknown, the agent is able to use the estimated representation learned from the upstream learning to parameterize the linear term $(P_h^{(\star,T+1)}V_{h+1}^*)(s,a)$. Then, the coefficient $w_h^*$ can be estimated by solving the following regularized least squares problem:
\vspace{-0.03in}
\begin{align}
\textstyle\widehat{w}_h=\argmin_{w\in\mathbb{R}^d} \lambda\norm{w}_2^2+\sum_{\tau=1}^{\Noff}\left[w^\top\widehat{\phi}_h(s_h^\tau,a_h^\tau)-\widehat{V}_{h+1}(s_{h+1}^\tau)\right]^2,
\end{align}
\vspace{-0.03in}which has the closed-from solution $\widehat{w}_h=\Lambda_h^{-1}\sum_{\tau=1}^{\Noff}\widehat{\phi}_h(s_h^\tau,a_h^\tau)\widehat{V}_{h+1}(s_{h+1}^\tau)$ with $\Lambda_h=\lambda I+\sum_{\tau=1}^{\Noff}\widehat{\phi}_h(s_h^\tau,a_h^\tau)\widehat{\phi}_h(s_h^\tau,a_h^\tau)^\top$. We obtain an estimate $(\widehat{\mathbb{B}}_h\widehat{V}_{h+1})(\cdot,\cdot)=r_h(\cdot,\cdot)+\widehat{w}_h^\top\widehat{\phi}_h(\cdot,\cdot)$.

\noindent{\bf Design pessimism for incorporating upstream misspecification of representation.} Pessimism is realized by subtracting an uncertainty metric $\Gamma_h$ from $\widehat{Q}_h$ (line~\ref{line:alg2-5}). The uncertainty metric $\Gamma_h$ quantifies the Bellman update error $|(\widehat{\mathbb{B}}_h\widehat{V}_{h+1}-\mathbb{B}_h\widehat{V}_{h+1})(s,a)|$ at step $h$. 
Our selection of $\Gamma_h$ is in the form of  $\xioff+\beta(\widehat{\phi}_h\Lambda_h^{-1}\widehat{\phi}_h)^{1/2}$, which captures both the \textbf{approximation error} (first term) and the \textbf{estimation error} in the Bellman update (second term). Both of the above errors are affected by the misspecification of representation from upstream learning. 
Intuitively, $m:=(\widehat{\phi}_h\Lambda_h^{-1}\widehat{\phi}_h)^{-1}$ represents the effective number of samples the agent has explored along the $\widehat{\phi}$ direction, and the penalty term $\beta/\sqrt{m}$ represents the uncertainty along the $\widehat{\phi}$ direction. By choosing a proper value for $\beta$, we can prove that with high probability, the uncertainty metric $\Gamma_h$ always upper bounds the Bellman update error.

\noindent{\bf Select greedy policy.} Line~\ref{line:alg2-6} executes a greedy policy $\widehat{\pi}_h$ to maximize $\widehat{Q}_h$, which will be the output of the algorithm after the backward iteration over steps completes.

\subsection{Sample complexity for downstream offline RL}\label{sec:off-result}
We first introduce the following coverage assumption for the estimated feature qualifying the offline dataset,
which can be easily checked in practice.
\begin{assumption}[Feature coverage]\label{assumption:offlinedata}
There exists absolute constant $\kappa_{\rho}$ such that
\begin{align*}
\forall h\in[H],\quad \lambda_{\min}(\Sigma_h)\geq \kappa_{\rho}, \mbox{ where } \Sigma_h=\mathbb{E}_{\rho}[\widehat{\phi}_h(s_h,a_h)\widehat{\phi}_h(s_h,a_h)^\top|s_1=s].
\end{align*}
\end{assumption}
Now we are ready to provide our main result of downstream offline RL in the following theorem and defer the proof to \Cref{sec: downstream offline}.

\begin{theorem}\label{thm:offline}
Under \Cref{assumption: reachability,assumption: finite measurement,assumption: smooth TV distance,assumption: Linear combination}, for any $\delta \in (0,1)$, if we set $\lambda=1$ and $\beta=O\left(d\sqrt{\iota}\right.$ $\left.+\sqrt{d\Noff}\xioff+\sqrt{p\log \Noff}\right)$, where $\iota=O\left(\log\left(pdH\Noff\xioffm/\delta\right)\right)$, then with probability at least $1-\delta$, the suboptimality gap of \Cref{alg: Offline RL} is at most \vspace{-0.04in}
\begin{align}
V_{P^{(*,T+1)},r}^*-V_{P^{(*,T+1)},r}^{\widehat{\pi}}\leq 2H\xioff+2\beta\sum_{h=1}^{H}\mathbb{E}_{\pi^*}\left[\norm{\widehat{\phi}_h(s_h,a_h)}_{\Lambda_h^{-1}}\middle|s_1=s\right]. \label{eqn:offline-result-1}
\end{align} 
\vspace{-0.04in}If additionally \Cref{assumption:offlinedata} holds, and the sample size satisfies $\Noff\geq 40\kappa_{\rho}\cdot\log(4dH/\delta)$, then we have
\vspace{-0.04in}
\begin{align}
V_{P^{(*,T+1)},r}^*-&V_{P^{(*,T+1)},r}^{\widehat{\pi}}\nonumber \\
&\leq\textstyle O\left(\kappa_{\rho}^{-1/2}\hspace{-0.01in}Hd^{1/2}\xioff\hspace{-0.02in}+\hspace{-0.01in}\kappa_{\rho}^{-1/2}\hspace{-0.01in}Hd\hspace{-0.03in}\sqrt{\frac{\log\left(pdH\Noff\xioffm/\delta\right)}{\Noff}} \hspace{-0.02in}+\hspace{-0.02in}\kappa_{\rho}^{-1/2}H\sqrt{\frac{p\log{\Noff}}{\Noff}} \right). \label{eqn:offline-result-2}
\end{align}

Further, if the linear combination misspecification error in \Cref{assumption: Linear combination} satisfies $\xi=\tilde{O}(\sqrt{d}/\sqrt{\Noff})$, then with at most  $\widetilde{O}\left({H^3dK^2T\Noff}+\left(H^3d^3K+H^5K^2+H^5d^2K\right)T^2\Noff+\frac{H^5K^3T\Noff}{d^2}\right)$ trajectories collected in upstream, we have
\begin{align}\label{eq:offlinelinear}
V_{P^{(*,T+1)},r}^*-&V_{P^{(*,T+1)},r}^{\widehat{\pi}}\leq \widetilde{O}\left(\kappa_{\rho}^{-1/2}\Noff^{-1/2}H\max\{d,\sqrt{p}\}\right). 
\end{align}
\end{theorem}
\vspace{-0.04in}To the best of our knowledge, \Cref{thm:offline} provides the first suboptimality gap for offline RL where the representation has an approximation error. The first result in \Cref{thm:offline} characterizes the suboptimality gap if the offline data is rolled out from a behavior policy\footnote{The result remains valid if offline data is compliant with MDP $\mathcal{M}^{T+1}$ (Assumption 2.2 in \cite{Jin2021IsPP}).}. 
If the offline data further satisfies the feature coverage assumption, we are able to upper bound the suboptimality gap by three terms in \Cref{eqn:offline-result-2}.
Such a result can be further simplified as $\widetilde{O}\left(Hd^{1/2}\xioff+\Noff^{-1/2}H\max\{d,\sqrt{p}\}\right)$, where the first term is a constant error that arises from the \textbf{approximation error} in Bellman update. The second term diminishes as $\Noff$ grows to infinity, and it arises from the \textbf{estimation error} of $\widehat{w}_h$ 
in regularized linear regression. The dependence on $d$ and $\sqrt{p}$ arises from the model complexity of the linear function approximation class when we seek concentration of the estimated value function. 

Compared to standard linear MDP with the knowledge of ground truth representation~\cite{Jin2021IsPP}, our suboptimality gap contains an additional upstream misspecification error (i.e., $\xioff$). The remaining term yields 
an order of $\widetilde{O}(Hd\Noff^{-1/2})$ which matches that in~\cite{Jin2021IsPP} if we specialize our reward to be linear with $p=O(d)$.

\Cref{thm:offline} also demonstrates that our downstream offline RL benefits from the learned representation in upstream. 
As shown in \Cref{eq:offlinelinear}, if the linear combination misspecification error $\xi$ is small enough, and the number of trajectories collected in upstream learning is sufficient large, then the suboptimality gap is dominated by the term $\widetilde{O}\left(\Noff^{-1/2}H\max\{d,\sqrt{p}\}\right)$, which improves $\widetilde{O}(\Noff^{-1/2}H^2d^2)$\footnote{We convert their results in discounted setting to our episodic MDP setting by replacing $1/(1-\gamma)=\Theta(H)$ and from PAC bound to suboptimality gap.} of REP-LCB \citep{uehara2021representation} under low-rank MDP with unknown representation, with better dependence on $H$ and $d$ due to the benefit of upstream representation learning.

\subsection{Algorithm for downstream online RL}\label{sec:onlinealg}
\begin{algorithm}
 		\begin{algorithmic}[1]
 			\caption{Downstream Online RL}\label{alg: Online RL}
 			\STATE {\bfseries Input:} {Feature $\hphi$, parameters $\lambda$, $\beta_n$.}
 			\FOR{$n=1,\ldots,N$}
 			\STATE Receive the initial state $s_1^n=s_1$.
 			\FOR{$h=H,\ldots,1$}
 			\STATE $\Lambda_h^n = \sum_{\tau=1}^{n-1}\hphi_h(s_h^\tau,a_h^\tau)\hphi_h(s_h^\tau,a_h^\tau)^\top+\lambda I_d$.
 			
 			\STATE $w_h^n={(\Lambda_h^n)}^{-1}\sum_{\tau=1}^{n-1}\hphi_h(s_h^\tau,a_h^\tau)V_{h+1}^n(s_{h+1}^\tau)$.
 			\STATE $Q_h^n(\cdot,\cdot)=\min\left\{r_h(\cdot,\cdot)+\hphi_h(\cdot,\cdot)^\top w_h^n+\beta_n\norm{\hphi_h(\cdot,\cdot)}_{(\Lambda_h^n)^{-1}},1\right\}$, $V_h^n(\cdot)=\max_a Q_h^n(\cdot,a)$. \label{line: online-7}
 			\ENDFOR
 			\STATE Let $\pi^n$ be the greedy policy induced by $\{Q_h^n\}_{h=1}^H$, i.e., $\pi_h^n(\cdot)={\arg\max}_{a\in \Ac}Q_h^n(\cdot,a)$
 			\FOR{$h=1,\ldots,H$}	
 			\STATE Take action $a_h^n=\pi^n(s_h^n)$, and observe $s_{h+1}^n$.
 			\ENDFOR
 			\ENDFOR
 		\STATE \textbf{Output:} $\pi^1,\ldots,\pi^n$.
 		\end{algorithmic}
 	\end{algorithm}
We present our downstream online RL algorithm in \Cref{alg: Online RL}, where the agent is allowed to interact with the new task environment for policy optimization, and utilizes the learned feature from upstream to implement linear approximation for state-action value function with traditional UCB type of algorithms.

The algorithm consists of three iterations, one outer iteration and two inner backward iterations over steps. \textbf{Construct optimistic action value function}: In each inner backward iteration, the agent first constructs the empirical feature covariance matrices and estimates the weights $\widehat{w}$ via the trajectories collected in previous iterations. Then, the empirical covariance matrices and the weights $\widehat{w}$ are used to approximate Bellman update and construct an optimistic action value function (line \ref{line: online-7}). \textbf{Learn optimistic policy}: Next, in the forward inner iteration, the agent learns a greedy optimistic policy and uses it to explore and collect a new trajectory. After the outer iteration completes, the greedy policies are output by \Cref{alg: Online RL}.

\Cref{alg: Online RL} differs from traditional UCB type of algorithms under linear MDP (i.e., LSVI-UCB in \cite{jin2020provably}) as follows. First, \cite{jin2020provably} assumes the true representation $\phi^\star$ is known with a small model misspecification error, whereas our representation is obtained from upstream, and hence the upstream learning error affects the learning accuracy here. Second, \Cref{alg: Online RL} is applicable to reward functions from a general function class $\Rc$, which do not need to be linear with respect to the representation. Both of the above differences will require additional techniques in our analysis of the performance guarantee to be presented in \Cref{sec:on-result}.

\subsection{Sample complexity for downstream online RL}\label{sec:on-result}

To analyze our downstream online \Cref{alg: Online RL}, we need to characterize how the misspecification error of the learned representation from upstream affects the downstream suboptimality gap. We further need to handle the general reward rather than the linear reward. Both make our analysis different from that in \cite{jin2020provably} which directly assumes an approximation of the representation and works with linear reward functions.
We next provide our main result of downstream online RL in the following theorem and defer the proof to \Cref{sec: downstream online}.
\begin{theorem}\label{theorem: downstream online RL sample complexity}
Under \Cref{assumption: reachability,assumption: finite measurement,assumption: smooth TV distance,assumption: Linear combination}, fix $\delta \in (0,1)$. If we set $\lambda=1$ and $\beta_n=O\left(d \sqrt{\iota_n}+\sqrt{nd}\xioff\right.$ $\left.+\sqrt{p\log n}\right)$ where $\iota_n=\log\left(2pdnH\xioffm/\delta\right)$. Let $\tilde{\pi}$ be the uniform mixture of $\pi^1,\ldots,\pi^n$. Then with probability at least $1-\delta$, the suboptimality gap of \Cref{alg: Online RL} satisfies 
\begin{align}
V_{P^{(*,T+1)},r}^*-V_{P^{(*,T+1)},r}^{{\tilde{\pi}}}=\widetilde{O}\left({Hd\xioff}+Hd^{1/2}\Non^{-1/2}\max\{d,\sqrt{p}\}\right) \label{ineq: online-regret}.
\end{align}
Furthermore, if the linear combination misspecification error $\xi$ in \Cref{assumption: Linear combination} satisfies $\tilde{O}(\sqrt{d}/\sqrt{\Non})$, then with at most  $\widetilde{O}\left({H^3dK^2T\Non}+\left(H^3d^3K+H^5K^2+H^5d^2K\right)T^2\Non+\frac{H^5K^3T\Non}{d^2}\right)$ trajectories collected in upstream, $\xioff$ reduces to $\tilde{O}(\sqrt{d}/\sqrt{\Non})$ and the second term in \Cref{ineq: online-regret} dominates so that the 
the suboptimality gap is bounded as
\begin{align}
V_{P^{(*,T+1)},r}^*-V_{P^{(*,T+1)},r}^{{\tilde{\pi}}}=\widetilde{O}\left(Hd^{1/2}\Non^{-1/2}\max\{d,\sqrt{p}\}\right). \label{ineq: simplified suboptimality gap}
\end{align}
\end{theorem}

We first explain the two terms in the suboptimality gap in \Cref{ineq: online-regret} as follows. The first term captures how the \textbf{approximation error} of using the estimated representation from upstream affects the suboptimality gap, and becomes small if such an approximation error is well controlled.
The second term arises from the \textbf{estimation error} of $\widehat{w}_h$ in constructing the optimistic action value function and vanishes as the number $\Non$ of samples becomes large. The dependence on $d$ and $\sqrt{p}$ arises from the model complexity of the linear function approximation class when we seek a concentration of the estimated value function. Thus, if the learned representation in upstream approximates the ground truth well 
and $\xioff$ reduces to $\tilde{O}(\sqrt{d}/\sqrt{\Non})$, then the suboptimality gap is dominated by the second term $\widetilde{O}\left(Hd^{1/2}\Non^{-1/2}\max\{d,\sqrt{p}\}\right)$ as in \Cref{ineq: simplified suboptimality gap}. 
This matches the suboptimality gap $\widetilde{O}\left(\sqrt{H^2d^3/\Non}\right)$\footnote{We rescale the result in \cite{jin2020provably} under the condition of $\sum_{h=1}^{H}r_h \leq 1$ and transform the regret bound to the suboptimality gap for a fair comparison.} in linear MDP~\citep{jin2020provably} if we specialize our reward to be linear and let $p=O(d)$.

\Cref{theorem: downstream online RL sample complexity} also demonstrates that our downstream online RL benefits from the learned representation from upstream.  Compared to the state-of-the-art work under the low-rank MDP  with unknown representation \citep{uehara2021representation} which has the suboptimality gap being bounded as $\widetilde{O}(\sqrt{H^5d^4K^2/\Non})$\footnote{We convert the results in discounted setting in \citep{uehara2021representation} to our episodic MDP setting by replacing $1/(1-\gamma)=\Theta(H)$ and transforming their PAC bound into suboptimality gap.}, 
our downstream online RL has an improved suboptimality gap in terms of the dependence on $H,d$ and $K$ benefited from the upstream representation learning.

\section{Conclusion}\label{sec: conclusion}
In this paper, we showed that a shared representation of low rank MDPs can be learned more efficiently and in a transferred manner. We provided upper bounds on the sample complexity of upstream reward-free multitask RL, and on the suboptimality gap of both downstream offline and online RL settings with the help of the learned representation from upstream. Our results show that multitask representation learning is provably more sample efficient than learning each task individually and capture the benefit of employing the learned representation from upstream to learn a near-optimal policy of a new task in downstream that shares the same representation. As future work, we note that our analysis of the downstream RL depends on some assumptions that connect the upstream and downstream MDPs. It is interesting to explore whether these assumptions can be further relaxed. 

\bibliography{MTRL}
\bibliographystyle{apalike}

 	\newpage
 	\appendix

\noindent\textbf{\LARGE Supplementary Materials}\newline
 	
\section{Proof of \Cref{theorem: Upstream sample complexity}}\label{Appd: proof of thm1}
We summarize frequently used notations in the following list.
\begin{align*}
\begin{array}{ll}
\Pi^n_t &\Uc(\pi^1_t,...,\pi^{n-1}_t)	 \\
f_h^{(n,t)}(s,a) &\|\hP_h^{(n,t)}(\cdot|s,a) - P^{(\star,t)}_h(\cdot|s,a)\|_{TV} \\
\hb_{h}^{(n,t)}(s_{h},a_{h}) &\min\left\{\tap_n\left\|\hphi_{h}^{(n)}(s_{h},a_{h})\right\|_{(\widehat{U}^{(n,t)}_{h})^{-1}},B\right\} \\
PCV\left(\hP^{(n,t)},\hb_{h}^{(n,t)},\pi_t;T\right) &\sum_{h=1}^{H-1}\sqrt{\sum_{t=1}^T\left\{\mathop{\Eb}_{s_{h} \sim (\hP^{(n,t)}, \pi_t) \atop a_{h}\sim\pi_t} \left[\hb_{h}^{(n,t)}(s_h,a_h)\right]\right\}^2} \\
\lambda_n & O(d\log(|\Phi|nTH/\delta)) \\
U_{h,\phi}^{(n,t)} &n\Eb_{s_h\sim(P^{  {(\star,t)}  },\Pi_t^n),a_h\sim \Uc(\Ac)}\left[\phi(s_h,a_h)(\phi(s_h,a_h))^\top\right] + \lambda_n I_d\\
W_{h,\phi}^{(n,t)} &n\Eb_{(s_h,a_h)\sim(P^{{(\star,t)}  },\Pi_t^n)}\left[\phi(s_h,a_h)(\phi(s_h,a_h))^\top\right] + \lambda_n I_d \\
\zeta_h^{(n,t)} &\mathop{\Eb}_{s_{h-1}\sim (P^{(\star,t)},\Pi_t^n)\atop {a_{h-1},a_h\sim \Uc(\Ac)
\atop s_h\sim P^{(\star,t)}(\cdot|s_{h-1},a_{h-1})}}\left[f_h^{(n,t)}(s_h,a_h)^2\right], \quad h\geq 2 \\
{\zeta_1^{(n,t)}} &{\mathop{\Eb}_{s_1\sim(P^{(\star,t)},\pi_t) \atop a_1 \sim \Uc(\Ac)} \left[ 
f_1^{(n,t)}(s_1,a_1)^2\right]} \\
\zeta_n  &\frac{2\log\left(2|\Phi||\Psi|^T nH/\delta\right)}{n} \\
\alpha_h^{(n,t)} &\sqrt{nK\zeta_h^{(n,t)}+nK\zeta_{h-1}^{(n,t)}+\lambda_n d}, \quad  h \geq 2 \\
\alpha_n &\sqrt{2nK\zeta_n+\lambda_nT d} \\
B_n &\sqrt{T+\frac{2K}{d^2}+\frac{2KT\ln(|\Psi|)}{d^2\ln(n|\Phi|T/\delta)}} \\
B & 2\sqrt{T+K/d^2} \\
\beta_1,\beta_2 & \mbox{some absolute constants} \\
\tap_n &\frac{\alpha_n}{\beta_1}
\end{array}
\end{align*}
For consistency of notation, in some generic cases, we write $s_h \sim (P, \pi)$ for all $h$ for the entire episode. But for such cases, $s_1$ should still be understood as a fixed state and independent of transition model $P$ and policy $\pi$.
For example in $\zeta_1^{(n,t)}$ as defined above, the expectation is taken with respect to $a_1\sim\mathcal{U}(\mathcal{A})$ and fixed state $s_1$. Furthermore, $\Uc(\pi^1_t,...,\pi^{n-1}_t)$ denotes uniform mixture of previous $n-1$ exploration policies $\pi^1_t,...,\pi^{n-1}_t$.

\textbf{Proof Overview:} The proof of \Cref{theorem: Upstream sample complexity} consists of three main steps, a final sample complexity step, and several supporting lemmas. 
\textbf{Step 1:} 
We develop a new upper bound on model estimation error for each task, which captures the advantage of joint MLE model estimation over single-task learning, as shown in \Cref{prop1: Bounded TV}.
\textbf{Step 2:} 
We establish the PCV as an uncertainty measure that captures the difference between the estimated and ground truth models. This justifies using such a PCV as a guidance for further exploration, as shown in \Cref{prop2: total value difference}.
 \textbf{Step 3:} We show that the summation of the PCVs over all iterations is sublinear with respect to the total number $N_u$ of iterations, as shown in \Cref{prop3: Bound of summation of exploration-driven reward function}, which further implies polynomial efficiency of REFUEL in learning the models. \textbf{Complexity characterization:} Combining the three steps together with a contradiction argument yields the final sample complexity.
 
 We provide details for the three main steps in \Cref{app:step1}-\Cref{app:step3}, the complexity characterization in \Cref{app:complexity}, and the supporting lemmas in \Cref{app:supportinglemma}.

\subsection{Step 1: A new upper bound on model estimation error}\label{app:step1}

We develop a new upper bound on model estimation error for each task, which captures the advantage of joint MLE model estimation over single-task learning, as shown in \Cref{prop1: Bounded TV}.
\begin{proposition}\label{prop1: Bounded TV}
For any $n \in [N_u]$, task $t$, policy $\pi_t$ and reward $r^t$, for all $h\geq 2$, we have
\begin{align}
\textstyle \mathop{\Eb}_{s_h\sim(\hP^{(n,t)},\pi_t) \atop a_h \sim \pi_t} & \left[f_h^{(n,t)}(s_h,a_h)\right] \nonumber \\
&\textstyle {\leq} \mathop{\Eb}_{s_{h-1} \sim (\hP^{(n,t)}, \pi_t) \atop a_{h-1}\sim\pi_t} \left[\min\left\{\alpha_h^{(n,t)}\left\|\hphi_{h-1}^{(n)}(s_{h-1},a_{h-1})\right\|_{(U^{(n,t)}_{h-1,\hphi^{(n,t)}})^{-1}},1\right\}\right],
\end{align}
and for $h=1$, we have
\begin{align}	    
\mathop{\Eb}_{a_1 \sim \pi_t} \left[ f_1^{(n,t)}(s_1,a_1)\right] {\leq} \sqrt{K\zeta_1^{(n,t)}}.
\end{align}
\end{proposition}
\begin{proof}
For $h=1$,
\begin{align*}
\mathop{\Eb}_{ a_1 \sim \pi_t} \left[ f_1^{(n,t)}(s_1,a_1)\right] 
\overset{(\romannumeral1)}{\leq} \sqrt{\mathop{\Eb}_{ a_1 \sim \pi_t} \left[ f_1^{(n,t)}(s_1,a_1)^2\right]} \overset{(\romannumeral2)}{\leq} \sqrt{K\zeta_1^{(n,t)}},
\end{align*}
where $(\romannumeral1)$ follows from Jensen's inequality, and $\RM{2}$ follows from importance sampling.
 		
Then for $h \geq 2$, we derive the following bound:
\begin{align*}
&\mathop{\Eb}_{s_h\sim(\hP^{(n,t)},\pi_t) \atop a_h \sim \pi_t}  \left[ 
f_h^{(n,t)}(s_h,a_h)\right]\\
& \quad \overset{(\romannumeral1)}{\leq}\mathop{\Eb}_{s_{h-1} \sim (\hP^{(n,t)}, \pi_t) \atop a_{h-1}\sim\pi_t} \left[\left\|\hphi_{h-1}^{(n)}(s_{h-1},a_{h-1})\right\|_{(U^{(n,t)}_{h-1,\hphi^{(n,t)}})^{-1}}\times \right.
\\
& \quad {
\scriptstyle \left. 
\sqrt{nK \mathop{\Eb}_{s_{h-1}\sim(P^{(\star,t)},\Pi_t^n)\atop{a_{h-1},a_h\sim \Uc(\Ac)\atop s_h\sim P^{(\star,t)}(\cdot|s_{h-1},a_{h-1})}}\left[f_h^{(n,t)}(s_h,a_h)^2\right] + \lambda_n d  + nK\mathop{\Eb}_{s_{h-2}\sim(P^{(\star,t)},\Pi_t^n)\atop {a_{h-2},a_{h-1}\sim \Uc(\Ac) \atop s_{h-1}\sim P^{(\star,t)}(\cdot|s_{h-2},a_{h-2})}}\left[f_{h-1}^{(n,t)}(s_{h-1},a_{h-1})^2\right]    }\right]
} \\
&\quad \overset{(\romannumeral2)}{\leq} \mathop{\Eb}_{s_{h-1} \sim (\hP^{(n,t)}, \pi_t) \atop a_{h-1}\sim\pi_t} \left[\sqrt{nK\zeta_h^{(n,t)}+nK\zeta_{h-1}^{(n,t)}+\lambda_n d}\left\|\hphi_{h-1}^{(n)}(s_{h-1},a_{h-1})\right\|_{(U^{(n,t)}_{h-1,\hphi^{(n,t)}})^{-1}}  
\right]\\
&\quad = \mathop{\Eb}_{s_{h-1} \sim (\hP^{(n,t)}, \pi_t) \atop a_{h-1}\sim\pi_t} \left[\alpha_h^{(n,t)}\left\|\hphi_{h-1}^{(n)}(s_{h-1},a_{h-1})\right\|_{(U^{(n,t)}_{h-1,\hphi^{(n,t)}})^{-1}}  
\right],
\end{align*}
where $(\romannumeral1)$ follows from \Cref{lemma:Step_Back} and because $|f_h^{(n,t)}(s_h,a_h)|\leq 1$, the first term inside the square root follows from the definition of $U_{h-1,\widehat{\phi}^{(n,t)}}^{(n,t)} $, the third term inside the square root follows from importance sampling and $(\romannumeral2)$ follows from \Cref{lemma: multitask MLE Guarantee}.

The proof is completed by noting that $|f_h^{(n,t)}(s_h,a_h)|\leq 1$.
\end{proof}
The following corollary is a direct application of \Cref{prop1: Bounded TV}.
\begin{corollary}[{Bounded} difference of value functions]\label{coro1:bounded difference of summation}
For $n\in[N]$, any task $t$, policy $\pi_t$ and reward $r^t$, we have
\begin{align}
&\left| V_{P^{(\star,t)},r^t}^{\pi_t}-V_{\hP^{(n,t)},r^t}^{\pi_t}\right|\nonumber\\
& \quad \leq \sum_{h=1}^{H-1} \mathop{\Eb}_{s_{h} \sim (\hP^{(n,t)}, \pi_t) \atop a_{h}\sim\pi_t} \left[\min\left\{\alpha_h^{(n,t)}\left\|\hphi_{h}^{(n)}(s_{h},a_{h})\right\|_{(U^{(n,t)}_{h,\hphi^{(n,t)}})^{-1}}, 1 \right\}\right]+\sqrt{K\zeta_1^{(n,t)}} \label{ineq: difference of value function bound}.
\end{align}
\end{corollary}
\begin{proof}
For task $t$, we have
\begin{align}			&\left|V_{P^{(\star,t)},r^t}^{\pi_t} - V_{\hP^{(n,t)},r^t}^{\pi_t}\right|\nonumber\\
&\quad\overset{\RM{1}}{=} \left| \sum_{h=1}^H \mathop{\Eb}_{s_h\sim(\hP^{(n,t)},\pi_t) \atop a_h \sim \pi_t} \left[ (P^{(\star,t)}_{h} - \hP_{h}^{(n,t)})V^{\pi_t}_{h+1,P^{(\star,t)},r^t}(s_h,a_h)\right]\right|\nonumber\\
&\quad \overset{\RM{2}}{\leq} \sum_{h=1}^H \mathop{\Eb}_{s_h\sim(\hP^{(n,t)},\pi_t) \atop a_h \sim \pi_t} \left[ {f_h^{(n,t)}(s_h,a_h)}\right]\nonumber\\
&\quad \overset{\RM{3}}{\leq}
\sum_{h=2}^H \mathop{\Eb}_{s_{h-1} \sim (\hP^{(n,t)}, \pi_t) \atop a_{h-1}\sim\pi_t} \left[\min\left\{\alpha_h^{(n,t)}\left\|\hphi_{h-1}^{(n)}(s_{h-1},a_{h-1})\right\|_{(U^{(n,t)}_{h-1,\hphi^{(n,t)}})^{-1}}, 1 \right\}\right]+\sqrt{K\zeta_1^{(n,t)}} \nonumber,
\end{align}
where $\RM{1}$ follows from \Cref{lemma: Simulation}, $\RM{2}$ follows from the fact that $V^{\pi_t}_{P^{(\star,t)},r^t}\leq 1$, and $\RM{3}$ follows from \Cref{prop1: Bounded TV}.
\end{proof}
 		
\subsection{Step 2: PCV as a new uncertainty metric}\label{app:step2}
In \Cref{prop1: Bounded TV} and \Cref{coro1:bounded difference of summation}, we provide upper bounds on both the total variation distance between the learned model and the true model, and the difference of value functions under any policy with arbitrary rewards. For each single task, such upper bounds measure the uncertainty of model estimation and guide exploration in the next iteration. In multitask RL, although these upper bounds also hold for each individual source task under any policy and reward even with tighter terms $\alpha_h^{(n,t)}$, the RHS of \Cref{ineq: difference of value function bound} cannot be used to guide exploration for each task $t$ individually as in single-task RL because $\zeta_h^{(n,t)}$ in $\alpha_h^{(n,t)}$ as a joint MLE Guarantee is unknown individually for each task $t$. 

This motivate us to jointly consider all the explorations of all source tasks. As we establish below in \Cref{prop2: total value difference} that our defined new notion of PCV used in \Cref{alg: Upstream} is a known upper bound for both the summation of the difference of value function and the summation of total variation distance over $T$ tasks, and can therefore serve as {an uncertainty quantifier} to guide exploration.

\begin{proposition}[PCV as uncertainty metric]\label{prop2: total value difference}
 Given any $\delta \in
\left(0,|\Psi|^{-\min\{T,\frac{K}{d^2}\}}\right)$, for any $n$, policy $\pi_t$ and reward function $r^t$ of task $t\in[T]$, set $\lambda_n=O(d\log(|\Phi|nTH/\delta))$. Then with probability $1-\delta$, we have 
\begin{align*}
&\sum_{t=1}^T\left|V_{P^{(\star,t)},r^t}^{\pi_t} - V_{\hP^{(n,t)},r^t}^{\pi_t}\right|{\leq} PCV\left(\hP^{(n,t)},\hb_{h}^{(n,t)},\pi_t;T\right)+\sqrt{KT\zeta_n},\\
&\sum_{t=1}^{T}\sum_{h=1}^H \mathop{\Eb}_{s_h\sim(\hP^{(n,t)},\pi_t) \atop a_h \sim \pi_t} \left[ {f_h^{(n,t)}(s_h,a_h)}\right] \leq PCV\left(\hP^{(n,t)},\hb_{h}^{(n,t)},\pi_t;T\right)+\sqrt{KT\zeta_n}.\\
\end{align*}
\end{proposition}
\begin{proof}
Following from the fact that $\sum_{t=1}^T\zeta_h^{(n,t)}\leq \zeta_n$, for all $h$ (see \Cref{lemma: multitask MLE Guarantee}), and the definitions of $\alpha_h^{(n,t)}$ and $\alpha_n$ for $h\geq 2$, with probability at least $1-\delta/2$, we have
\begin{align}
\sum_{t=1}^{T}(\alpha_h^{(n,t)})^2\leq nK\sum_{t=1}^{T}\left(\zeta_h^{(n,t)}+\zeta_{h-1}^{(n,t)}\right)+\lambda_n Td\leq 2nK\zeta_n+\lambda_n Td=\alpha_n^2. \label{eqn:prop2:+1}
\end{align}
Moreover, the ratio between $\alpha_n$ and $\alpha_h^{(n,t)}$ is upper bounded as follows:
\begin{align}
\frac{\alpha_n}{\alpha_h^{(n,t)}}&=\sqrt{\frac{2nK\zeta_n+\lambda_nT d}{nK\zeta_h^{(n,t)}+nK\zeta_{h-1}^{(n,t)}+\lambda_nd}} \nonumber \\
&\leq\sqrt{\frac{2nK\zeta_n+\lambda_nT d}{\lambda_nd}}  \nonumber \\
& =\sqrt{T+\frac{2K\ln(n|\Phi||\Psi|^T/\delta)}{d^2\ln(n|\Phi|T/\delta)}}\nonumber \\
& =\sqrt{T+\frac{2K}{d^2}+\frac{2KT\ln(|\Psi|)}{d^2\ln(n|\Phi|T/\delta)}} = B_n \leq B,  \label{eqn:prop2:+2}
\end{align}
where the last inequality follows because $B_n \leq B=2\sqrt{T+K/d^2}$ if $\delta \in { (0, |\Psi|^{-\frac{\min\{T,K\}}{d^2}} ) }$.

Following from \Cref{coro1:bounded difference of summation}, we take the summation over all source tasks and obtain
\begin{align*}
 		&\sum_{t=1}^T\left|V_{P^{(\star,t)},r^t}^{\pi_t} - V_{\hP^{(n,t)},r^t}^{\pi_t}\right|\nonumber\\
 		&  \scriptstyle\leq \sum_{h=2}^H \sum_{t=1}^T\mathop{\Eb}_{s_{h-1} \sim (\hP^{(n,t)}, \pi_t) \atop a_{h-1}\sim\pi_t} \left[\min\left\{\alpha_h^{(n,t)}\left\|\hphi_{h-1}^{(n)}(s_{h-1},a_{h-1})\right\|_{(U^{(n,t)}_{h-1,\hphi^{(n,t)}})^{-1}},1\right\}\right]+\sum_{t=1}^T\sqrt{K\zeta_1^{(n,t)}}\\
 		& \scriptstyle= \sum_{h=2}^H \sum_{t=1}^T\hspace{-0.02in}\alpha_h^{(n,t)}\mathop{\Eb}_{s_{h-1} \sim (\hP^{(n,t)}, \pi_t) \atop a_{h-1}\sim\pi_t}\hspace{-0.03in} \left[\min\left\{\left\|\hphi_{h-1}^{(n)}(s_{h-1},a_{h-1})\right\|_{(U^{(n,t)}_{h-1,\hphi^{(n,t)}})^{-1}},(\alpha_h^{(n,t)})^{-1}\hspace{-0.02in}\right\}\hspace{-0.02in}\right]\hspace{-0.02in}+\sum_{t=1}^T\hspace{-0.02in}\sqrt{K\zeta_1^{(n,t)}}\\
 		&\scriptstyle \overset{(\romannumeral1)}{\leq} \sum_{h=2}^H\sqrt{\sum_{t=1}^T\left[\mathop{\Eb}_{s_{h-1} \sim (\hP^{(n,t)}, \pi_t) \atop a_{h-1}\sim\pi_t} \min\left\{\left\|\hphi_{h-1}^{(n)}(s_{h-1},a_{h-1})\right\|_{(U^{(n,t)}_{h-1,\hphi^{(n,t)}})^{-1}},(\alpha_h^{(n,t)})^{-1}\right\}\right]^2\left(\sum_{t=1}^T(\alpha_h^{(n,t)})^2\right)}\\
 		& \qquad +\sqrt{KT\zeta_n}\\
 		& \scriptstyle \overset{\RM{2}}{\leq} \sum_{h=2}^H\sqrt{\sum_{t=1}^T\left[\mathop{\Eb}_{s_{h-1} \sim (\hP^{(n,t)}, \pi_t) \atop a_{h-1}\sim\pi_t} \min\left\{\alpha_n\left\|\hphi_{h-1}^{(n)}(s_{h-1},a_{h-1})\right\|_{(U^{(n,t)}_{h-1,\hphi^{(n,t)}})^{-1}},\frac{\alpha_n }{\alpha_h^{(n,t)}}\right\}\right]^2}+\sqrt{KT\zeta_n}\\
 		& \overset{(\romannumeral3)}{\leq} \sum_{h=2}^H\sqrt{\sum_{t=1}^T\left[\mathop{\Eb}_{s_{h-1} \sim (\hP^{(n,t)}, \pi_t) \atop a_{h-1}\sim\pi_t} \min\left\{\alpha_n\left\|\hphi_{h-1}^{(n)}(s_{h-1},a_{h-1})\right\|_{(U^{(n,t)}_{h-1,\hphi^{(n,t)}})^{-1}},B\right\}\right]^2}+\sqrt{KT\zeta_n}\\
 		& \overset{(\romannumeral4)}{\leq} \sum_{h=2}^H\sqrt{\sum_{t=1}^T\left\{\mathop{\Eb}_{s_{h-1} \sim (\hP^{(n,t)}, \pi_t) \atop a_{h-1}\sim\pi_t} \left[ {  \hb_{h-1}^{(n,t)}  }
 		\right]\right\}^2}+\sqrt{KT\zeta_n}\\
 		&= PCV\left(\hP^{(n,t)},\hb_{h}^{(n,t)},\pi_t;T\right)+\sqrt{KT\zeta_n}.
\end{align*}
where in $(\romannumeral1)$, we apply Cauchy-Schwarz inequality for both terms and use the fact that $\sum_{t=1}^{T}\zeta_h^{(n,t)}\leq \zeta_n$ (see {\Cref{lemma: multitask MLE Guarantee}}), $\RM{2}$ follows from \Cref{eqn:prop2:+1}, $(\romannumeral3)$ follows from \Cref{eqn:prop2:+2}, and $(\romannumeral4)$ follows from \Cref{coro: concentration of the bonus term} in \Cref{app:supportinglemma}.

Similarly, following from \Cref{prop1: Bounded TV}, we take the summation over all source tasks and obtain:
\begin{align*}
\sum_{t=1}^{T}\sum_{h=1}^H \mathop{\Eb}_{s_h\sim(\hP^{(n,t)},\pi_t) \atop a_h \sim \pi_t} \left[ {f_h^{(n,t)}(s_h,a_h)}\right] \leq PCV\left(\hP^{(n,t)},\hb_{h}^{(n,t)},\pi_t;T\right)+\sqrt{KT\zeta_n}.
\end{align*}
\end{proof}

\subsection{Step 3: Sublinear accumulation of PCV}\label{app:step3}

Recall that the exploration policy is derived by an oracle:
\begin{align}
 		\pi_1^n, \dots, \pi_T^n = \mathop{\arg\max}_{\pi_1, \dots, \pi_T}PCV\left(\hP^{(n,t)},\hb_{h}^{(n,t)},\pi_t;T\right)  \label{eqn:step3:++1}
\end{align}
In this step, we show that the summation of exploration-driven reward function PCV over $n,t,h$ is sublinear with respect to the total number $N_u$ of iterations, as given in \Cref{prop3: Bound of summation of exploration-driven reward function}, which further implies polynomial efficiency of REFUEL in learning the models. 
\begin{proposition} 
\label{prop3: Bound of summation of exploration-driven reward function}
Given any $\delta \in
\left(0,|\Psi|^{-\min\{T,\frac{K}{d^2}\}}\right)$, set $\lambda_n=O(d\log(|\Phi|nTH/\delta))$. Then with probability $1-\delta$, under exploration policy $\{\pi_t^n\}_{n\in [N_u]}$ for each task $t$, the summation of PCVs over $n$ is sublinear with respect to the number $N_u$ of iteration rounds:
\begin{align}
 		&\sum_{n=1}^{N_u}  \left\{PCV\left(\hP^{(n,t)},\hb_{h}^{(n,t)},\pi_t^n;T\right)+\sqrt{KT\zeta_n} \right\}\nonumber\\
 		& \quad \scriptstyle \leq \frac{4\beta_2}{\beta_1}H\sqrt{{N_u}Td}\left( \sqrt{K^2d\log^2\left(|\Phi||\Psi|^T N_uH/\delta\right)}+\sqrt{Kd^3T\log^2(2N_uTH|\Phi|/\delta)}+\sqrt{d^2B^2\log^2(N_uTH|\Phi|/\delta)} \right)\nonumber\\
 		& \quad + 6BH^2\sqrt{{N_u}TKd}\left( \sqrt{K{\log^2\left(|\Phi||\Psi|^T N_uH/\delta\right)}}+\sqrt{d^2\log^2(N_uTH|\Phi|/\delta)}\right) \label{ineq: simplified order}.
\end{align}
\end{proposition}
\begin{proof}
We proceed the bound as follows:
\begin{align}
 		&\sum_{n=1}^{N_u}  \left\{PCV\left(\hP^{(n,t)},\hb_{h}^{(n,t)},\pi_t^n;T\right)+\sqrt{KT\zeta_n} \right\}\nonumber\\
 		& \quad \leq N_u\sqrt{KT\zeta_n}+\sum_{n=1}^{N_u}\sum_{h=1}^{H-1} \sqrt{\sum_{t=1}^T\left\{\mathop{\Eb}_{s_{h} \sim (\hP^{(n,t)}, \pi_t^n) \atop a_{h}\sim\pi_t^n} \left[\hb_{h}^{(n,t)}(s_{h},a_{h})\right]\right\}^2}\nonumber\\
 		& \quad \overset{\RM{1}}{\leq} N_u\sqrt{KT\zeta_n}+\underbrace{\sum_{n=1}^{N_u}\sum_{h=1}^{H-1} \sqrt{\sum_{t=1}^T\left\{\mathop{\Eb}_{s_{h} \sim (P^{(\star,t)}, \pi_t^n) \atop a_{h}\sim\pi_t^n} \left[\hb_{h}^{(n,t)}(s_{h},a_{h})\right]\right\}^2}}_{(a)}\nonumber\\
 		&\quad+\underbrace{\sum_{n=1}^{N_u}\sum_{h=1}^{H-1} \sqrt{\sum_{t=1}^T\left\{\mathop{\Eb}_{s_{h} \sim (P^{(\star,t)}, \pi_t^n) \atop a_{h}\sim\pi_t^n} \left[\hb_{h}^{(n,t)}(s_{h},a_{h})\right]-\mathop{\Eb}_{s_{h} \sim (\hP^{(n,t)}, \pi_t^n) \atop a_{h}\sim\pi_t^n} \left[\hb_{h}^{(n,t)}(s_{h},a_{h})\right]\right\}^2}}_{(b)}, \label{ineq: division into a and b}
\end{align}
where $\RM{1}$ follows from the fact that for any vector $x,y\in \Rb^T$, $\norm{x+y}_2\leq \norm{x}_2+\norm{y}_2$.

In the sequel, we first upper-bound the terms $(a)$ and $(b)$, and then combine the upper bounds with \Cref{ineq: division into a and b} as our final step to obtain the desired result.

{\bf I) Bound term $(a)$.}

Denote the trace operator as $\mathrm{tr}(\cdot)$. We first obtain
\begin{align}
& nK\mathop{\Eb}_{s_{h}\sim (P^\star, {\Pi}_t^n)\atop a_h\sim \Uc(\Ac)}\left[\hb_{h}^{(n,t)}(s_{h},a_{h})\right] \nonumber \\ 
& \quad \overset{(\romannumeral1)}{\leq} \frac{nK\beta_2^2}{\beta_1^2}\mathop{\Eb}_{s_{h}\sim (P^\star, {\Pi}_t^n)\atop a_h\sim \Uc(\Ac)}\left[\alpha_n^2\left\|\hphi_{h}^{(n)}(s_{h},a_{h})\right\|_{({U}^{(n,t)}_{h,\hphi^{(n,t)}})^{-1}}^2\right]\nonumber\\
& \quad \overset{\RM{2}}{\leq} \frac{K\beta_2^2\alpha_n^2}{\beta_1^2}  {\rm tr}(I_d) =K\beta_2^2(2nK\zeta_n+T\lambda_n d)d=\frac{Kd\beta_2^2\alpha_n^2}{\beta_1^2},  \label{ineq: prop1_trace}
\end{align}
where $(\romannumeral1)$ follows from \Cref{coro: concentration of the bonus term} and $\RM{2}$ follows from the following derivation:
\begin{align*}
 &n\mathop{\Eb}_{s_{h}\sim (P^\star, {\Pi}_t^n)\atop a_h\sim \Uc(\Ac)}\left[\left\|\hphi_{h}^{(n)}(s_{h},a_{h})\right\|_{({U}^{(n,t)}_{h,\hphi^{(n,t)}})^{-1}}^2\right]\\
 &\quad= n\mathop{\Eb}_{s_{h}\sim (P^\star, {\Pi}_t^n)\atop a_h\sim \Uc(\Ac)}\left[\mathrm{tr}\left(\hphi_h^{(n,t)}(s_h,a_h)\hphi_h^{(n,t)}(s_h,a_h)^\top ({U}^{(n,t)}_{h,\hphi^{(n,t)}})^{-1}\right)\right]\\
 &\scriptstyle\quad=\mathrm{tr}\left(\mathop{\Eb}_{s_{h}\sim (P^\star, {\Pi}_t^n)\atop a_h\sim \Uc(\Ac)}\left[n\hphi_h^{(n,t)}(s_h,a_h)\hphi_h^{(n,t)}(s_h,a_h)^\top\right] (n\Eb_{s_h\sim(P^\star,\Pi_t^n),a_h\sim \Uc(\Ac)}\left[\phi(s_h,a_h)(\phi(s_h,a_h))^\top\right] + \lambda_n I_d)^{-1}\right)\\
 &\quad \leq \mathrm{tr}(I_d).
\end{align*}

Following from \Cref{lemma:Step_Back}, for $h \geq 2$, we have
\begin{align} 
 		&\mathop{\Eb}_{s_{h} \sim (P^{(\star,t)}, \pi_t^n) \atop a_{h}\sim\pi_t^n} \left[\hb_{h}^{(n,t)}(s_{h},a_{h})\bigg|s_{h-1},a_{h-1}\right] \nonumber \\
 		&\quad\leq \left\|\phi_{h-1}^\star(s_{h-1},a_{h-1})\right\|_{(W_{h-1,\sphi}^{(n,t)})^{-1}}
 		\sqrt{nK\mathop{\Eb}_{s_{h}\sim (P^\star, {\Pi}_t^n)\atop a_h\sim \Uc(\Ac)}\left[\hb_{h}^{(n,t)}(s_{h},a_{h})\right]+\lambda_nd B^2 } \nonumber  \\
 		&\quad \leq \left\|\phi_{h-1}^\star(s_{h-1},a_{h-1})\right\|_{(W_{h-1,\sphi}^{(n,t)})^{-1}}
 		\sqrt{\frac{Kd\beta_2^2\alpha_n^2}{\beta_1^2}+\lambda_ndB^2 } \label{ineq: hatb 1}.
\end{align}
where the last inequality follows from \Cref{ineq: prop1_trace}.

Furthermore, for $h=1$, we have
\begin{align}
 		\mathop{\Eb}_{s_{1} \sim (P^{(\star,t)}, \pi_t^n) \atop a_{1}\sim\pi_t^n} \left[\hb_{1}^{(n,t)}(s_{1},a_{1})\right] 
 		 &\overset{\RM{1}}{=}\mathop{\Eb}_{ a_{1}\sim\pi_t^n} \left[\hb_{1}^{(n,t)}(s_{1},a_{1})\right]
 		\overset{\RM{2}}{\leq} \sqrt{\mathop{\Eb}_{ a_{1}\sim\pi_t^n} \left[\hb_{1}^{(n,t)}(s_{1},a_{1})^2\right]}\nonumber\\
 		& \overset{\RM{3}}{\leq} \sqrt{K\mathop{\Eb}_{a_1\sim \Uc(\Ac)}\left[\left\|\hphi_{1}^{(n)}(s_{1},a_{1})\right\|_{(\widehat{U}^{(n,t)}_{1,\hphi^{(n,t)}})^{-1}}^2\alpha_n^2\right]} \nonumber \\
 		&\overset{\RM{4}}{\leq} \sqrt{\frac{K\beta_2^2\alpha_n^2d}{n\beta_1^2}} \label{ineq: hatb 2},
\end{align}
where $\RM{1}$ follows from the fact that the initial state $s_1$ is fixed, $\RM{2}$ follows from Cauchy Schwarz inequality and Jensen's inequality, $\RM{3}$ follows from importance sampling, and $\RM{4}$ follows from a step similar to  \Cref{ineq: prop1_trace}.
 		
Substituting \Cref{ineq: hatb 1} and \Cref{ineq: hatb 2} into the term $(a)$, we obtain
\begin{align}
 		&\scriptstyle \sum_{n=1}^{N_u}\sum_{h=1}^{H-1} \sqrt{\sum_{t=1}^T\left\{\mathop{\Eb}_{s_{h} \sim (P^{(\star,t)}, \pi_t^n) \atop a_{h}\sim\pi_t^n} \left[\hb_{h}^{(n,t)}(s_{h},a_{h})\right]\right\}^2}\nonumber\\
 		& \scriptstyle \quad \leq \sum_{n=1}^{N_u}\left\{\sum_{h=2}^{H-1}
 		 \sqrt{\sum_{t=1}^T\left\{\mathop{\Eb}_{s_{h} \sim (P^{(\star,t)}, \pi_t^n) \atop a_{h}\sim\pi_t^n} \left[\left\|\phi_{h-1}^\star(s_{h-1},a_{h-1})\right\|_{(W_{h-1,\sphi}^{(n,t)})^{-1}}
 		\sqrt{\frac{Kd\beta_2^2\alpha_n^2}{\beta_1^2}+\lambda_ndB^2}\right]\right\}^2}\right.\nonumber\\
 		&\qquad \quad\left.+\sqrt{\frac{TK\beta_2^2\alpha_n^2d}{n\beta_1^2}}\right\}\nonumber\\
 		& \quad \scriptstyle \overset{(\romannumeral1)}{\leq} \frac{\beta_2}{\beta_1}\left\{\sqrt{K\alpha_{N_u}^2  d+\lambda_{N_u} d B^2}\sum_{n=1}^{N_u}\sum_{h=2}^{H-1} \sqrt{\sum_{t=1}^T\mathop{\Eb}_{s_{h} \sim (P^{(\star,t)}, \pi_t^n) \atop a_{h}\sim\pi_t^n} \left[\left\|\phi_{h-1}^\star(s_{h-1},a_{h-1})\right\|_{(W_{h-1,\sphi}^{(n,t)})^{-1}}^2\right]} \right.\nonumber\\ 
 		&\qquad  \quad \left.+2\sqrt{{N_u}TK\alpha_{N_u}^2d}\right\}\nonumber\\
 		& \quad \scriptstyle\overset{\RM{2}}{\leq} \frac{\beta_2}{\beta_1}\left\{\sqrt{K\alpha_{N_u}^2   d+\lambda_{N_u} dB^2 }\sum_{h=2}^{H-1} \sqrt{{N_u}\sum_{t=1}^T\sum_{n=1}^{N_u}\mathop{\Eb}_{s_{h} \sim (P^{(\star,t)}, \pi_t^n) \atop a_{h}\sim\pi_t^n} \left[\left\|\phi_{h-1}^\star(s_{h-1},a_{h-1})\right\|_{(W_{h-1,\sphi}^{(n,t)})^{-1}}^2\right]}\right.\nonumber\\ 
 		& \qquad \quad \left.+2\sqrt{{N_u}TK\alpha_{N_u}^2d}\right\}\nonumber\\
 		& \quad \overset{\RM{3}}{\leq} \frac{\beta_2}{\beta_1}\left\{ \sqrt{K\alpha_{N_u}^2   d+\lambda_{N_u} dB^2 }\sum_{h=2}^{H-1} \sqrt{{N_u}\sum_{t=1}^Td\log\left(1+\frac{{N_u}}{d\lambda_1}\right)}+2\sqrt{{N_u}TK\alpha_{N_u}^2d}\right\}\nonumber\\
 		& \quad \leq \frac{2\beta_2}{\beta_1} \sqrt{K\alpha_{N_u}^2   d+\lambda_{N_u} dB^2 }H \sqrt{{N_u}Td\log\left(1+\frac{{N_u}}{d\lambda_1}\right)}, \label{ineq: term a final bound}
\end{align}
where $(\romannumeral1)$ follows from the fact that $\frac{\beta_2}{\beta_1} \geq 1$ (see \Cref{coro: concentration of the bonus term}), $\alpha_{N_u} \geq \alpha_{N_u-1}\geq \ldots \geq \alpha_1$ {and $\sum_{n=1}^{N_u}1/\sqrt{n}\leq 1+\int_{1}^{N_u}1/\sqrt{x}\dif x \leq 2\sqrt{N_u}$}, {$\RM{2}$ follows from Cauchy-Schwarz inequality}, and $\RM{3}$ follows from \Cref{lemma: Elliptical_potential}.
 		
{\bf II) Bound term $(b)$.}

We proceed the derivation as follows:
\begin{align}
 		& \sqrt{\sum_{t=1}^T\left\{\mathop{\Eb}_{s_{h} \sim (P^{(\star,t)}, \pi_t^n) \atop a_{h}\sim\pi_t^n} \left[\hb_{h}^{(n,t)}(s_{h},a_{h})\right]-\mathop{\Eb}_{s_{h} \sim (\hP^{(n,t)}, \pi_t^n) \atop a_{h}\sim\pi_t^n} \left[\hb_{h}^{(n,t)}(s_{h},a_{h})\right]\right\}^2}\nonumber\\
 		& \overset{\RM{1}}{\leq} \sqrt{\sum_{t=1}^T\left\{ \sum_{h^\prime=1}^{h}B\mathop{\Eb}_{s_{\ph} \sim (P^{(\star,t)}, \pi_t^n) \atop a_{\ph}\sim\pi_t^n}\left[ {f_\ph^{(n,t)}(s_\ph,a_\ph)}\right]\right\}^2}  \nonumber \\
 		& \overset{\RM{2}}{\leq} \sqrt{B^2\sum_{t=1}^T h \sum_{h^\prime=1}^{h}\left\{\mathop{\Eb}_{s_{\ph} \sim (P^{(\star,t)}, \pi_t^n) \atop a_{\ph}\sim\pi_t^n}\left[ {f_\ph^{(n,t)}(s_\ph,a_\ph)}\right]\right\}^2}\nonumber\\
 		&  \leq B\sqrt{\sum_{t=1}^T H\left\{ \sum_{h^\prime=2}^{H}\left\{\mathop{\Eb}_{s_{\ph} \sim (P^{(\star,t)}, \pi_t^n) \atop a_{\ph}\sim\pi_t^n}\left[ {f_\ph^{(n,t)}(s_\ph,a_\ph)}\right]\right\}^2+\left[\mathop{\Eb}_{ a_{1}\sim\pi_t^n} {f_1^{(n,t)}(s_1,a_1)}\right]^2\right\}}\nonumber\\
 		&  \scriptstyle \overset{(\romannumeral3)}{\leq} B\hspace{-0.03in}\sqrt{\sum_{t=1}^T\hspace{-0.03in}H \sum_{h^\prime=2}^{H}\left\{\mathop{\Eb}_{s_{\ph-1} \sim (P^{(\star,t)}, \pi_t^n) \atop a_{\ph-1}\sim\pi_t^n}\hspace{-0.03in} \left[\left\|\sphi_{\ph-1}(s_{\ph-1},a_{\ph-1})\right\|_{(U^{(n,t)}_{\ph-1,\sphi})^{-1}}\hspace{-0.02in}\sqrt{nK\zeta_h^{(n,t)}\hspace{-0.03in}+\lambda_n d}\right]\right\}^2+KHT\zeta_n}\nonumber\\
 		& \leq B\sqrt{H\sum_{h=1}^{H-1}\sum_{t=1}^T\left\{\mathop{\Eb}_{s_{h} \sim (P^{(\star,t)}, \pi_t^n) \atop a_{h}\sim\pi_t^n} \left[\left\|\sphi_{h}(s_{h},a_{h})\right\|_{(U^{(n,t)}_{h,\sphi})^{-1}}\right]\right\}^2(nK\zeta_h^{(n,t)}+\lambda_n d)+KHT\zeta_n}\label{ineq: part b},
\end{align}
where $\RM{1}$ follows from \Cref{lemma: Simulation} with a sparse reward $r_{h'}(\cdot,\cdot)=\widehat{b}_\ph^{(n,t)}(\cdot,\cdot)\mathbf{1}\{h'=h\}$, where $\mathbf{1}\{\cdot\}$ is the indicator function, $\RM{2}$ follows because $(\sum_{\ph=1}^h x_\ph)^2 \leq h(\sum_{\ph=1}^h x_\ph^2)$, and $(\romannumeral3)$ follows from \Cref{lemma:Step_Back}, {importance sampling and because $\zeta_1^{(n,t)}\leq \zeta_n$ (see \Cref{lemma: multitask MLE Guarantee})}.
 		
We further substitute \Cref{ineq: part b} into the term $(b)$ and obtain
\begin{align}
 		& \scriptstyle \sum_{n=1}^{N_u}\sum_{h=1}^{H-1}B \sqrt{H\sum_{h=1}^{H-1}\sum_{t=1}^T\left\{\mathop{\Eb}_{s_{h} \sim (P^{(\star,t)}, \pi_t^n) \atop a_{h}\sim\pi_t^n} \left[\left\|\sphi_{h}(s_{h},a_{h})\right\|_{(U^{(n,t)}_{h,\sphi})^{-1}}\right]\right\}^2(nK\zeta_h^{(n,t)}+\lambda_n d)+KHT\zeta_n}\nonumber\\
 		& \quad \scriptstyle \overset{(\romannumeral1)}{\leq}\sum_{h=1}^{H-1}B \sqrt{{N_u}H({N_u}K\zeta_{N_u}+\lambda_{N_u} d)\sum_{n=1}^{N_u}\sum_{h=1}^{H-1}\sum_{t=1}^T\left\{\mathop{\Eb}_{s_{h} \sim (P^{(\star,t)}, \pi_t^n) \atop a_{h}\sim\pi_t^n} \left[\left\|\sphi_{h}(s_{h},a_{h})\right\|_{(U^{(n,t)}_{h,\sphi})^{-1}}\right]\right\}^2}\nonumber\\
 		 &\qquad  \quad +\sum_{h=1}^{H-1}B\sqrt{KHTN_u\sum_{n=1}^{N_u}\zeta_n}\nonumber\\
 		& \quad \scriptstyle \overset{(\romannumeral2)}{\leq} \sum_{h=1}^{H-1} B\sqrt{{N_u}H({N_u}K\zeta_{N_u}+\lambda_{N_u} d)K\sum_{h=1}^{H-1}\sum_{t=1}^T\sum_{n=1}^{N_u}\mathop{\Eb}_{s_{h} \sim (P^{(\star,t)}, \pi_t^n) \atop a_{h}\sim \Uc(\Ac)} \left[\left\|\sphi_{h}(s_{h},a_{h})\right\|_{(U^{(n,t)}_{h,\sphi})^{-1}}^2\right]}\nonumber\\
 		&\qquad  \quad +\sqrt{{N_u}KHT}HBN_u\zeta_{N_u}\nonumber\\
 		& \quad \overset{\RM{3}}{\leq} HB \sqrt{{N_u}H^2T({N_u}K\zeta_{N_u}+\lambda_{N_u} d)Kd\log\left(1+\frac{{N_u}}{d\lambda_1}\right)}+\sqrt{{N_u}KHT}HBN_u\zeta_{N_u},\nonumber\\
        \label{ineq: term b final bound}
\end{align}
where $\RM{1}$ follows from Cauchy-Schwarz inequality, and because $\sqrt{x+y}\leq\sqrt{x}+\sqrt{y}$ for $x,y\geq 0$, {and $nK\zeta_h^{(n,t)}+\lambda_n d \leq nK\zeta_n+\lambda_n d$, where the latter bound is increasing in $n$}, $\RM{2}$ follows because $\sum_{n=1}^{N_u}\zeta_n\leq \log^2\left(2|\Phi||\Psi|^T N_uH/\delta\right)\leq N_u^2\zeta_{N_u}^2$, 
and $\RM{3}$ follows from \Cref{lemma: Elliptical_potential} and importance sampling.

{\bf III) Final step.}
 		
We substitute \Cref{ineq: term a final bound} and \Cref{ineq: term b final bound} into \Cref{ineq: division into a and b}, and have
\begin{align}
 		\sum_{n=1}^{N_u}&  \left\{PCV\left(\hP^{(n,t)},\hb_{h}^{(n,t)},\pi_t^n;T\right)+\sqrt{KT\zeta_n} \right\}\nonumber\\
 		& \leq \sqrt{{N_u}KHT}HBN_u\zeta_{N_u}+\frac{2\beta_2}{\beta_1} \sqrt{K\alpha_{N_u}^2   d+\lambda_{N_u} dB^2 }H \sqrt{{N_u}Td\log\left(1+\frac{{N_u}}{d\lambda_1}\right)}\nonumber
 		\\
 		& \qquad + B \sqrt{{N_u}H^4T({N_u}K\zeta_{N_u}+\lambda_{N_u} d)Kd\log\left(1+\frac{{N_u}}{d\lambda_1}\right)}
 		\label{ineq: summation bound 1}.
 		\end{align}
 		Then, we substitute the definitions of $\zeta_{N_u}$, $\alpha_{N_u}$ and $\lambda_{N_u}$ into \Cref{ineq: summation bound 1} and simplify the expression by taking only dominating terms as follows:
 		\begin{align}
 		&\sum_{n=1}^{N_u} \left\{PCV\left(\hP^{(n,t)},\hb_{h}^{(n,t)},\pi_t^n;T\right)+\sqrt{KT\zeta_n} \right\}\nonumber\\
 		& \leq  2\sqrt{N_uKHT}HB{\log\left(2|\Phi||\Psi|^T N_uH/\delta\right)}\nonumber\\
 		&\quad \scriptstyle+\frac{2\beta_2}{\beta_1} \sqrt{Kd(4K\log\left(2|\Phi||\Psi|^T N_uH/\delta\right)+d^2T\log(2N_uTH|\Phi|/\delta))   +d^2\log(2N_uTH|\Phi|/\delta)B^2 }H \sqrt{{N_u}Td\log\left(1+\frac{{N_u}}{d\lambda_1}\right)}
 		\nonumber\\
 		&\quad + B \sqrt{{N_u}H^4T(2K{\log\left(2|\Phi||\Psi|^T N_uH/\delta\right)}+d^2\log(2N_uTH|\Phi|/\delta))Kd\log\left(1+\frac{{N_u}}{d\lambda_1}\right)}\nonumber\\
 		&\leq \scriptstyle \frac{4\beta_2}{\beta_1}H\sqrt{{N_u}Td}\left( \sqrt{K^2d\log^2\left(|\Phi||\Psi|^T N_uH/\delta\right)}+\sqrt{Kd^3T\log^2(2N_uTH|\Phi|/\delta)}+\sqrt{d^2B^2\log^2(N_uTH|\Phi|/\delta)} \right)\nonumber\\
 		&\quad+ 6BH^2\sqrt{{N_u}TKd}\left( \sqrt{K{\log^2\left(|\Phi||\Psi|^T N_uH/\delta\right)}}+\sqrt{d^2\log^2(N_uTH|\Phi|/\delta)}\right)\nonumber.
\end{align}
\end{proof}

\subsection{Complexity characterization: Proof of \Cref{theorem: Upstream sample complexity}}\label{app:complexity}

Next, {equipped with \Cref{prop1: Bounded TV,prop2: total value difference,prop3: Bound of summation of exploration-driven reward function} and \Cref{coro1:bounded difference of summation}}
, we are able to derive the sample complexity bound of \Cref{alg: Upstream}. 

 		We prove \Cref{theorem: Upstream sample complexity} by contradiction. First for any $n$ and policy $\pi_t$, we have 
 		\begin{align}
 		\sum_{t=1}^{T}&\mathop{\Eb}_{s_h\sim (P^{(\star,t)},\pi_t)\atop a_h\sim \pi_t}  \left[f_h^{(n,t)}(s_h,a_h)\right]\nonumber\\
 		&=\sum_{t=1}^{T} \left(\mathop{\Eb}_{s_h\sim (\hP^{(n,t)},\pi_t)\atop a_h\sim \pi_t}  \left[f_h^{(n,t)}(s_h,a_h)\right]-\mathop{\Eb}_{s_h\sim (P^{(\star,t)},\pi_t)\atop a_h\sim \pi_t}  \left[f_h^{(n,t)}(s_h,a_h)\right]\right)\nonumber\\
 		&\qquad +\sum_{t=1}^{T}\mathop{\Eb}_{s_h\sim (\hP^{(n,t)},\pi_t)\atop a_h\sim \pi_t }  \left[f_h^{(n,t)}(s_h,a_h)\right]\nonumber\\
 		& \overset{\RM{1}}{\leq} 2\left\{PCV\left(\hP^{(n,t)},\hb_{h}^{(n,t)},\pi_t;T\right)+\sqrt{KT\zeta_n}\right\}\nonumber\\
 		& \overset{\RM{2}}{\leq} 2\left\{PCV\left(\hP^{(n,t)},\hb_{h}^{(n,t)},\pi_t^n;T\right)+\sqrt{KT\zeta_n}\right\}\label{ineq: theorem 1.1},
 		\end{align}
 		where $\RM{1}$ follows from \Cref{prop2: total value difference} and $\RM{2}$ follows from the definition of $\{\pi_t^n\}_{t \in[T]}$ (see \Cref{eqn:step3:++1}).
 		
If for any $n \in [N_u]$, $T\epsilon_u < 2\left\{PCV\left(\hP^{(n,t)},\hb_{h}^{(n,t)},\pi_t^n;T\right)+\sqrt{KT\zeta_n}\right\}$, which is exactly the termination criteria in \Cref{alg: Upstream}, then
 		\begin{align}
 		N_u&T\epsilon_u  \nonumber \\
 		& < \sum_{n=1}^{N_u}2\left\{PCV\left(\hP^{(n,t)},\hb_{h}^{(n,t)},\pi_t^n;T\right)+\sqrt{KT\zeta_n}\right\}\nonumber\\
 		&\leq \scriptstyle \frac{4\beta_2}{\beta_1}H\sqrt{{N_u}Td}\left( \sqrt{K^2d\log^2\left(|\Phi||\Psi|^T N_uH/\delta\right)}+\sqrt{Kd^3T\log^2(2N_uTH|\Phi|/\delta)}+\sqrt{d^2B^2\log^2(N_uTH|\Phi|/\delta)} \right)\nonumber\\
 		&\quad + 6BH^2\sqrt{{N_u}TKd}\left( \sqrt{K{\log^2\left(|\Phi||\Psi|^T N_uH/\delta\right)}}+\sqrt{d^2\log^2(N_uTH|\Phi|/\delta)}\right), \label{ineq: contradiction}
 		\end{align}
 		where the last inequality follows from \Cref{prop3: Bound of summation of exploration-driven reward function} and \Cref{ineq: theorem 1.1}.
 		
 		Note that we assume $\delta$ is small enough satisfying $\delta \leq  {|\Psi|^{-\frac{\min\{T,K\}}{d^2}}}$. If 
 		\begin{align*}
 			N_u>\frac{400\beta_2^2H^2Td^2K^2\log^2\left(400\beta_2^2H^2d^2K^2|\Phi| H/(\beta_1^2\delta^2\epsilon_u^2)\right)}{T\beta_1^2\epsilon_u^2},
 		\end{align*} 
 		then use the fact that for $c \geq e^2, n \geq 1, \alpha  \in \Rb^+$, if
 		$n \geq 2c \log (\alpha c)$, then $n \geq c\log (\alpha n)$,
 		we have 
 		\begin{align*}
 			\frac{4\beta_2}{\beta_1} {H\sqrt{{N_u}Td} \sqrt{K^2d\log^2\left(|\Phi||\Psi|^T N_uH/\delta\right)}} \leq \frac{\epsilon_uT N_u}{5},
 		\end{align*}
 		which is exactly the first term in \Cref{ineq: contradiction}. Similarly, we are able to upper bound each of the other four terms by $\frac{\epsilon_u N_u}{5}$ in \Cref{ineq: contradiction} with the iteration number $N_u$ being at most:
 		\begin{align*}
 		 \widetilde{O}\left(\frac{H^2d^2K^2}{T\epsilon_u^2}+\frac{\left(H^2d^4K+H^4dK^2+H^4d^3K\right)}{\epsilon_u^2}+\frac{H^4K^3}{dT\epsilon_u^2}\right).
 		\end{align*}
 		Combining the above bound with \Cref{ineq: contradiction}, we have
 		\begin{align*}
 		    N_uT\epsilon_u < 
 		    5 \times \frac{N_uT\epsilon_u}{5}=N_uT\epsilon_u,
 		\end{align*}
 		which leads to a contradiction and shows that \Cref{alg: Upstream} is able to terminate at a certain iteration $n_u$ and output desired models with the  number $HN_u$ of trajectories being at most:
 		\begin{align*}
 		 \widetilde{O}\left(\frac{H^3d^2K^2}{T\epsilon_u^2}+\frac{\left(H^3d^4K+H^5dK^2+H^5d^3K\right)}{\epsilon_u^2}+\frac{H^5K^3}{dT\epsilon_u^2}\right).
 		\end{align*}
 		
 		Furthermore, let $\pi_t^\star$ be the optimal policy under $\Mc^t$ given reward the $r^t$. And form \Cref{alg: Upstream}, the algorithm terminates at iteration $n_u$ and outputs $\hP^{(t)}$ for $t \in [T]$. Then we have
 	\begin{align*}
 		 &\sum_{t=1}^T V_{P^{(\star,t)},r^t}^\star-V_{P^{(\star,t)},r^t}^{\widehat{\pi}_t}\\
 		 & \quad = \sum_{t=1}^T V_{P^{(\star,t)},r^t}^\star-V_{\hP^{(t)},r^t}^{\pi_t^\star}+V_{\hP^{(t)},r^t}^{\pi_t^\star}-V_{\hP^{(t)},r^t}^{\widehat{\pi}_t}+V_{\hP^{(t)},r^t}^{\widehat{\pi}_t}-V_{P^{(\star,t)},r^t}^{\widehat{\pi}_t}\\
 		  &\quad \overset{\RM{1}}{\leq} \sum_{h=1}^{H-1}\sqrt{\sum_{t=1}^T\mathop{\Eb}_{s_{h} \sim (\hP^{(t)}, \pi_t^\star) \atop a_{h}\sim\pi_t^\star} \left[\hb_{h}^{(n_u,t)}(s_h,a_h)\right]^2}+\sqrt{KT\zeta_{{n_u}}}\\
 		  & \qquad \quad +\sum_{h=1}^{H-1}\sqrt{\sum_{t=1}^T\mathop{\Eb}_{s_{h} \sim (\hP^{(t)}, \widehat{\pi}_t) \atop a_{h}\sim\widehat{\pi}_t} \left[\hb_{h}^{(n_u,t)}(s_h,a_h)\right]^2}+\sqrt{KT\zeta_{{n_u}}}\\
 		  &\quad \leq 2\left\{PCV\left(\hP^{(t)},\hb_{h}^{(n_u,t)},\pi_t^{n_u};T\right)
 		  +\sqrt{KT\zeta_{{n_u}}}\right\}\\
 		  & \quad \overset{\RM{2}}{\leq} T\epsilon_u,
 	\end{align*}
 	where $\RM{1}$ follows from the definition of $\{\pi_t^\star\}_{t \in[T]}$ and \Cref{prop2: total value difference}, and $\RM{2}$ follows from the termination criteria of \Cref{alg: Upstream}.
 	
 	\subsection{Supporting Lemmas}\label{app:supportinglemma}
 	Recall $U_{h,\phi}^{(n,t)} = n\Eb_{s_h\sim(P^\star,\Pi_t^n),a_h\sim \Uc(\Ac)}\left[\phi(s_h,a_h)(\phi(s_h,a_h))^\top\right] + \lambda_n I$. Then $U_{h,\hphi^{(n)}}^{(n,t)}$ is the counterpart of $\widehat{U}_h^{(n,t)}$ in expectation. The following lemma provides the concentration of the bonus term. See Lemma 39 in \citet{zanette2020learning} for the version of fixed $\phi$ and Lemma 11 in \citet{uehara2021representation}.
 	\begin{lemma}\label{lemma: concentration of the bonus term}
 		({\rm Concentration of the bonus term}). Fix $\delta \in (0,1)$, and set $\lambda_n = \Theta(d\log(2nTH|\Phi|/\delta))$ for any $n$.
 		With probability at least $1-\delta/2$, we have that 
 		$\forall n \in \mathbb{N}^+, h \in [H], t \in [T], \hphi \in \Phi, $                                            
 		\begin{equation*}
 		\beta_1 \left\|\hphi_{h}^{(n)}(s,a)\right\|_{(U_{h,\hphi}^{(n,t)})^{-1}} \leq \left\|\hphi_{h}^{(n)}(s,a)\right\|_{(\widehat{U}_{h}^{(n,t)})^{-1}} \leq \beta_2 \left\|\hphi_{h}^{(n)}(s,a)\right\|_{(U_{h,\hphi}^{(n,t)})^{-1}}.
 		\end{equation*}
 	\end{lemma}
 	Since $\hb_h^{(n,t)}(s_h,a_h)=\min\left\{{\tap}_n\left\|\hphi_{h}^{(n)}(s,a)\right\|_{(\widehat{U}_{h}^{(n,t)})^{-1}},B\right\}$. {Setting} $\tap_n=\frac{\alpha_n}{{\beta_1}}$ and applying \Cref{lemma: concentration of the bonus term}, we can immediately obtain the following corollary.
 	\begin{corollary} \label{coro: concentration of the bonus term}
 		Fix $\delta  \in (0,1)$, under the same setting of \Cref{lemma: concentration of the bonus term}, with probability at least $1-\delta/2$, we have that $\forall n \in \mathbb{N}^+, h \in [H], \phi \in \Phi,$
 		\begin{equation*}
 		\min\left\{\alpha_n \left\|\hphi_{h}^{(n)}(s_{h},a_{h})\right\|_{(U_{h,\hphi}^{(n,t)})^{-1}},B\right\} \leq \hb_h^{(n,t)}(s_h,a_h) \leq \frac{\beta_2}{\beta_1} \alpha_n \left\|\hphi_{h}^{(n)}(s_{h},a_{h})\right\|_{(U_{h,\hphi}^{(n,t)})^{-1}}.
 		\end{equation*}
 	\end{corollary}
 	Recall that $f_h^{(n,t)}(s,a)=\|\hP_h^{(n,t)}(\cdot|s,a) - P^{(\star,t)}_h(\cdot|s,a)\|_{TV}$ represents the estimation error of task $t$ in terms of the total variation distance in the $n$-th iteration at step $h$, given state $s$ and action $a$ in \Cref{alg: Upstream}. Inspired by the proof of Theorem 21 in \citet{NEURIPS2020_e894d787}, We show that if we uniformly choose the exploration policies for each task, the summation of the estimation error can be bounded with high probability. 
 	\begin{lemma}[Multitask MLE guarantee]\label{lemma: multitask MLE Guarantee}
 		Given $\delta\in(0,1)$, consider the transition kernels learned from line \ref{line: MLE} and \ref{line: transition} in \Cref{alg: Upstream}, we have the following inequality holds for any $n,h\geq 2$ with probability at least $1-\delta/2$:
 		\begin{align}
 		\sum_{t=1}^{T}\mathop{\Eb}_{s_{h-1}\sim (P^{(\star,t)},\Pi_t^n)\atop {a_{h-1},a_h\sim \Uc(\Ac)
 				\atop s_h\sim P^{(\star,t)}(\cdot|s_{h-1},a_{h-1})}}\left[f_h^{(n,t)}(s_h,a_h)^2\right]
 		\leq \zeta_n, \quad\mbox{ where } \zeta_n : = \frac{2\log\left(2|\Phi||\Psi|^T nH/\delta\right)}{n}.
 		\end{align} 
 		In addition, for $h=1$,
 		\begin{align*}
 		\sum_{t=1}^T\mathop{\Eb}_{a_1 \sim \Uc(\Ac)}\left[f_1^{(n,t)}(s_1,a_1)^2\right]\leq \zeta_n.
 		\end{align*}
 		Furthermore, define 
 		\begin{align}
 		&\zeta_h^{(n,t)} = \mathop{\Eb}_{s_{h-1}\sim (P^{(\star,t)},\Pi_t^n)\atop {a_{h-1},a_h\sim \Uc(\Ac)
 		\atop s_h\sim P^{(\star,t)}(\cdot|s_{h-1},a_{h-1})}}\left[f_h^{(n,t)}(s_h,a_h)^2\right],h\geq 2,\\
 	    &{\zeta_1^{(n,t)}}={\mathop{\Eb}_{s_1\sim(P^{(\star,t)},\pi_t) \atop a_1 \sim \Uc(\Ac)} \left[f_1^{(n,t)}(s_1,a_1)^2\right]}.
 		\end{align}
 		We have 
 		\begin{align}
 		\zeta_h^{(n,t)} \leq \sum_{t=1}^T \zeta_h^{(n,t)} \leq \zeta_n  = \frac{2\log\left(2|\Phi||\Psi|^T nH/\delta\right)}{n}.
 		\end{align}
 	\end{lemma}
 	
 	\begin{proof}[Proof of \Cref{lemma: multitask MLE Guarantee}] 
 Consider a sequential conditional probability estimation setting with an instance space $\mathcal{X}$ and a target space $\mathcal{Y}$ where the conditional density is given by $p(y | x) = f ^\star(x, y)$. We are given a dataset $D:= \{(x_i,y_i)\}_{i=1}^n$ , where $x_i \sim \Dc_i = \Dc_i(x_{1:i-1},y_{1:i-1})$ and $y_i \sim p(\cdot | x_i)$. Let $D^\prime$ denote a tangent sequence $\{(x_i^\prime,y_i^\prime)\}_{i=1}^n$ where $x_i^\prime \sim \Dc_i(x_{1:i-1},y_{1:i-1})$ and $y_i^\prime \sim p(\cdot|x_i^\prime)$. Further, we consider a function class $\mathcal{F}: (\mathcal{X} \times \mathcal{Y}) \rightarrow R$ and assume that the reachability condition $f^\star \in \mathcal{F}$ holds.	
 
 	We first introduce two useful lemmas from \cite{NEURIPS2020_e894d787}.
 	\begin{lemma}[Lemma 25 of \cite{NEURIPS2020_e894d787}]\label{lemma: Hellinger distance}
 		For any two conditional probability densities $f_1,f_2$ and any distribution $\Dc \in \triangle(\mathcal{X})$, we have
 		\begin{align*}
 		\Eb_{x \sim D}\|f_1(x,\cdot)-f_2(x,\cdot)\|^2_{TV} \leq -2 \log \Eb_{x \sim \Dc, y \sim f_2(\cdot|x)}\left[\exp\left(-\frac{1}{2}\log(f_2(x,y)/f_1(x,y))\right)\right].
 		\end{align*}
 	\end{lemma}
 	\begin{lemma}[Lemma 24 of \cite{NEURIPS2020_e894d787}]\label{lemma: decouple}
 		Let $D$ ba a dataset of $n$ samples and $D^\prime$ be corresponding tangent sequence. Let $L(f,D)=\sum_{i=1}^{n}l(f,(x_i,y_i))$ be any function that decomposes additively across examples where $l$ is any function, and let $\widehat{f}(D)$ be any estimator taking as input random variable $D$ and with range $\mathcal{F}$. Then
 		\begin{align*}
 		\Eb_{D}\left[\exp\left(L(\widehat{f}(D),D)-\log\Eb_{D^\prime}\left[\exp(L(\widehat{f}(D),D^\prime))\right]-\log|\mathcal{F}|\right)\right] \leq 1.
 		\end{align*}
 	\end{lemma}
Suppose $\widehat{f}(D)$ is learned from the following maximum likelihood problem:
\begin{align}
\widehat{f}(D):= {\arg\max}_{f \in \mathcal{F}}\sum_{(x_i,y_i)\in D}\log f(x_i,y_i).
\end{align}
Combining Chernoff method and \Cref{lemma: decouple}, we obtain an exponential tail bound, i.e., with probability at least $1-\delta$, 
 		\begin{align}
 		-\log\Eb_{D^\prime}\left[\exp(L(\widehat{f}(D),D^\prime))\right] \leq -L(\widehat{f}(D),D)+\log|\mathcal{F}|+\log(1/\delta). \label{ineq: chernoff in lemma A.8}
 		\end{align}
To proceed, we let $L(f,D)=\sum_{i=1}^{n} - \frac{1}{2} \log(f^\star(x_i,y_i)/f(x_i,y_i))$ where $D$ is a dataset $\{(x_i,y_i)\}_{i=1}^n$(and $D^\prime=\{(x_i^\prime,y_i^\prime)\}_{i=1}^n$ is tangent sequence). In multitask RL setting, let $x=\{(s^t,a^t)\}_{t=1}^T,y=\{(s')^t\}_{t=1}^T$ and $f(x,y)= \prod_{t=1}^T P^t[(s')^t|s^t,a^t]$. Then, dataset $D$ can be decomposed into $D=\bigcup_{t=1}^T D^t$ where $D^t=\{s^t_i,a^t_i,(s')^t_i\}_{i=1}^{n}$. Similarly $D^\prime=\bigcup_{t=1}^T (D')^t$, and $\Dc_i^t:=\Dc_i^t(s^t_{1:i-1},a^t_{1:i-1},(s')^t_{1:i-1})$. Hence, the cardinality $|\mathcal{F}|=|\Phi||\Psi|^T$ in the multitask setting.

Then, the RHS of \Cref{ineq: chernoff in lemma A.8} can be bounded as
 		\begin{align}
 \text{RHS of \Cref{ineq: chernoff in lemma A.8}} & =	\sum_{i=1}^{n} \frac{1}{2} \log(f^\star(x_i,y_i)/\widehat{f}(x_i,y_i))+\log|\mathcal{F}|+\log(1/\delta) \nonumber \\
 &\leq \log|\mathcal{F}|+\log(1/\delta)={\log\left(|\Phi||\Psi|^T /\delta\right)} \label{ineq: RHS of Lemma A.8},
 		\end{align}
 		where the inequality follows because $\widehat{f}$ is MLE and from the assumption of reachability, and the last equality follows because $|\mathcal{F}|=|\Phi||\Psi|^T$.
 		
 Next, the LHS of \Cref{ineq: chernoff in lemma A.8} can be bounded as
 		\begin{align}
 \text{LHS of \Cref{ineq: chernoff in lemma A.8}}	& \overset{\RM{1}}{=}- \log \Eb_{D^\prime}\left[\exp\left(\sum_{i=1}^n-\frac{1}{2}\log\left(\frac{f^\star(x_i^\prime,y_i^\prime)}{\widehat{f}(x_i^\prime,y_i^\prime)}\right)\right)\bigg|D\right]\nonumber\\
 		& \overset{\RM{2}}{=} - \log \Eb_{D^\prime}\left[\exp\left(\sum_{i=1}^n-\frac{1}{2}\log\left(\prod_{t=1}^T\frac{P^{(\star,t)}[(s')^t_i|s^t_i,a^t_i]}{\widehat{P}^{(n,t)}[(s')^t_i|s^t_i,a^t_i]}\right)\right)\bigg|D\right]\nonumber\\
 		& \overset{(\romannumeral3)}{=}  -\sum_{t=1}^T \log \Eb_{(D')^t}\left[\exp\left(\sum_{i=1}^n-\frac{1}{2}\log\left(\frac{P^{(\star,t)}[(s')^t_i|s^t_i,a^t_i]}{\widehat{P}^{(n,t)}[(s')^t_i|s^t_i,a^t_i]}\right)\right)\bigg|D\right]\nonumber\\
 		& \overset{(\romannumeral4)}{=}  -\sum_{t=1}^T \sum_{i=1}^n \log \Eb_{D_i^{t}}\left[\exp\left(-\frac{1}{2}\log\left(\frac{P^{(\star,t)}[(s')^t_i|s^t_i,a^t_i]}{\widehat{P}^{(n,t)}[(s')^t_i|s^t_i,a^t_i]}\right)\right)\right]\nonumber\\
 		& \overset{(\romannumeral5)}{\geq} \sum_{t=1}^T\frac{1}{2}\sum_{i=1}^n\Eb_{(s,a) \sim \Dc_i^t}\left\|\widehat{P}^{(n,t)}(\cdot|s,a)-P^{(\star,t)}(\cdot|s,a)\right\|^2_{TV}\nonumber\\
 		& \overset{(\romannumeral6)}{=} 
 		\frac{n}{2}\sum_{t=1}^{T}\mathop{\Eb}_{s_{h-1}\sim (P^{(\star,t)},\Pi_t^n)\atop {a_{h-1},a_h\sim \Uc(\Ac)\atop s_h\sim P^{(\star,t)}(\cdot|s_{h-1},a_{h-1})}}\left[f_h^{(n,t)}(s_h,a_h)^2\right]\label{ineq: LHS of Lemma A.8},
 		\end{align}
 		where $\RM{1}$ follows from the above definition of $L(f,D)$, $\RM{2}$ follows from the above definition of $f(x,y)$, $(\romannumeral3)$ follows because the data of $T$ tasks are independent conditional on $D$, $\RM{4}$ follows because $\hP^{(n,t)}$ is independent of the dataset $(D')^t$ and from the definition of $D^\prime$, $(\romannumeral5)$ follows from \Cref{lemma: Hellinger distance}, and $(\romannumeral6)$ follows because the data collected in $i$-th iteration uses policy $\pi_{i-1}^t$ followed by two steps of uniform random actions and from the definition of $\Pi_n^t$.

Combining \Cref{ineq: chernoff in lemma A.8,ineq: RHS of Lemma A.8,ineq: LHS of Lemma A.8}, we have 
 		\begin{align}
 		\frac{n}{2}\sum_{t=1}^{T}\mathop{\Eb}_{s_{h-1}\sim (P^{(\star,t)},\Pi_t^n)\atop {a_{h-1},a_h\sim \Uc(\Ac)
 				\atop s_h\sim P^{(\star,t)}(\cdot|s_{h-1},a_{h-1})}}\left[f_h^{(n,t)}(s_h,a_h)^2\right] \leq {\log\left(|\Phi||\Psi|^T /\delta\right)}
 		. \label{ineq: fixed version lemma 3}
 		\end{align}
 		{We substitute $\delta$ with ${\delta}/{2nH}$ to ensure \Cref{ineq: fixed version lemma 3} holds for any $h \in [H]$ and $n$ with probability at least $1-\delta/2$, which finishes the proof.} 
 	\end{proof}

We next introduce a one-step back lemma, which extends the one-step back inequality for infinite-horizon stationary MDP in  \citet{uehara2021representation,NEURIPS2020_e894d787} to non-stationary transition kernels with finite horizon. The lemma shows that for any function $g \in \Sc \times \Ac \rightarrow \Rb$, policy $\pi$ and transition kernel $P$, 
we can upper bound the expectation $\scriptstyle \mathop{\Eb}_{s_h \sim (P, {\pi}) \atop a_h \sim \pi}[g(s_h,a_h)]$ by the product of two terms. The first term represents the convergence guarantee of $g(s_h,a_h)$ following other policies, which is $\scriptstyle \mathop{\Eb}_{s_{h}\sim(P^\star,\Pi)\atop a_h \sim \Uc(\Ac) }[g^2(s_h,a_h)]$. The second term can be described as the distribution shift coefficient $\scriptstyle \mathop{\Eb}_{s_{h-1}\sim (P, \pi) \atop a_{h-1} \sim \pi}\left[\left\|\phi_{h-1}(s_{h-1},a_{h-1})\right\|_{(U_{h-1,\phi})^{-1}}\right]$, which measures the difference caused by distribution shift from $\pi$ and other policies.
\begin{lemma}[One-step back inequality for non-stationary finite-horizon MDP]
\label{lemma:Step_Back}
 For each task $t$, let $P\in\{\hP^{(n,t)},P^{(\star,t)}\}$ with embeddings $\phi$ and $\mu$ be a generic MDP model, and $U^{t}_{h,\phi} = \lambda I + n\Eb_{s_h,a_h\sim (P^{(\star,t)},\Pi)}[\phi\phi^\top]\in\{U_{h,\phi}^{(n,t)}, W_{h,\phi}^{(n,t)}\}$ be the covariance matrix following a generic policy $\Pi$ under the true environment $P^{(\star,t)}$. Note that $\phi\in\{\hphi^{(n)},\phi^\star\}$ corresponds to $P$. Further, let $f^t(s_h,a_h)$ be the total variation between $P^{(\star,t)}$ and $P$ at time step $h$. Take any $g \in \mathcal{S} \times \mathcal{A} \rightarrow \mathbb{R}$ such that $\|g\|_\infty \leq B_g$, i.e., $\sup_{s,a}|g(s,a)|\leq B_g$. Then, $\forall h \geq 2, \forall\, {\rm policy }\,\pi$,
 		\begin{align*}
 		\mathop{\Eb}_{s_h \sim (P, {\pi}) \atop a_h \sim \pi}[g(s_h,a_h)]  &\leq \mathop{\Eb}_{s_{h-1}\sim (P, \pi) \atop a_{h-1} \sim \pi}\left[\left\|\phi_{h-1}(s_{h-1},a_{h-1})\right\|_{(U^t_{h-1,\phi})^{-1}} \times\right.\nonumber\\
 		&\quad\scriptstyle\left.
 		\sqrt{nK\mathop{\Eb}_{s_{h}\sim(P^{(\star,t)},\Pi)\atop a_h \sim \Uc(\Ac) }[g^2(s_h,a_h)]+\lambda dB_g^2 + nB_g^2\mathop{\Eb}_{s_{h-1}\sim(P^{(\star,t)},\Pi)\atop a_{h-1}\sim\Pi }\left[f^t(s_{h-1},a_{h-1})^2\right]}\right].
 		\end{align*}
 	\end{lemma}
 	\begin{proof}
First, we have
    \begin{align*}
    &\quad \ \mathop{\Eb}_{s_h\sim(P,\pi) \atop a_h \sim \pi}\left[ 
     g(s_h,a_h)\right]\nonumber\\
    &=\mathop{\Eb}_{s_{h-1} \sim (P,\pi) \atop a_{h-1} \sim \pi} \left[\int_{s_h}\sum_{a_h}g(s_h,a_h)\pi(a_h|s_h)\langle\phi_{h-1}(s_{h-1},a_{h-1}),\mu_{h-1}(s_h)\rangle d{s_h}\right]\\
    & \leq\mathop{\Eb}_{s_{h-1} \sim (P, \pi) \atop a_{h-1}\sim\pi} \left[\left\|\phi_{h-1}(s_{h-1},a_{h-1})\right\|_{(U^t_{h-1,\phi})^{-1}}\left\|\int\sum_{a_h}g(s_h,a_h)\pi(a_h|s_h)\mu_{h-1}(s_h)d{s_h}\right\|_{U^t_{h-1,\phi}}\right],
    \end{align*}
where the inequality follows from Cauchy's inequality. We further develop the following bound:
    \begin{align*}
       &\hspace{-5mm} \left\|\int\sum_{a_h}g(s_h,a_h)\pi(a_h|s_h)\mu_{h-1}(s_h)d{s_h}\right\|_{U^t_{h-1,\phi}}^2\\
        & \overset{(\romannumeral1)}{\leq} n \mathop{\Eb}_{s_{h-1}\sim (P^{(\star,t)},\Pi) \atop a_{h-1}\sim \Pi}
        \left[\left(\int_{s_h}\sum_{a_h}g(s_h,a_h)\pi(a_h|s_h)\mu(s_h)^\top\phi(s_{h-1},a_{h-1})d{s_h}\right)^2\right] + \lambda d B_g^2\\
        & \leq n \mathop{\Eb}_{s_{h-1}\sim(P^{(\star,t)},\Pi) \atop a_{h-1} \sim \Pi } \left[\mathop{\Eb}_{s_h \sim P(\cdot|s_{h-1},a_{h-1}) \atop a_h \sim \pi} \left[g(s_h,a_h)^2\right]\right] + \lambda d B_g^2\\
        &\overset{(\romannumeral2)}{\leq} n \mathop{\Eb}_{s_{h-1}\sim(P^{(\star,t)},\Pi) \atop a_{h-1}\sim\Pi }\left[\mathop{\Eb}_{s_h\sim P^{(\star,t)} \atop a_h\sim \pi}\left[g(s_h,a_h)^2\right]\right] + \lambda d B_g^2 + nB_g^2\mathop{\Eb}_{s_{h-1}\sim(P^{(\star,t)},\Pi) \atop a_{h-1}\sim\Pi }\left[f^t(s_{h-1},a_{h-1})^2\right]\\
        &\overset{(\romannumeral3)}{\leq} n K\mathop{\Eb}_{s_{h}\sim(P^{(\star,t)},\Pi)\atop a_h\sim \Uc(\Ac) }\left[g(s_h,a_h)^2\right] + \lambda d B_g^2 + nB_g^2\mathop{\Eb}_{s_{h-1}\sim(P^{(\star,t)},\Pi) \atop a_{h-1}\sim \Pi }\left[f^t(s_{h-1},a_{h-1})^2\right],
    \end{align*}
where $(\romannumeral1)$ follows from the assumption $\|g\|_{\infty}\leq B_g$, $(\romannumeral2)$ follows because $f(s_h,a_h)$ is the total variation between $P^\star$ and $P$ at time step $h$, and $(\romannumeral3)$ follows from importance sampling.
This finishes the proof.
\end{proof}

\section{Proof of \Cref{lemma: Approximate Feature for new task}}

\Cref{lemma: Approximate Feature for new task} serves a central role for bridging the upstream and downstream learning, which shows that the feature $\hphi$ learned in upstream is a $\xi_{down}$-approximate feature map and can approximate the true feature in the new task. 

 	\begin{proof}[Proof of \Cref{lemma: Approximate Feature for new task}]
 		Under \Cref{assumption: reachability,assumption: finite measurement,assumption: smooth TV distance}, for any $t \in[T]$, we have
 		\begin{align}
 		&\max_{s \in \Sc, a \in \Ac}\|P^1_h(\cdot|s,a)-P^2_h(\cdot|s,a)\|_{TV}  \nonumber \\
 		& \overset{\RM{1}}{\leq} C_R \mathop{\Eb}_{(s_h,a_h)\sim \Uc(\Sc,\Ac)}{\|P^1_h(\cdot|s_h,a_h)-P^2_h(\cdot|s_2,a_2)\|_{TV}}\nonumber\\
 		& \overset{\RM{2}}{\leq} \frac{C_R\upsilon}{\kappa_u}\mathop{\Eb}_{s_h \sim (P^{(\star,t)},\pi_t^0) \atop a_h \sim \Uc(\Ac)}\left[\left\|P^1_h(\cdot|s,a)-P^2_h(\cdot|s,a)\right\|_{TV}\right], \label{ineq: lemma1eq1}
 		\end{align}
 		where $\RM{1}$ follows from \Cref{assumption: smooth TV distance} and $\RM{2}$ follows \Cref{assumption: reachability} and \Cref{assumption: finite measurement}.
 		
 		Then, $\forall (s, a) \in \Sc \times \Ac, h \in [H] $, we have 
 		\begin{align}
 		\sum_{t=1}^T \|\hP_h^{(t)}(\cdot|s,a)-P^{(\star,t)}_h(\cdot|s,a)\|_{TV}
 		& \leq \sum_{t=1}^T\max_{s \in \Sc, a \in \Ac}\|\hP_h^{(t)}(\cdot|s,a)-P^{(\star,t)}_h(\cdot|s,a)\|_{TV}\nonumber\\ 
 		& \overset{\RM{1}}{\leq} \frac{C_R\upsilon}{\kappa_u}\sum_{t=1}^T\mathop{\Eb}_{s_h \sim (P^{(\star,t)},\pi_t^0) \atop a_h \sim \Uc(\Ac)}\left[\left\|\hP_h^{(t)}(\cdot|s,a)-P^{(\star,t)}_h(\cdot|s,a)\right\|_{TV}\right]\nonumber\\
 		& \overset{(\romannumeral2)}{\leq} \frac{C_RT\upsilon\epsilon_u}{\kappa_u}, \label{ineq: uniform point-wise TV bound}
 		\end{align}
 		where $\RM{1}$ follows from \Cref{ineq: lemma1eq1}, and $(\romannumeral2)$ follows from \Cref{theorem: Upstream sample complexity}.
 		
 		Define $\hmu^\star(\cdot)=\sum_{t=1}^Tc_t\hmu^{(t)}(\cdot)$, then we have
 		\begin{align*}
 		& \left\|P_h^{(\star,T+1)}(\cdot|s,a)-\left\langle\hphi_h(s,a),\hmu_h^\star(\cdot)\right\rangle\right\|_{TV}\\
 		& \quad = \left\|P_h^{(\star,T+1)}(\cdot|s,a)-\left\langle\hphi_h(s,a),\sum_{t=1}^Tc_t\hmu_h^{(t)}(\cdot)\right\rangle\right\|_{TV}\\
 		& \quad \leq \left\|P_h^{(\star,T+1)}(\cdot|s,a)-\sum_{t=1}^Tc_t\hP_h^{(t)}(\cdot|s,a)\right\|_{TV}\\
 		& \quad \leq \left\|P_h^{(\star,T+1)}(\cdot|s,a)-\sum_{t=1}^Tc_tP_h^{(\star,t)}(\cdot|s,a)\right\|_{TV}+\sum_{t=1}^Tc_t\left\|P_h^{(\star,t)}(\cdot|s,a)-\hP_h^{(t)}(\cdot|s,a)\right\|_{TV}\\
 		& \quad \overset{(\romannumeral1)}{\leq} \xi+\frac{C_LC_RT\upsilon\epsilon_u}{\kappa_u},
 		\end{align*}
 		where $(\romannumeral1)$ follows from \Cref{assumption: Linear combination}, \Cref{ineq: uniform point-wise TV bound} and the fact that $c_t \in [0,C_L]$. 
 		
 		Furthermore, by normalization for any $g: \Sc \rightarrow [0,1]$, we obtain
 		\begin{align*}
 		     \norm{\int \hmu^\star_h(s)g(s)ds }_2\leq \sum_{t=1}^Tc_t\norm{\int \hmu^{(t)}(s)g(s)ds }_2\leq C_L\sqrt{d}.
 		\end{align*}
 	\end{proof}

\section{Proof of Theorem~\ref{thm:offline}}\label{sec: downstream offline}
Recall $\xioff=\xi+\frac{C_LC_RT v\epsilon_u}{\kappa_u}$ and for any $h \in [H]$, we define
\begin{align*}
&P_h^{(\star,T+1)}(\cdot|s,a)=\langle\phi_h^\star(s,a),\mu_h^{(\star,T+1)}(\cdot)\rangle, \\
&\overline{P}_h(\cdot|s,a)=\langle\widehat{\phi}_h(s,a),\widehat{\mu}^*_h(\cdot)\rangle.  
\end{align*}
Given a reward function $r$, for any function $f:\mathcal{S}\mapsto \mathbb{R}$ and $h\in[H]$, we define the transition operators and their corresponding Bellman operators as
\begin{align*}
&(P_h^{(\star,T+1)}\hspace{-0.02in}f)(s,a)\hspace{-0.02in}=\hspace{-0.045in}\int_{s'}\hspace{-0.02in}\langle\phi_h^\star(s,a),\mu_h^{(\star,T+1)}\hspace{-0.02in}(s')\rangle f(s')ds'\hspace{-0.03in}, \\ &(\mathbb{B}_hf)(s,a)\hspace{-0.01in}=\hspace{-0.02in}r_h(s,a)\hspace{-0.02in}+\hspace{-0.02in}(P_h^{(\star,T+1)}\hspace{-0.02in}f)(s,a), \\
&(\overline{P}_hf)(s,a)=\int_{s'}\langle\widehat{\phi}_h(s,a),\widehat{\mu}_h^\star(s')f(s')\rangle ds', \\ &(\overline{\mathbb{B}}_hf)(s,a)=r_h(s,a)+(\overline{P}_hf)(s,a).
\end{align*}
We further denote $(\widehat{\mathbb{B}}_h \widehat{V}_{h+1})(s,a)=r_h(s,a)+\widehat{\phi}_h(s,a)^\top \widehat{w}_h$, $h \in [H]$.

We remark that throughout this section, the expectation is taken with respect to the transition kernel of the target task, i.e., $P^{(\star,T+1)}$.

{\bf Proof Overview:} The proof of \Cref{thm:offline} consists of two main steps and a final suboptimality gap characterization. {\bf Step 1:} We decompose the suboptimality gap into the summation of the uncertainty metric of each step in \Cref{lemma:offline-pessimism}. We note that the reward function $r_h$ here is from a general class, not necessarily a linear function. {\bf Step 2:} We provide an upper bound on the Bellman update error as shown in \Cref{lemma:bellmanerrorbound}, where our main technical contribution lies in capturing the impact of the misspecification of the representation taken from upstream estimation on such error. {\bf Suboptimality gap characterization:} Based on the first two steps, we select uncertainty metric $\Gamma_h$ and obtain an instance-dependent suboptimality gap, which we further bound under the feature coverage assumption. 

We provide details for the two main steps and the suboptimality gap characterization in \Cref{app:theorem2step1}-\Cref{app:theorem2step3}, respectively.

\subsection{Suboptimality decomposition}\label{app:theorem2step1}

In this step, we decompose the suboptimality gap into the summation of the uncertainty metric of each step in \Cref{lemma:offline-pessimism}. We note that the reward function $r$ here is from a general class, not necessarily a linear function. To this end, we first provide the following lemma.
\begin{lemma}\label{lemma:offline-BV_bound}
If $|(\mathbb{B}_h\widehat{V}_{h+1}-\widehat{\mathbb{B}}_h\widehat{V}_{h+1})(s,a)|\leq \Gamma_h(s,a)$ for all $(h,s,a)\in [H]\times\mathcal{S}\times\mathcal{A}$, then it holds that $(\mathbb{B}_h\widehat{V}_{h+1})(s,a)\leq 1$, $\forall (h,s,a)\in [H]\times\mathcal{S}\times\mathcal{A}$.
\end{lemma}
\begin{proof} 
It suffices to show 
\begin{align*}
(\mathbb{B}_{h}\widehat{V}_{h+1})(s,a)\leq \max_{a_{h+1},\ldots,a_H}\mathbb{E}\left[\sum_{h'=h}^{H}r_{h'}(s_{h'},a_{h'})\middle|s_{h}=s,a_{h}=a\right].
\end{align*}
We prove it by induction. For $h'=H$, since $\widehat{V}_{H+1}=0$, we have
\begin{align*}
(\mathbb{B}_H\widehat{V}_{H+1})(s,a)= r_H(s,a)+(P_{H}^{(\star,T+1)}\widehat{V}_{H+1})(s,a)=r_{H}(s,a).
\end{align*}

Suppose for $h'=h+1$, $h\in[H-1]$, we have 
\begin{align*}
(\mathbb{B}_{h+1}\widehat{V}_{h+2})(s,a)\leq \max_{a_{h+2},\ldots,a_H}\mathbb{E}\left[\sum_{h'=h+1}^{H}r_{h'}(s_{h'},a_{h'})\middle|s_{h+1}=s,a_{h+1}=a\right],
\end{align*}
which is bounded in $[0,1]$ since $r_h\geq 0$, $\forall h$, and for any trajectory it holds that $\sum_{h=1}^{H}r_h\leq 1$.

Further note that 
\begin{align*}
\widehat{Q}_{h+1}(s,a)&=\min\{r_{h+1}(s,a)+\widehat{w}_{h+1}^\top\widehat{\phi}_{h+1}(s,a)-\Gamma_{h+1}(s,a),1\}^+ \\
&\overset{\RM{1}}{\leq}\min\{(\mathbb{B}_{h+1}\widehat{V}_{h+2})(s,a),1\}^+ \\
&\overset{\RM{2}}{\leq}\max\{0,(\mathbb{B}_{h+1}\widehat{V}_{h+2})(s,a)\} \\
&\overset{\RM{3}}{\leq} \max_{a_{h+2},\ldots,a_H}\mathbb{E}\left[\sum_{h'=h+1}^{H}r_{h'}(s_{h'},a_{h'})\middle|s_{h+1}=s,a_{h+1}=a\right] , 
\end{align*}
where $\RM{1}$ follows from the assumption $|(\mathbb{B}_h\widehat{V}_{h+1}-\widehat{\mathbb{B}}_h\widehat{V}_{h+1})(s,a)|\leq \Gamma_h(s,a)$, $\RM{2}$ follows because $(\mathbb{B}_{h+1}\widehat{V}_{h+2})(s,a)\leq 1$ by the induction hypothesis, and $\RM{3}$ follows from the fact that $r_h\geq 0$, $\forall h$.

Therefore, for $h'=h$, we have
\begin{align*}
&(\mathbb{B}_h\widehat{V}_{h+1})(s,a)=r_{h}(s,a)+(P_h^{(*,T+1)}\widehat{V}_{h+1})(s,a) \\
&=r_{h}(s,a)+\int_{s'} P_h^{(*,T+1)}(s'|s,a)\widehat{V}_{h+1}(s')ds' \\
&\leq r_{h}(s,a)+\int_{s'} P_h^{(*,T+1)}(s'|s,a)\max_{a'} \widehat{Q}_{h+1}(s',a') ds' \\
&\leq r_h(s,a)+\int_{s'}ds'P_h^{(*,T+1)}(s'|s,a)\max_{a',a_{h+2},\ldots,a_H}\mathbb{E}\left[\sum_{h'=h+1}^{H}r_{h'}(s_{h'},a_{h'})\middle|s_{h+1}=s',a_{h+1}=a'\right] \\
&\leq r_h(s,a)+\int_{s'}ds'P_h^{(*,T+1)}(s'|s,a)\max_{a_{h+1},a_{h+2},\ldots,a_H}\mathbb{E}\left[\sum_{h'=h+1}^{H}r_{h'}(s_{h'},a_{h'})\middle|s_{h+1}=s'\right] \\
&\leq \max_{a_{h+1},\ldots,a_H}\mathbb{E}\left[\sum_{h'=h}^{H}r_{h'}(s_{h'},a_{h'})\middle|s_{h}=s,a_h=a\right].
\end{align*}
By backward induction from $H$ to $1$, the proof is complete.
\end{proof}

We denote the Bellman update error as $\zeta_h(s,a)=(\mathbb{B}_h\widehat{V}_{h+1})(s,a)-\widehat{Q}_h(s,a)$.
The following lemma shows that it suffices to bound the pessimistic penalty.
\begin{lemma}\label{lemma:offline-pessimism}
Suppose with probability at least $1-\delta$, for all $(h,s,a)\in [H]\times\mathcal{S}\times\mathcal{A}$, it holds that $|(\mathbb{B}_h\widehat{V}_{h+1}-\widehat{\mathbb{B}}_h\widehat{V}_{h+1})(s,a)|\leq \Gamma_h(s,a)$. Denote $\{\widehat{\pi}_h\}_{h=1}^H$ as the output of \Cref{alg: Offline RL}. Then, with probability at least $1-\delta$, for any $(h,s,a)\in[H]\times\mathcal{S}\times\mathcal{A}$, we have $0\leq \zeta_h(s,a)\leq 2\Gamma_h(s,a)$. Moreover, it holds that for any policy $\pi$, with probability at least $1-\delta$,
\begin{align*}
V_{P^{(*,T+1)},r}^\pi(s)-V_{P^{(*,T+1)},r}^{\widehat{\pi}}(s)\leq 2\sum_{h=1}^{H}\mathbb{E}_{\pi}[\Gamma_h(s_h,a_h)|s_1=s].
\end{align*}
\end{lemma}
\begin{proof}
First, we show that $\zeta_h(s,a)\geq 0$.
 Recall that
 \begin{align*}
     \widehat{Q}_h(\cdot,\cdot)\hspace{-0.02in}=\hspace{-0.02in}\min\{r_h(\cdot,\cdot)\hspace{-0.02in}+\hspace{-0.02in}\hphi_h(\cdot,\cdot)\widehat{w}_h\hspace{-0.02in}-\hspace{-0.02in}\Gamma_h(\cdot,\cdot),1\}^+.
 \end{align*}

If $r_h(s,a)+\widehat{\phi}_h(s,a)^\top\widehat{w}_h-\Gamma_h(s,a)\leq 0$, then $\widehat{Q}_h(s,a)=0$, which implies that $\zeta_h(s,a)=(\mathbb{B}_h\widehat{V}_{h+1})(s,a)-\widehat{Q}_h(s,a)=(\mathbb{B}_h\widehat{V}_{h+1})(s,a)\geq 0$. 

If $r_h(s,a)+\widehat{\phi}_h(s,a)^\top\widehat{w}_h-\Gamma_h(s,a)> 0$, then $\widehat{Q}_h\leq r_h(s,a)+\widehat{\phi}_h(s,a)^\top\widehat{w}_h-\Gamma_h(s,a)=(\widehat{\mathbb{B}}_h\widehat{V}_{h+1})(s,a)-\Gamma_h(s,a)$, which implies that
\begin{align*}
\zeta_h(s,a)= (\mathbb{B}_h\widehat{V}_{h+1})(s,a)-\widehat{Q}_h(s,a)\geq (\mathbb{B}_h\widehat{V}_{h+1})(s,a)-(\widehat{\mathbb{B}}_h\widehat{V}_{h+1})(s,a)+\Gamma_h(s,a)\geq 0.
\end{align*}

We next show that $\zeta_h(s,a)\leq 2\Gamma_h(s,a)$. Note that
\begin{align*}
r_h(s,a)+\widehat{\phi}_h(s,a)^\top\widehat{w}_h-\Gamma_h(s,a)
\overset{\RM{1}}{=}(\widehat{\mathbb{B}}_h\widehat{V}_{h+1})(s,a)-\Gamma_h(s,a)\overset{\RM{2}}{\leq} (\mathbb{B}_h\widehat{V}_{h+1})(s,a)\overset{\RM{3}}{\leq} 1,
\end{align*}
where $\RM{1}$ follows from the definition of $(\widehat{\mathbb{B}}_h\widehat{V}_{h+1})(s,a)$, $\RM{2}$ follows because $|(\mathbb{B}_h\widehat{V}_{h+1}-\widehat{\mathbb{B}}_h\widehat{V}_{h+1})(s,a)|\leq \Gamma_h(s,a)$, and $\RM{3}$ follows from \Cref{lemma:offline-BV_bound}.

Therefore, 
\begin{align*}
\widehat{Q}_h(s,a)&=\min\{r_h(s,a)+\widehat{\phi}_h(s,a)^\top\widehat{w}_h-\Gamma_h(s,a),1\}^+ \\
&=\max\{r_h(s,a)+\widehat{\phi}_h(s,a)^\top\widehat{w}_h-\Gamma_h(s,a),0\} \\
&\geq r_h(s,a)+\widehat{\phi}_h(s,a)^\top\widehat{w}_h-\Gamma_h(s,a) \\
&=(\widehat{\mathbb{B}}_h\widehat{V}_{h+1})(s,a)-\Gamma_h(s,a).
\end{align*}

By the definition of $\zeta_h$, we have
\begin{align*}
\zeta_h(s,a)&=(\mathbb{B}_h\widehat{V}_{h+1})(s,a)-\widehat{Q}_h(s,a) \\
&\leq (\mathbb{B}_h\widehat{V}_{h+1})(s,a)-(\widehat{\mathbb{B}}_h\widehat{V}_{h+1})(s,a)+\Gamma_h(s,a) \\
&\leq 2\Gamma_h(s,a).
\end{align*}

Then we obtain
\begin{align}
&V_{P^{(*,T+1)},r}^\pi(s)-V_{P^{(*,T+1)},r}^{\widehat{\pi}}(s)\nonumber\\
& \quad \overset{\RM{1}}{\leq}
\sum_{h=1}^{H}\mathbb{E}_\pi[\zeta_h(s_h,a_h)|s_1=s]-\sum_{h=1}^{H}\mathbb{E}_{\widehat{\pi}}[\zeta_h(s_h,a_h)|s_1=s] \label{eqn: offline-BV-1}\\
&\quad \overset{\RM{2}}{\leq} 2\sum_{h=1}^{H}\mathbb{E}_{\pi}[\Gamma_h(s_h,a_h)|s_1=s],\nonumber
\end{align}
where $\RM{1}$ follows from \Cref{lemma:offline-decom} and definition of $\widehat{\pi}$, and $\RM{2}$ follows because with probability at least $1-\delta$, for all $(h,s,a)\times[H]\times\mathcal{S}\times\mathcal{A}$, $0\leq \zeta_h(s,a)\leq 2\Gamma_h(s,a)$ holds.
\end{proof}

\subsection{Bounding Bellman update error $|(\mathbb{B}_h\widehat{V}_{h+1}-\widehat{\mathbb{B}}_h\widehat{V}_{h+1})(s,a)|$}\label{app:theorem2step2}

In this step, we provide an upper bound on the Bellman update error as shown in \Cref{lemma:bellmanerrorbound}, where the main effort lies in analyzing the impact of the misspecification of the representation taken from upstream estimation.
To this end, we first introduce a concentration lemma that upper-bounds the stochastic noise in regression.
\begin{lemma}\label{lemma:offline-term-IV}
Under the setting of Theorem~\ref{thm:offline}, if we choose $\lambda=1$, $\beta(\delta)=c_\beta \left(d \sqrt{\iota(\delta)}+\sqrt{d\Noff }\xioff+\sqrt{p\log{\Noff}}\right)$ where $\iota(\delta)=\log\left(2pdH\Noff \xioffm/\delta\right)$, there exists an absolute constant $\widetilde{C}$ such that with probability at least $1-\delta$, it holds that for all $h\in[H]$, 
\begin{align*}
&\norm{\sum_{\tau=1}^{\Noff}\widehat{\phi}_h(s_h^\tau,a_h^\tau)\left[(P_h^{(*,T+1)}\widehat{V}_{h+1})(s_h^\tau,a_h^\tau)-\widehat{V}_{h+1}(s_{h+1}^\tau)\right]}_{\Lambda_h^{-1}} \leq \widetilde{C}\left[d\sqrt{\iota}+\sqrt{p\log {\Noff}}\right].
\end{align*}
\end{lemma}
\begin{proof}
Note that our reward functions here are selected from a general function class $\Rc$, not necessarily linear with respect to the feature function $\hphi$. The value function $\widehat{V}_{h+1}$ has the form of
 \begin{align}
V(\cdot):= \min \left\{\max_{a\in\mathcal{A}} w^T\phi(\cdot,a)+r(\cdot,a)+\beta\sqrt{\phi(\cdot,a)^\top\Lambda^{-1}\phi(\cdot,a)},1\right\}
\end{align}
for some $w\in\mathbb{R}^{d}$, $r\in\mathcal{R}$ and positive definite matrix $\Lambda\succeq \lambda I_d$. Let $\mathcal{V}$ be the function class of $V(\cdot)$ and $\mathcal{N}_\varepsilon$ be the $\varepsilon$-covering number of $\mathcal{V}$ with respect to the distance $\mathrm{dist}(V,V')=\sup_s|V(s)-V'(s)|$.

Note that for any $h \in [H], v\in \Rb^d$, we have
\begin{align*}
 	   	\left|v^\top \widehat{w}_h\right| 
 	   	&= \left|v^\top \Lambda_h^{-1}\sum_{\tau=1}^{\Noff}\hphi_h(s_h^\tau,a_h^\tau)\widehat{V}_{h+1}(s_{h+1}^\tau)\right| \\
 	   	& \leq \sum_{\tau=1}^{\Noff}\left|v^\top \Lambda_h^{-1}\hphi_h(s_h^\tau,a_h^\tau)\right| \\
 	   	& {\leq} \sqrt{\left[\sum_{\tau=1}^{\Noff}\norm{v}^2_{\Lambda_h^{-1}}\right]\left[\sum_{\tau=1}^{\Noff}\norm{\hphi_{h}(s_h^\tau,a_h^\tau)}^2_{\Lambda_h^{-1}}\right]}\\
 	   	& \leq \norm{v}_2 \sqrt{d\Noff/\lambda}, 
\end{align*}
where the second inequality follows from Cauchy-Schwarz inequality and the last inequality follows from the fact that $\norm{v}_{\Lambda_h^{-1}}=\norm{\Lambda_h^{-1}}_{\mathrm{op}}^{1/2}\cdot\norm{v}_2\leq \sqrt{1/\lambda}\norm{v}_2$. Here $\norm{\cdot}_{\mathrm{op}}$ is the matrix operator norm and 
\begin{align}
\sum_{\tau=1}^{\Noff}\norm{\widehat{\phi}_h(s_h^\tau,a_h^\tau)}^2_{\Lambda_h^{-1}}=\mathrm{tr}\left(\Lambda_h^{-1}\sum_{\tau=1}^{\Noff}\left(\widehat{\phi}_h(s_h^\tau,a_h^\tau)\widehat{\phi}_h(s_h^\tau,a_h^\tau)^\top\right)\right)\leq \mathrm{tr}(I_d)=d. \label{eqn:offline-new1}
\end{align}
Thus $\norm{\widehat{w}_h}_2=\max_{v:\norm{v}_2=1}	\left|v^\top w_h^n\right|\leq\sqrt{d\Noff/\lambda}$. 

Then, using Lemma D.3, Lemma D.4 in \cite{jin2020provably} and \Cref{lemma: covering number}, we have for any fixed $\varepsilon >0$ that with probability at least $1-\delta$, for all $h\in[H]$:
\begin{align}
 	   	&\norm{\sum_{\tau=1}^{\Noff}\hphi_h(s_h^\tau,a_h^\tau)\left[\widehat{V}_{h+1}(s_{h+1}^\tau)-(P_h^{(\star,T+1)}\widehat{V}_{h+1})(s_{h}^\tau,a_h^\tau)\right]}^2 _{\Lambda_h^{-1}}\nonumber\\
 	   	 &\leq 4\left[\frac{d}{2}\log\left(\frac{\Noff+\lambda}{\lambda}\right)+\log\frac{H\Nc_\varepsilon}{\delta}\right]+\frac{8\Noff^2\varepsilon^2}{\lambda}\nonumber\\
 	   	 &\leq 4\left[\frac{d}{2}\log\left(\frac{\Noff+\lambda}{\lambda}\right)+d\log\left(1+\frac{6\sqrt{d\Noff}}{\varepsilon\sqrt{\lambda}}\right)+d^2\log\left(1+18\frac{d^{1/2}\beta^2}{\varepsilon^2\lambda}\right)\right. \nonumber \\
 	   	 &\qquad\qquad\qquad\qquad \qquad\qquad \qquad\qquad \qquad\qquad \qquad\quad  \left.+\log \Nc_\Rc(\frac{\varepsilon}{3})+\log\frac{H}{\delta}\right]  +\frac{8\Noff^2\varepsilon^2}{\lambda}\nonumber\\
 	   	 & \overset{\RM{1}}{\leq} 4\left[\frac{d}{2}\log\left(\frac{\Noff+\lambda}{\lambda}\right)+d\log\left(1+\frac{6\sqrt{d\Noff}}{\varepsilon\sqrt{\lambda}}\right)+d^2\log\left(1+18\frac{d^{1/2}\beta^2}{\varepsilon^2\lambda}\right) \right. \nonumber \\
 	   	 &\qquad\qquad\qquad\qquad \qquad\qquad \qquad\qquad \qquad\qquad\left.+p\log\left(\frac{3}{\varepsilon}\right)+\log\frac{H}{\delta}\right] +\frac{8\Noff^2\varepsilon^2}{\lambda} \label{ineq: lemma 10 1},
\end{align}
where $\RM{1}$ follows from \Cref{assumption: Linear combination}.

We select the $\varepsilon$-covering number parameters as $R=\sqrt{d\Noff/\lambda}$, $B=\beta$ (see \Cref{lemma: covering number}). Furthermore, we choose $\lambda=1$, $\beta(\delta)=c_\beta \left(d \sqrt{\iota(\delta)}+\sqrt{d\Noff }\xioff+\sqrt{p\log\Noff}\right)$,  $\varepsilon=d/\Noff$ where $\iota(\delta)=\log\left(2pd\Noff H\xioffm/\delta\right)$. Then \Cref{ineq: lemma 10 1} can be bounded by
\begin{align*}
&d\log\left(1+\Noff\right)+d\log\left(1+d^{-1/2}\Noff^{3/2}\right)+p\log (\frac{3\Noff}{d})+\log\frac{H}{\delta} \\ 
&\qquad\qquad +d^2\log(1+d^{-3/2}\Noff^2[\beta(\delta)]^2) \\
&\uset{\sim}{<}d\log\left(\Noff\right)+d\log\left(d^{-1/2}\Noff^{3/2}\right)+p\log (\frac{3\Noff}{d})+\log\frac{H}{\delta} \\
&\qquad\qquad +d^2\log\left( d^{1/2}\iota\Noff^3\xioff^2p^2\right) \\
&\uset{\sim}{<} d^2\iota+p\log {\Noff},
\end{align*}
where the notation $f(x)\uset{\sim}{<} g(x)$ denotes that there exists a universal positive constant c (independent of $x$) such that $f(x)\leq c g(x)$.

Therefore, 
\begin{align*}
&\norm{\sum_{\tau=1}^{\Noff}\hphi_h(s_h^\tau,a_h^\tau)\left[\widehat{V}_{h+1}(s_{h+1}^\tau)-(P_h^{(\star,T+1)}\widehat{V}_{h+1})(s_{h}^\tau,a_h^\tau)\right]}_{\Lambda_h^{-1}}
\uset{\sim}{<}d\sqrt{\iota}+\sqrt{p\log \Noff},
\end{align*}
where we use $\sqrt{x+y}\leq \sqrt{x}+\sqrt{y}$ for all $x,y\geq 0$.
\end{proof}
The following lemma provides our main result which upper-bounds the Bellman update error $|(\mathbb{B}_h\widehat{V}_{h+1}-\widehat{\mathbb{B}}_h\widehat{V}_{h+1})(s,a)|$.
\begin{lemma}\label{lemma:bellmanerrorbound}
Fix $\delta \in (0,1)$. Under the setting of Theorem~\ref{thm:offline}, if we choose $\lambda=1$, $\beta(\delta)=c_\beta \left(d \sqrt{\iota(\delta)}+\right.$ $\left.\sqrt{d\Noff }\xioff+\sqrt{p\log\Noff}\right)$, where $\iota(\delta)=\log\left(2pdH\Noff \xioffm/\delta\right)$, then with probability at least $1-\delta$, the following bound holds:
	\begin{align}
	    \left|(\mathbb{B}_h\widehat{V}_{h+1}-\widehat{\mathbb{B}}_h\widehat{V}_{h+1})(s,a)\right| \leq \beta(\delta)\norm{\widehat{\phi}_h(s,a)}_{\Lambda_h^{-1}}+\xioff. \label{eqn:offline-6}
	\end{align}
\end{lemma}
\begin{proof}[Proof of \Cref{lemma:bellmanerrorbound}]
 For $h \in [H]$, define $\hw^\star_h=\int_{s'}\widehat{\mu}^\star(s')\widehat{V}_{h+1}(s')ds'$. It is easy to verify that $\widehat{\phi}_h(s,a)^\top \hw^\star_h=(\overline{P}_h\widehat{V}_{h+1})(s,a)$ and $(\overline{B}_h\widehat{V}_{h+1})(s,a)=r_h(s,a)+\widehat{\phi}_h(s,a)^\top \hw^\star_h$. Then we have
\begin{align}
 	&\left|(\mathbb{B}_h\widehat{V}_{h+1}-\widehat{\mathbb{B}}_h\widehat{V}_{h+1})(s,a)\right| \nonumber \\
 	&=\left|(\mathbb{B}_h\widehat{V}_{h+1}-\overline{\mathbb{B}}_h\widehat{V}_{h+1}+\overline{\mathbb{B}}_h\widehat{V}_{h+1}-\widehat{\mathbb{B}}_h\widehat{V}_{h+1})(s,a) \right| \nonumber \\
 	&\leq \left| (P_h^{(\star,T+1)}\widehat{V}_{h+1})(s,a)-(\overline{P}_h\widehat{V}_{h+1})(s,a)\right|+\left|\widehat{\phi}_h(s,a)^\top(\hw^\star_h-\widehat{w}_h)\right| \nonumber \\
 	&\leq \xioff+\left|\widehat{\phi}_h(s,a)^\top(\hw^\star_h-\widehat{w}_h)\right|, \label{eqn:offline-1}
\end{align}
where the last inequality follows because $\left|\widehat{V}_{h+1}(s)\right|\leq 1$ for all $s\in\mathcal{S}$ and from Lemma~\ref{lemma: Approximate Feature for new task}.

 Recall that $\widehat{w}_h=\Lambda_h^{-1}\left(\sum_{\tau=1}^{\Noff}\widehat{\phi}_h(s_h^\tau,a_h^\tau)\widehat{V}_{h+1}(s^\tau_{h+1})\right)$, where $\Lambda_h\hspace{-0.02in}=\sum_{\tau=1}^{\Noff}\widehat{\phi}(s_h,a_h)\widehat{\phi}_h(s_h,a_h)^\top+\hspace{-0.02in}\lambda I_d$. Then the second term in \Cref{eqn:offline-1} can be further decomposed as 
\begin{align}
&\widehat{\phi}_h(s,a)^\top(\hw^\star_h-\widehat{w}_h) \nonumber \\
&=\widehat{\phi}_h(s,a)^\top\hspace{-0.02in}\Lambda_h^{-1}\hspace{-0.02in}\left\{\left(\sum_{\tau=1}^{\Noff}\widehat{\phi}(s_h,a_h)\widehat{\phi}_h(s_h,a_h)^\top\hspace{-0.02in}+\hspace{-0.02in}\hspace{-0.02in}\lambda I_d\right)\hw^\star_h\hspace{-0.02in}-\hspace{-0.02in}\left(\sum_{\tau=1}^{\Noff}\widehat{\phi}_h(s_h^\tau,a_h^\tau)\widehat{V}_{h+1}(s^\tau_{h+1})\right)\right\} \nonumber \\
&=\underbrace{\lambda\widehat{\phi}_h(s,a)^\top\hspace{-0.03in}\Lambda_h^{-1}\hw^\star_h}_{(\mathrm{I})}\hspace{-0.02in}+\hspace{-0.02in}\underbrace{\widehat{\phi}_h(s,a)^\top\hspace{-0.03in}\Lambda^{-1}_h\hspace{-0.03in}\left\{\hspace{-0.02in}\sum_{\tau=1}^{\Noff}\hspace{-0.02in}\widehat{\phi}_h(s_h^\tau,a_h^\tau)\hspace{-0.02in}\left[(P_h^{(*,T+1)}\widehat{V}_{h+1})(s_h^\tau,a_h^\tau)\hspace{-0.02in}-\hspace{-0.02in}\widehat{V}_{h+1}(s_{h+1}^\tau)\right]\right\} }_{(\mathrm{II})} \nonumber \\
&\qquad +\underbrace{\widehat{\phi}_h(s,a)^\top\Lambda_h^{-1}\left\{\sum_{\tau=1}^{\Noff}\widehat{\phi}_h(s_h^\tau,a_h^\tau)\left[(\overline{P}_h\widehat{V}_{h+1}-P_h^{(*,T+1)}\widehat{V}_{h+1})(s_h^\tau,a_h^\tau)\right]\right\}  }_{(\mathrm{III})}. \label{eqn:offline-2}
\end{align}

We next bound the three terms in the above equation individually.

Term (I) is upper-bounded as
\begin{align}
\left|(\mathrm{I})\right|\leq \lambda \norm{w_h}_{\Lambda_h^{-1}}\cdot\norm{\widehat{\phi}_h(s,a)}_{\Lambda_h^{-1}}\leq \sqrt{d\lambda}\norm{\widehat{\phi}_h(s,a)}_{\Lambda_h^{-1}}, \label{eqn:offline-3}
\end{align}
where the first inequality follows from Cauchy-Schwarz inequality and the second inequality follows from the fact that $\norm{w_h}_{\Lambda_h^{-1}}=\norm{\Lambda_h^{-1}}_{\mathrm{op}}^{1/2}\cdot\norm{w_h}_2\leq \sqrt{d/\lambda}$.

Term (II) is upper-bounded as
\begin{align}
|(\mathrm{(II)})|&{\leq} \norm{\widehat{\phi}_h(s,a)}_{\Lambda_h^{-1}}{\norm{\sum_{\tau=1}^{\tau}\widehat{\phi}_h(s_h^\tau,a_h^\tau)\left[(P_h^{(*,T+1)}\widehat{V}_{h+1})(s_h^\tau,a_h^\tau)-\widehat{V}_{h+1}(s_{h+1}^\tau)\right]}_{\Lambda_h^{-1}} } \nonumber\\
& {\leq} \widetilde{C}\left[d\sqrt{\iota}+\sqrt{p\log \Noff}\right]\norm{\widehat{\phi}_h(s,a)}_{\Lambda_h^{-1}}, \label{eqn:offline-new2}
\end{align}
where the first inequality follows from Cauchy-Schwarz inequality and the second inequality follows from \Cref{lemma:offline-term-IV}.

Term (III) is upper-bounded as
\begin{align}
|\mathrm{(III)}|&\leq \left|\widehat{\phi}_h(s,a)^\top\Lambda_h^{-1}\left(\sum_{\tau=1}^{\Noff}\widehat{\phi}_h(s_h^\tau,a_h^\tau)\right)\right|\cdot\xioff \nonumber \\
&\overset{\RM{1}}{\leq} \sum_{\tau=1}^{\Noff}\left|\widehat{\phi}_h(s,a)^\top\Lambda_h^{-1}\widehat{\phi}_h(s_h^\tau,a_h^\tau)\right|\cdot\xioff \nonumber \\
&\overset{\RM{2}}{\leq} \sqrt{\left(\sum_{\tau=1}^{\Noff}\norm{\widehat{\phi}_h(s,a)}^2_{\Lambda_h^{-1}}\right)}\sqrt{\left(\sum_{\tau=1}^{\Noff}\norm{\widehat{\phi}_h(s_h^\tau,a_h^\tau)}_{\Lambda_h^{-1}}^2\right)}\cdot\xioff \nonumber \\
&\overset{\RM{3}}{\leq} \xioff\cdot\sqrt{d\Noff}\norm{\widehat{\phi}_h(s,a)}_{\Lambda_h^{-1}}, \label{eqn:offline-4} 
\end{align}
where $\RM{1}$ follows because $\left|\widehat{V}_{h+1}(s)\right|\leq 1$ for all $s\in\mathcal{S}$ and from Lemma~\ref{lemma: Approximate Feature for new task}, $\RM{2}$ follows from Cauchy-Schwarz inequality, and $\RM{3}$ follows from \Cref{eqn:offline-new1}.

Choosing $\lambda=1$, $\beta(\delta)=c_\beta \left(d \sqrt{\iota(\delta)}+\sqrt{d\Noff }\xioff+\sqrt{p\log \Noff}\right)$, where $\iota(\delta)=\log\left(2pdH\Noff \xioffm/\delta\right)$, and combining \Cref{eqn:offline-1,eqn:offline-2,eqn:offline-3,eqn:offline-4,eqn:offline-new2}, we conclude that with probability at least $1-\delta$, for any $(s,a,h)\in\mathcal{S}\times\mathcal{A}\times[H]$, the following bound holds:
\begin{align}
\left|(\mathbb{B}_h\widehat{V}_{h+1}-\widehat{\mathbb{B}}_h\widehat{V}_{h+1})(s,a)\right| 
\leq \beta(\delta)\norm{\widehat{\phi}_h(s,a)}_{\Lambda_h^{-1}}+\xioff.
\end{align}
\end{proof}

\subsection{Suboptimality gap characterization: proof of \Cref{thm:offline}}\label{app:theorem2step3}

Based on the previous lemmas, we establish the suboptimality gap.

In \Cref{lemma:offline-pessimism}, let $\Gamma_h=\beta\norm{\widehat{\phi}_h(s,a)}_{\Lambda_h^{-1}}+\xioff$. Then \Cref{lemma:offline-pessimism} implies that with probability at least $1-\delta$, 
\begin{align}
V_{P^(*,T+1),r}^*&-V_{P^(*,T+1),r}^{\widehat{\pi}} \nonumber \\
&\leq 2\sum_{h=1}^{H}\mathbb{E}_{\pi^*}[\Gamma_h(s_h,a_h)|s_1=\widetilde{s}_1] \nonumber \\
&\leq 2H\xioff+2\beta\sum_{h=1}^{H}\mathbb{E}_{\pi^*}\left[\norm{\widehat{\phi}_h(s_h,a_h)}_{\Lambda_h^{-1}}\middle|s_1=s\right]. \label{eqn:offline-7}
\end{align}

We next show the second part of \Cref{thm:offline}, which is the suboptimality bound under the feature coverage assumption (see \Cref{assumption:offlinedata}). We first note that Appendix B.4 in \cite{Jin2021IsPP} shows that if $\Noff\geq 40/\kappa_\rho \cdot \log(4dH/\delta)$, then with probability at least $1-\delta/2$, for all $(s,a,h)\in\mathcal{S}\times\mathcal{A}\times[H]$,
\begin{align*}
\norm{\hphi_h(s,a)}_{\Lambda_h^{-1}}\leq \sqrt{\frac{2}{\kappa_\rho}}\cdot\frac{1}{\sqrt{\Noff}}. 
\end{align*}

By selecting $\beta(\delta/2)=c_\beta \left(d \sqrt{\iota(\delta/2)}+\sqrt{d\Noff }\xioff+\sqrt{p\log \Noff}\right)$ (see \Cref{lemma:offline-term-IV}), with probability at least $1-\delta/2$, we have
\begin{align*}
&V_{P^(*,T+1),r}^*-V_{P^(*,T+1),r}^{\widehat{\pi}}\leq 2H\xioff+2\beta\sum_{h=1}^{H}\mathbb{E}_{\pi^*}\left[\norm{\widehat{\phi}_h(s_h,a_h)}_{\Lambda_h^{-1}}\middle|s_1=s\right].
\end{align*}

By a union bound, we have with probability at least $1-\delta$, the following bound holds:
\begin{align}
&V_{P^(*,T+1),r}^*-V_{P^(*,T+1),r}^{\widehat{\pi}} \nonumber \\
&\leq 2H\left(\xioff+\beta(\delta/2)\cdot \sqrt{\frac{2}{\kappa_\rho}}\cdot\frac{1}{\sqrt{\Noff}}\right) \nonumber \\
&=O\left(\kappa_\rho^{-1/2}Hd^{1/2}\xioff+\kappa_\rho^{-1/2}Hd\sqrt{\frac{\log\left(pdH\Noff\xioffm/\delta\right)}{\Noff}}+\kappa_\rho^{-1/2}H\sqrt{\frac{p\log{\Noff}}{\Noff}}\right) .\label{eq:offlinegap1}
\end{align}

Furthermore, if the linear combination misspecification error $\xi$ (\Cref{assumption: Linear combination}) is $\tilde{O}(\sqrt{d}/\sqrt{\Noff})$, and the number of trajectories collected in upstream is as large as $$\widetilde{O}\left({H^3dK^2T\Noff}+\left(H^3d^3K+H^5K^2+H^5d^2K\right)T^2\Noff+\frac{H^5K^3T\Noff}{d^2}\right),$$ then $\xioff$ reduces to $\tilde{O}(\sqrt{d}/\sqrt{\Noff})$ by definition and \Cref{theorem: Upstream sample complexity}. Hence, the suboptimality gap is dominated by the last term in \Cref{eq:offlinegap1} and is given by
\begin{align*}
    \widetilde{O}\left(Hd^{1/2}\Non^{-1/2}\max\{d,\sqrt{p}\}\right).
\end{align*}

\section{Proof of \Cref{theorem: downstream online RL sample complexity}} \label{sec: downstream online}
Recall $\xioff=\xi+\frac{C_LC_RT v\epsilon_u}{\kappa_u}$ and for any $h \in [H]$, we define $P_h^{(\star,T+1)}(\cdot|s,a)=\langle\phi_h^\star(s,a),\mu_h^{(\star,T+1)}(\cdot)\rangle,
\overline{P}_h(\cdot|s,a)=\langle\widehat{\phi}_h(s,a),\widehat{\mu}^*_h(\cdot)\rangle,\widehat{P}_h(\cdot|s,a)=\langle\widehat{\phi}_h(s,a),\widehat{\mu}_h(\cdot)\rangle.$ For any function $f:\mathcal{S}\mapsto \mathbb{R}$ and $h\in[H]$, define 
 	\begin{align*}
 	&P_h^{(\star,T+1)}f(s,a)=\int_{s'}\langle\phi_h^\star(s,a),\mu_h^{(\star,T+1)}(s')\rangle f(s')ds', \\
 	&(\overline{P}_hf)(s,a)=\int_{s'}\widehat{\phi}(s,a)\widehat{\mu}^\star(s')f(s')ds', \\
 	&(\widehat{P}_hf)(s,a)=\int_{s'}\widehat{\phi}(s,a)\widehat{\mu}(s')f(s')ds'.
 	\end{align*}

Throughout this section, denote $\pi^n$ as the greedy policy induced by $\{Q^n_h\}_{h=1}^H$, and note that $\lambda_h^n, w_h^n, Q_h^n, V_h^n$ are defined in \Cref{alg: Online RL}. We further remark that the expectation (for example: $V_h^\pi(s)$) is taken with respect to the transition kernel of the target task, i.e., $P^{(\star,T+1)}$.

{\bf Proof Overview:} The proof of \Cref{theorem: downstream online RL sample complexity} consists of two main steps and a final suboptimality gap analysis. {\bf Step 1:} We bound the difference between the estimated action value function $r_h(s,a)+\langle \hphi_h(s,a),w_h^n\rangle$ in \Cref{alg: Online RL} and the true action value function $Q_h^\pi(s,a)$ under a certain policy $\pi$ recursively as shown in \Cref{lemma: model free simulation}. {\bf Step 2:} We prove the estimated action value function $Q_h^n$ in \Cref{alg: Online RL} is near-optimistic with respect to the optimal true action value function over steps as shown in \Cref{lemma: UCB}. Our main technical contribution lies in capturing the impact of the misspecification of the representation taken from upstream learning on these two steps. {\bf Suboptimality gap analysis:} Based on the first two steps, we first decompose the value function difference recursively, and then obtain a final suboptimality gap.

\subsection{Bounding the action value function difference}
Following the proof similar to that for \Cref{lemma:offline-term-IV}, we introduce the concentration lemma for online RL that upper-bounds the stochastic noise in
regression.	 
\begin{lemma}\label{lemma:online-concentration}
Fix $\delta \in (0,1)$. Under the setting of Theorem~\ref{theorem: downstream online RL sample complexity}, we choose $\lambda=1$, $\beta_n=c_\beta \left(d \sqrt{\iota_n}+\sqrt{dn }\xioff\right.$ $\left.  +\sqrt{p\log{n}}\right)$, where $\iota_n=\log\left(2pdHn \xioffm/\delta\right)$. Then, there exists an absolute constant $\widetilde{C}$ such that with probability at least $1-\delta/2$,  the following inequality holds for any $n \in [N_{on}], h\in[H]$:
\begin{align*}
&\norm{\sum_{\tau=1}^{n-1}\widehat{\phi}_h(s_h^\tau,a_h^\tau)\left[(P_h^{(*,T+1)}{V}^n_{h+1})(s_h^\tau,a_h^\tau)-{V}^n_{h+1}(s_{h+1}^\tau)\right]}_{\Lambda_h^{-1}} \leq \widetilde{C}\left[d\sqrt{\iota_n}+\sqrt{p\log {n}}\right].
\end{align*}
\end{lemma}
    \begin{lemma} \label{lemma: model free simulation}
    	Fix $\delta \in (0,1)$. There exists a constant $c_\beta$ such that for $\beta_n=c_\beta \left(d \sqrt{\iota_n}+\sqrt{nd}\xioff+\sqrt{p\log {n}}\right)$ where $\iota_n=\log\left(2dnH\xioff/\delta\right)$, and for any policy $\pi$, with probability at least $1-\delta/2$, we have for any $s\in \Sc, a \in \Ac, h \in [H], n \in [N_{on}]$ that:
        \begin{align*}
        \left(r_h(s,a)+\langle \hphi_h(s,a),w_h^n\rangle\right)-Q_h^\pi(s,a)=P_h^{(\star,T+1)}(V_{h+1}^n-V_{h+1}^\pi)(s,a)+\Delta_h^k(s,a),
        \end{align*}
        for some $\Delta_h^n(s,a)$ that satisfies $\|\Delta_h^k(s,a)\| \leq \beta_n \norm{\hphi_h(s,a)}_{(\Lambda_h^n)^{-1}}+2\xioff$.
    \end{lemma}
    \begin{proof}
        For policy $\pi$, define $w_h^\pi=\int V_{h+1}^\pi(s^\prime)\hmu^{\star}(s^\prime)ds^\prime$. Hence, $\langle\hphi_h(s,a),w_h^\pi\rangle=\overline{P}_hV_{h+1}^\pi(s,a)$ and $\norm{w_h^\pi}_2 \leq C_L\sqrt{d}$ by \Cref{lemma: Approximate Feature for new task}. These facts further yield that for any $s \in \Sc, a \in \Ac, h \in [H]$:
        \begin{align*}
         &\left|Q_h^\pi(s,a)-\left(r_h(s,a)+\langle\hphi_h(s,a),w_h^\pi\rangle\right)\right| = \left|P_h^{(\star,T+1)}V_{h+1}^\pi(s,a)-\overline{P}_h V_{h+1}^\pi(s,a)\right|
         \leq \xioff,
        \end{align*}
        where the last inequality follows from \Cref{lemma: Approximate Feature for new task}.
        
        Then, we further derive
        \begin{align}
        &\left(r_h(s,a)+\langle \hphi_h(s,a),w_h^n\rangle\right)-Q_h^\pi(s,a)\nonumber\\
        &\quad =\left(r_h(s,a)+\langle \hphi_h(s,a),w_h^n\rangle\right)-\left(r_h(s,a)+\langle\hphi_h(s,a),w_h^\pi\rangle\right)+\left(r_h(s,a)+\langle\hphi_h(s,a),w_h^\pi\rangle\right)-Q_h^\pi(s,a)\nonumber\\
        &\quad \leq  \langle\hphi_h(s,a),w_h^n\rangle-\langle\hphi_h(s,a),w_h^\pi\rangle + \left|\left(r_h(s,a)+\langle\hphi_h(s,a),w_h^\pi\rangle\right)-Q_h^\pi(s,a)\right|.\label{ineq: initial RHS}
        \end{align}
        The first term can be bounded by 
        \begin{align*}
        &\langle\hphi_h(s,a),w_h^n\rangle-\langle\hphi_h(s,a),w_h^\pi\rangle\\
        & \quad =\hphi_h(s,a)^\top(\Lambda_h^n)^{-1}\sum_{\tau=1}^{n-1}\hphi(s_h^\tau,a_h^\tau)V_{h+1}^n(s_{h+1}^\tau)-\hphi_h(s,a)^\top w_h^\pi\\
         & \quad =\hphi_h(s,a)^\top(\Lambda_h^n)^{-1}\left\{\sum_{\tau=1}^{n-1}\hphi(s_h^\tau,a_h^\tau)V_{h+1}^n(s_{h+1}^\tau)-\lambda w_h^\pi-\sum_{\tau=1}^{n-1}\hphi(s_h^\tau,a_h^\tau)\overline{P}_h V_{h+1}^\pi\right\}\\
        & \quad =\underbrace{-\lambda\hphi_h(s,a)^\top(\Lambda_h^n)^{-1} w_h^\pi}_{\URM{1}}+\underbrace{\hphi_h(s,a)^\top(\Lambda_h^n)^{-1}\left\{\sum_{\tau=1}^{n-1}\hphi(s_h^\tau,a_h^\tau)\left[V_{h+1}^n(s_{h+1}^\tau)-P_h^{(\star,T+1)}V_{h+1}^n(s_{h}^\tau,a_h^\tau)\right]\right\}}_{\URM{2}}\\
        & \qquad +\underbrace{\hphi_h(s,a)^\top(\Lambda_h^n)^{-1}\left\{\sum_{\tau=1}^{n-1}\hphi(s_h^\tau,a_h^\tau)\overline{P}_h \left(V_{h+1}^n-V_{h+1}^\pi\right)(s_h^\tau,a_h^\tau)\right\}}_{\URM{3}}\\
        & \qquad +\underbrace{\hphi_h(s,a)^\top(\Lambda_h^n)^{-1}\left\{\sum_{\tau=1}^{n-1}\hphi(s_h^\tau,a_h^\tau)\left(P_h^{(\star,T+1)}-\overline{P}_h\right) V_{h+1}^n(s_{h}^\tau,a_h^\tau)\right\}}_{\URM{4}}.
        \end{align*}
        We next bound the above four terms individually. 
        
        For $\URM{1}$, we derive the following bound:
        \begin{align}
        	|\URM{1}| \leq \norm{\hphi_h(s,a)}_{(\Lambda_h^n)^{-1}}\norm{\lambda w_h^\pi}_{(\Lambda_h^n)^{-1}}\leq \sqrt{\lambda}\norm{w_h^\pi}_2 \norm{\hphi_h(s,a)}_{(\Lambda_h^n)^{-1}}=C_L\sqrt{\lambda d}\norm{\hphi_h(s,a)}_{(\Lambda_h^n)^{-1}} \label{ineq: URM1}. 
        \end{align}
        
        For $\URM{2}$, by \Cref{lemma:online-concentration}, we have
        \begin{align*}
        	&|\URM{2}| \\
        	& \quad\leq \norm{\hphi_h(s,a)}_{(\Lambda_h^n)^{-1}}\norm{\sum_{\tau=1}^{n-1}\hphi_h(s_h^\tau,a_h^\tau)\left[V_{h+1}^n(s_{h+1}^\tau)-P_h^{(\star,T+1)}V_{h+1}^n(s_{h}^\tau,a_h^\tau)\right]}_{(\Lambda_h^n)^{-1}} \\
        	& \quad\leq \left(\widetilde{C}d\sqrt{\iota_n}+\sqrt{p\log n}\right)\norm{\hphi_h(s,a)}_{(\Lambda_h^n)^{-1}}.
        \end{align*}
        
        For $\URM{3}$, we have
        \begin{align*}
        |\URM{3}| &\leq \left|\hphi_h(s,a)^\top(\Lambda_h^n)^{-1}\left\{\sum_{\tau=1}^{n-1}\hphi(s_h^\tau,a_h^\tau)\hphi(s_h^\tau,a_h^\tau)^\top \int \left(V_{h+1}^n-V_{h+1}^\pi\right)(s^\prime)\hmu_h^\star(s^\prime) ds^\prime\right\}\right|\\
        & \leq \left|\hphi_h(s,a)^\top(\Lambda_h^n)^{-1}\left(\Lambda_h^n-\lambda I\right)\int \left(V_{h+1}^n-V_{h+1}^\pi\right)(s^\prime)\hmu_h^\star(s^\prime) ds^\prime\right|\\
        & = \underbrace{\left|\hphi_h(s,a)^\top\int \left(V_{h+1}^n-V_{h+1}^\pi\right)(s^\prime)\hmu_h^\star(s^\prime) ds^\prime\right|}_{(a)}+
        \underbrace{\left|\lambda\hphi_h(s,a)^\top(\Lambda_h^n)^{-1}\int \left(V_{h+1}^n-V_{h+1}^\pi\right)(s^\prime)\hmu_h^\star(s^\prime) ds^\prime\right|}_{(b)}.
        \end{align*}
        
        For term $(a)$, we have
        \begin{align*}
        	(a)&
        	=\hphi_h(s,a)^\top\int \left(V_{h+1}^n-V_{h+1}^\pi\right)(s^\prime)\hmu_h^\star(s^\prime) ds^\prime\\
        	&= \overline{P}_h\left(V_{h+1}^n-V_{h+1}^\pi\right)(s,a)\\
        	&\leq P^{(\star,T+1)}\left(V_{h+1}^n-V_{h+1}^\pi\right)(s,a)+\xioff,
        \end{align*}
        where the last inequality follows from \Cref{lemma: Approximate Feature for new task}.
        
        For term (b), similarly to \Cref{ineq: URM1}, we have
        \begin{align*}
        	(b) \leq C_L\sqrt{\lambda d}\norm{\hphi_h(s,a)}_{(\Lambda_h^n)^{-1}}.
        \end{align*}
        
         For $\URM{4}$, we derive
        \begin{align*}
        |\URM{4}| &\leq \left|\hphi_h(s,a)^\top(\Lambda_h^n)^{-1}\left\{\sum_{\tau=1}^{n-1}\hphi_h(s_h^\tau,a_h^\tau)\right\}\right|\xioff\\
        & \leq \sum_{\tau=1}^{n-1}\left|\hphi_h(s,a)^\top(\Lambda_h^n)^{-1}\hphi_h(s_h^\tau,a_h^\tau)\right|\xioff\\
        & \overset{\RM{1}}{\leq} \sqrt{\left[\sum_{\tau=1}^{n-1}\norm{\hphi_{h}(s,a)}^2_{(\Lambda_h^n)^{-1}}\right]\left[\sum_{\tau=1}^{n-1}\norm{\hphi_{h}(s_h^\tau,a_h^\tau)}^2_{(\Lambda_h^n)^{-1}}\right]}\xioff\\
        & \overset{\RM{2}}{\leq} \xioff\sqrt{dn}\norm{\hphi_{h}(s,a)}_{(\Lambda_h^n)^{-1}},
        \end{align*}
        where $\RM{1}$ follows from Cauchy-Schwarz inequality and $\RM{2}$ follows because
        \begin{align*}        	\sum_{\tau=1}^{n-1}\norm{\hphi_{h}(s_h^\tau,a_h^\tau)}^2_{(\Lambda_h^n)^{-1}}={\rm tr}\left((\Lambda_h^n)^{-1}\sum_{\tau=1}^{n-1}\left(\hphi_{h}(s_h^\tau,a_h^\tau)\hphi_{h}(s_h^\tau,a_h^\tau)^\top\right)\right)\leq {\rm tr}(I_d)=d.
        \end{align*}
 Substituting the bounds on $\URM{1},\URM{2},\URM{3},\URM{4}$ into \Cref{ineq: initial RHS}, we finish the proof.
    \end{proof}
 \subsection{Proving optimism of value function}   
    \begin{lemma}\label{lemma: UCB}
        Under the setting of \Cref{theorem: downstream online RL sample complexity}, with probability at least $1-\delta/2$, for any $s \in \Sc,a \in \Ac, h\in[H], n\in[N_{on}]$, we have 
        \begin{align}
            Q_h^n(s,a) \geq Q^\star_h(s,a)-2(H-h+1)\xioff. \label{ineq: near optimistic}
        \end{align}
    \end{lemma}
    \begin{proof}
    We prove this lemma by induction.
    First, for step $H$, by \Cref{lemma: model free simulation}, we have 
    \begin{align*}
        &\left|\left(r_H(s,a)+\langle \hphi_H(s,a),w_H^n\rangle\right)-Q_H^\pi(s,a)\right|\\
        &\quad=\left|P_H^{(\star,T+1)}(V_{H+1}^n-V_{H+1}^\pi)(s,a)+\Delta_H^k(s,a)\right|\\
        &\quad  \leq \beta_n \norm{\hphi_H(s,a)}_{(\Lambda_H^n)^{-1}}+2\xioff.
    \end{align*}
    Thus, 
    \begin{align*}
        Q_H^n(s,a)&= \min\left\{r_H(\cdot,\cdot)+\langle\hphi(\cdot,\cdot),w_H^n\rangle+\beta_n\norm{\hphi(\cdot,\cdot)}_{(\Lambda_H^n)^{-1}},1\right\}\\
        &  \geq Q^\star_H(s,a)-2\xioff.
    \end{align*}
    
    Suppose \Cref{ineq: near optimistic} holds for step $h+1$. Then for step $h$, following from \Cref{lemma: model free simulation}, we have:
    \begin{align*}
        \Big(r_h(s,a)&+\langle \hphi_h(s,a),w_h^n\rangle\Big)-Q_h^\star(s,a)\\
        & = \Delta_h^n(s,a)+P_h^{(\star,T+1)}(V_{h+1}^n-V_{h+1}^\star)(s,a)\\
        &  \geq -\beta_n \norm{\hphi_h(s,a)}_{(\Lambda_h^n)^{-1}}-2\xioff-2(H-h)\xioff \\
        &  \geq -\beta_n \norm{\hphi_h(s,a)}_{(\Lambda_h^n)^{-1}}-2(H-h+1)\xioff. 
    \end{align*}
    
    Therefore,
    \begin{align*}
    Q_h^n(s,a)&= \min\left\{r_h(\cdot,\cdot)+\langle\hphi(\cdot,\cdot),w_h^n\rangle+\beta_n\norm{\hphi(\cdot,\cdot)}_{(\Lambda_h^n)^{-1}},1\right\}\\
    & \geq Q^\star_h(s,a)-2(H-h+1)\xioff,
    \end{align*}
    which finishes the proof.
    \end{proof}
\subsection{Suboptimality gap: proof of \Cref{theorem: downstream online RL sample complexity}}
Before proving \Cref{theorem: downstream online RL sample complexity}, we introduce the following lemma to decompose the value function difference recursively. 
    \begin{lemma}\label{lemma: recursive formula}
    	Fix $\delta \in (0,1)$. Let $\delta_h^n=V_h^n(s_h^n)-V_h^{\pi^n}(s_h^n)$ and $\xi_{h+1}^n=\Eb\left[\delta_{h+1}^n|s_h^n,a_h^n\right]-\delta_{h+1}^n$. Then, with probability at least $1-\delta/2$, for $h \in[H], n\in [N_{on}]$:
    	\begin{align*}
    		\delta_{h}^n \leq \delta_{h+1}^n+ \xi_{h+1}^n+2\beta_n\norm{\hphi_{h}(s_h^n,a_h^n)}_{(\Lambda_h^n)^{-1}}+2\xioff.
    	\end{align*}
    	\begin{proof}
    		By \Cref{lemma: model free simulation}, with probability at least $1-\delta/2$, for any $s\in \Sc, a \in \Ac, h \in [H], n \in [N_{on}]$, we have
    		\begin{align*}
    		&Q_h^n(s,a)-Q_h^{\pi^n}(s,a)\\
    		& \quad = \Delta_h^n(s,a)+P_h^{(\star,T+1)}(V_{h+1}^n-V_{h+1}^{\pi^n})(s,a)\\
    		& \quad \leq \beta_n \norm{\hphi_h(s,a)}_{(\Lambda_h^n)^{-1}}+2\xioff+P_h^{(\star,T+1)}(V_{h+1}^n-V_{h+1}^{\pi^n})(s,a). 
    		\end{align*}
    		By the definition of $\pi^n$ in \Cref{alg: Online RL}, we have $\pi^n(s_h^n)=a_h^n={\arg\max}_{a\in \Ac}Q_h^n(s_h,a)$. Then $Q_h^n(s_h^n,a_h^n)-Q_h^{\pi^n}(s_h^n,a_h^n)=V_h^n(s_h^n)-V_h^{\pi^n}(s_h^n)=\delta_{h}^n$. Thus,
    		\begin{align*}
    			\delta_{h}^n \leq \delta_{h+1}^n+ \xi_{h+1}^n+2\beta_n\norm{\hphi_{h}(s_h^n,a_h^n)}_{(\Lambda_h^n)^{-1}}+2\xioff.
    		\end{align*}
    	\end{proof}
    \end{lemma}
	Finally, we combine \Cref{lemma:online-concentration,lemma: model free simulation,lemma: UCB,lemma: recursive formula} to prove \Cref{theorem: downstream online RL sample complexity}.
	\begin{proof}[Proof of \Cref{theorem: downstream online RL sample complexity}]
	The regret can be bounded by
	\begin{align} &\sum_{n=1}^{N_{on}}\left(V^\star_{P^{(\star,T+1)},r}-V^{\pi^n}_{P^{(\star,T+1)},r}\right) \nonumber\\
	&\quad\overset{\RM{1}}{\leq} \sum_{n=1}^{N_{on}}\left\{\left(V^n_1-V^{\pi^n}_{P^{(\star,T+1)},r}\right)+2H\xioff\right\}\nonumber\\
	& \quad\overset{\RM{2}}{\leq} \sum_{n=1}^{N_{on}}\left\{\sum_{h=1}^H\left[\xi_h^n+2\beta_n\norm{\hphi_{h}(s_h^n,a_h^n)}_{(\Lambda_h^n)^{-1}}+2\xioff\right]+2H\xioff\right\}\nonumber\\
	& \quad{\leq} \underbrace{\sum_{n=1}^{N_{on}}\sum_{h=1}^H\xi_h^n}_{\URM{1}}+\underbrace{2\sum_{n=1}^{N_{on}}\sum_{h=1}^H\beta_n\norm{\hphi_{h}(s_h^n,a_h^n)}_{(\Lambda_h^n)^{-1}}}_{\URM{2}}+{4HN_{on}\xioff},\label{ineq: Thm 4.7 Sum}
 	\end{align}
 	where $\RM{1}$ follows from \Cref{lemma: UCB} and $\RM{2}$ follows from \Cref{lemma: recursive formula}.\\
For term $\URM{1}$, note that $\left\{\xi_h^n\right\}_{n=1,h=1}^{N_{on},H}$ is a  martingale difference with $|\xi_h^n| \leq 2$. By Azuma–Hoeffding inequality, with proability at least $1-\delta/4$, we have
 	\begin{align}
 	    \left|\sum_{n=1}^{N_{on}}\sum_{h=1}^H\xi_h^n\right| \leq \sqrt{8N_{on}H\log(8/\delta)}.\label{ineq: Thm 4.7 I}
 	\end{align}
 	For term $\URM{2}$, we derive
 	\begin{align}
 		\URM{2}&=2\sum_{h=1}^H\sum_{n=1}^{N_{on}}\beta_n\norm{\hphi_{h}(s_h^n,a_h^n)}_{(\Lambda_h^n)^{-1}}\nonumber\\
 		&\overset{\RM{1}}{\leq} 2\sum_{h=1}^H\sqrt{\sum_{n=1}^{N_{on}}\beta_n^2}\sqrt{\sum_{n=1}^{N_{on}}\norm{\hphi_{h}(s_h^n,a_h^n)}_{(\Lambda_h^n)^{-1}}^2}\nonumber
 		\\
 		&\overset{\RM{2}}{\leq} 2\sum_{h=1}^H\sqrt{2c_\beta^2\left(d^2\iota_n N_{on}+N_{on}^2d\xioff^2+p\Non\log\Non\right)}\sqrt{2d\log\left(1+\frac{N_{on}}{d\lambda}\right)}\nonumber\\	
 		&{\leq} 2H\sqrt{2c_\beta^2\left(d^2\iota_n N_{on}+N_{on}^2d\xioff^2+p\Non\log\Non\right)}\sqrt{4d\log{N_{on}}}	\nonumber\\
 		&\overset{\RM{3}}{\leq} 4\sqrt{2}c_\beta\left(\sqrt{H^2d^3\iota_n N_{on}\log{N_{on}}}+{HdN_{on}\xioff}\sqrt{\log{N_{on}}}+H\sqrt{dp\Non}\log{\Non}\right),\label{ineq: Thm 4.7 II}
 	\end{align}
 	where $\RM{1}$ follows from Cauchy-Schwarz inequality, $\RM{2}$ follows from \Cref{lemma: Elliptical_potential}, and $\RM{3}$ follows because $\forall x,y \geq 0, \sqrt{x+y}\leq\sqrt{x}+\sqrt{y}$.
 	
 	Combining \Cref{ineq: Thm 4.7 Sum}, \Cref{ineq: Thm 4.7 I} and \Cref{ineq: Thm 4.7 II}, we obtain
 	\begin{align*}
 	\sum_{n=1}^{N_{on}}&\left(V^\star_{P^{(\star,T+1)},r}-V^{\pi^n}_{P^{(\star,T+1)},r}\right) \\
 	 &  \leq 8\sqrt{2}c_\beta\left(\sqrt{H^2d^3\iota_n \Non\log{\Non}}+{Hd\Non\xioff}\sqrt{\log{\Non}}+H\sqrt{dp\Non}\log{\Non}\right)\\
 	 & =\widetilde{O}\left(Hd\Non\xioff+H\sqrt{d^3\Non}+H\sqrt{dp\Non}\right)\\
 	 & =\widetilde{O}\left({HdN_{on}\xioff}+H\sqrt{dN_{on}}\max\{d,\sqrt{p}\}\right).
 	\end{align*}  
 	Dividing both sides by $\Non$, we have 
 	\begin{align}\label{eq:onlinegap1}
 	    V^\star_{P^{(\star,T+1)},r}-V^{\tilde{\pi}}_{P^{(\star,T+1)},r} \leq \widetilde{O}\left({Hd\xioff}+Hd^{1/2}N_{on}^{-1/2}\max\{d,\sqrt{p}\}\right).
 	\end{align}
 	Furthermore, if the linear combination misspecification error $\xi$ (\Cref{assumption: Linear combination}) is $\tilde{O}(\sqrt{d}/\sqrt{\Non})$, and the number of trajectories collected in upstream is as large as $$\widetilde{O}\left({H^3dK^2T\Non}+\left(H^3d^3K+H^5K^2+H^5d^2K\right)T^2\Non+\frac{H^5K^3T\Non}{d^2}\right),$$ then $\xioff$ reduces to $\tilde{O}(\sqrt{d}/\sqrt{\Non})$ by definition and \Cref{theorem: Upstream sample complexity}, and hence the second term in \Cref{eq:onlinegap1} dominates. The suboptimality gap thus becomes
\begin{align*}
    \widetilde{O}\left(Hd^{1/2}\Non^{-1/2}\max\{d,\sqrt{p}\}\right). 
\end{align*} 
	\end{proof}

\section{Auxiliary Lemmas}

In this section, we provide several lemmas that are commonly used for the analysis of MDP problems.

The following lemma \citep{dann2017unifying} will be useful to measure the difference between two value functions under two MDPs and reward functions. We define $P_h V_{h+1}(s_h,a_h) = \Eb_{s\sim P_h(\cdot|s_h,a_h)}\left[V(s)\right]$ as a shorthand notation.
 	\begin{lemma} \label{lemma: Simulation}{\rm(Simulation lemma).}
 		Suppose $P_1$ and $P_2$ are two MDPs and $r_1$, $r_2$ are the corresponding reward functions. Given a policy $\pi$, we have,  
 		\begin{align*}
 		V_{h,P_1,r_1}^{\pi}&(s_h) - V_{h,P_2,r_2}^{\pi}(s_h)\\
 		&= \sum_{\ph=h}^H  \mathop{\Eb}_{s_\ph \sim (P_2,\pi) \atop a_\ph \sim \pi}\left[r_1(s_\ph,a_\ph) - r_2(s_\ph,a_\ph) + (P_{1,\ph} - P_{2,\ph})V^{\pi}_{\ph+1,P_1,r}(s_\ph,a_\ph)|s_h\right]\\
 		& = \sum_{\ph=h}^H  \mathop{\Eb}_{s_\ph \sim (P_1,\pi) \atop a_\ph \sim \pi}\left[r_1(s_\ph,a_\ph) - r_2(s_\ph,a_\ph) + (P_{1,\ph} - P_{2,\ph})V^{\pi}_{\ph+1,P_2,r}(s_\ph,a_\ph)|s_h\right].
 		\end{align*}
 	\end{lemma}

The following lemma is essential in bounding the suboptimality in downstream offline RL (see Lemma~3.1 in~\citet{Jin2021IsPP}).
\begin{lemma}\label{lemma:offline-decom}
Let $\{\widehat{\pi}_h\}_{h=1}^H$ be the policy such that $\widehat{V}_h(s)=\langle\widehat{Q}_h(s,\cdot),\widehat{\pi}_h(\cdot|s)\rangle_{\mathcal{A}}$ and $\zeta_h(s,a)=(\mathbb{B}_h\widehat{V}_{h+1}(s,a))-\widehat{Q}_h(s,a)$. Then for any $\widehat{\pi}$ and $s\in\mathcal{S}$, we have 
\begin{align*}
V_1^\pi(s)-V_1^{\widehat{\pi}}(s)&=\sum_{h=1}^{H}\mathbb{E}_\pi[\zeta_h(s_h,a_a)|s_1=s]-\sum_{h=1}^{H}\mathbb{E}_{\widehat{\pi}}[\zeta_h(s_h,a_h)|s_1=s] \nonumber \\
&\qquad +\sum_{h=1}^{H}\mathbb{E}_{\pi}[\langle\widehat{Q}_h(s_h,\cdot),\pi_h(\cdot|s_h)-\widehat{\pi}(\cdot|s_h)\rangle|s_1=s],
\end{align*}
where the expectation is taken over $s_h,a_h$.
\end{lemma}
  	The following lemma is a standard inequality in the regret analysis for linear MDPs in reinforcement learning (see Lemma G.2 in \citet{NEURIPS2020_e894d787} and Lemma 10 in \citet{uehara2021representation}).	
 	\begin{lemma} \label{lemma: Elliptical_potential} {\rm (Elliptical potential lemma).}
 		Consider a sequence of $d \times d$ positive semidefinite matrices $X_1, \dots, X_N$ with ${\rm tr}(X_n) \leq 1$ for all $n \in [N]$. Define $M_0=\lambda_0 I$ and $M_n=M_{n-1}+X_n$. Then
 		\begin{equation*}
 		\sum_{n=1}^N {\rm tr}\left(X_nM_{n-1}^{-1}\right) \leq 2\log \det(M_N)- 2\log \det(M_0) \leq 2d\log\left(1+\frac{N}{d\lambda_0}\right).
 		\end{equation*}
 	\end{lemma}
 	Next, we introduce some useful inequalities that help convert the finite sample error bound into the sample complexity.

\begin{lemma}[$\varepsilon$-Covering Number]\label{lemma: covering number}
 Let $\mathcal{V}$ denote a class of function mapping from $\mathcal{S}$ to $\mathbb{R}$ with the following parametric form
\begin{align*}
V(\cdot)=\min\left\{\max_{a\in\mathcal{A}} r(\cdot,a)+w^\top\phi(\cdot,a)+\alpha\sqrt{\phi(\cdot,a)^\top\Lambda^{-1}\phi(\cdot,a)},1\right\},
\end{align*}
where the parameters $(r,w,\beta,\Lambda)$ satisfy $r\in\mathcal{R}$, $\norm{w}\leq L$, $\alpha\in[0,B]$ and $\Sigma\succeq \lambda I$. Assume $\norm{\phi(s,a)}\leq 1$ for all $(s,a)$ pairs, and let $\mathcal{N}(\varepsilon;r,R,B,\lambda)$ be the $\varepsilon$-covering number of $\mathcal{V}$ with respect to the distance $\mathrm{dist}(V,V')=\sup_s |V(s)-V'(s)|$. Further let $\mathcal{N}_\mathcal{R}(\varepsilon)$ be the $\epsilon$-covering number of function class $\mathcal{R}$. Then 
\begin{align*}
\log|\mathcal{N}(\epsilon;R,B,\lambda)|\leq d\log(1+6R/\varepsilon)+d^2\log(1+18d^{1/2}B^2/(\varepsilon^2\lambda))+\log \Nc_\Rc\big(\frac{\varepsilon}{3}\big)
\end{align*}
where $\Nc_\Rc(\frac{\epsilon}{3})$ is the $\frac{\epsilon}{3}$ covering number with respect to the reward function class $\Rc$.
\end{lemma} 
\begin{proof}
The proof is essentially the same as that in~\cite{jin2020provably} except that the function $r(s,a)$ is not necessarily linear with respect to the representation $\phi(s,a)$, and the function $\phi$ is selected from a function class $\Phi$. 

Reparametrize the function class $\mathcal{V}$ by letting $A=\alpha^2\Lambda^{-1}$, and we have
\begin{align*}
V(\cdot)=\min\left\{\max_{a\in\mathcal{A}} r(\cdot,a)+w^\top\phi(\cdot,a)+\sqrt{\phi(\cdot,a)^\top A\phi(\cdot,a)},1\right\},
\end{align*}
where $r\in\mathcal{R}$, $\norm{w}\leq R$, and $\norm{A}\leq B^2\lambda^{-1}$. For any two functions $V_1,V_2\in\mathcal{V}$, let them take the above form with parameters $(r_1,w_1,A_1)$ and $(r_2,w_2,A_2)$, respectively. Since both $\min\{\cdot,1\}$ and $\max_a$ are contractions, we have
\begin{align}
&\mathrm{dist}(V_1,V_2) \nonumber \\
&\leq \sup_{s,a}\left|\left[r_1(s,a)+w_1^\top\phi(s,a)+\sqrt{\phi(s,a)^\top A_1\phi(s,a)}\right]\right.  \nonumber \\ 
&\qquad \qquad \qquad \qquad  -\left.\left[r_2(s,a)+w_2^\top\phi(s,a)+\sqrt{\phi(s,a)^\top A_2\phi(s,a)}\right]\right| \nonumber \\
&\leq \sup_{s,a}\left|r_1(s,a)-r_2(s,a)\right|+\sup_{\phi:\norm{\phi}\leq 1}\left|\left[w_1^\top\phi+\sqrt{\phi^\top A_1\phi}\right]-\left[w_2^\top\phi+\sqrt{\phi^\top A_2\phi}\right]\right| \nonumber \\
&\leq \sup_{s,a}\left|r_1(s,a)-r_2(s,a)\right|+\sup_{\phi:\norm{\phi}\leq 1}\left|(w_1-w_2)^\top\phi\right|+\sup_{\phi:\norm{\phi}\leq 1}\sqrt{\left|\phi^\top (A_1-A_2)\phi\right|} \nonumber \\
&=\sup_{s,a}\left|r_1(s,a)-r_2(s,a)\right|+\norm{w_1-w_2}+\sqrt{\norm{A_1-A_2}} \nonumber \\
&\leq \sup_{s,a}\left|r_1(s,a)-r_2(s,a)\right|+\norm{w_1-w_2}+\sqrt{\norm{A_1-A_2}_{\mathrm{F}}},  \label{eqn:lemma-cover-1}
\end{align}
where the second to last inequality follows from the fact that $|\sqrt{x}-\sqrt{y}|\leq \sqrt{|x-y|}$ for any $x,y\geq 0$. For matrices, $\norm{\cdot}$ and $\norm{\cdot}_{\mathrm{F}}$ denote the matrix operator norm and Frobenius norm, respectively.

Let $\mathcal{C}_\mathcal{R}$ be an $\frac{\varepsilon}{3}$-cover of $\mathcal{R}$ such that $|\mathcal{C}_\mathcal{R}|= \mathcal{N}_\mathcal{R}(\frac{\varepsilon}{3})$. Let $\mathcal{C}_w$ be an $\frac{\varepsilon}{3}$-cover of $\{w\in\mathbb{R}^d|\norm{w}\leq R\}$ with respect to the $l_2$-norm of a vector, and let $\mathcal{C}_A$ be an $\frac{\varepsilon^2}{9}$-cover of $\{A\in\mathbb{R}^{d\times d}|\norm{A}_{\mathrm{F}}\leq d^{1/2}B^2\lambda^{-1}\}$ with respect to the Frobenius norm. By Lemma D.5 in~\cite{jin2020provably}, it holds that
\begin{align*}
|\mathcal{C}_w|\leq (1+6R/\varepsilon)^d, \qquad |\mathcal{C}_A|\leq [1+18d^{1/2}B^2/(\lambda\varepsilon^2)]^{d^2}.
\end{align*}
By \Cref{eqn:lemma-cover-1}, for any $V_1\in V$, there exists $r_2\in\mathcal{C}_\mathcal{R}$, $w_2\in\mathcal{C}_w$ and $A$ such that $V_2$ parametrized by $(r_2,w_2,A_2)$ satisfies $\mathrm{dist}(V_1,V_2)\leq \varepsilon$. Hence, it holds that $\mathcal{N}(\epsilon;R,B,\lambda)\leq |\mathcal{N}_\mathcal{R}(\frac{\varepsilon}{3})|\cdot|\mathcal{C}_w|\cdot|\mathcal{C}_A|$, which yields the desired result.
\end{proof}

\end{document}